\pdfoutput=1
\documentclass{article}

\usepackage{microtype}
\usepackage{graphicx}
\usepackage{booktabs} 

\usepackage{hyperref}

\hypersetup{
    colorlinks=true,
    linkcolor=blue,
    filecolor=magenta,      
    urlcolor=cyan,
}

\usepackage{amsthm}
\usepackage{amssymb}
\usepackage{amsmath}
\usepackage{style}
\usepackage{braket}
\usepackage{wrapfig}
\usepackage{comment}
\usepackage{style}
\usepackage{fullpage}

\usepackage{mathtools}
\usepackage{enumitem}
\definecolor{mydarkblue}{rgb}{0,0.08,0.45}

\newcommand{\nishanth}[1]{{\color{purple}{Nishanth: #1}}}

\definecolor{Burgundy}{RGB}{144,0,32}

\usepackage{subcaption}

\newcommand{\sm}[1]{\boldsymbol{#1}}

\newcommand{\bbet}{\sm{\beta}}

\newcommand{\bnorm}{\beta}

\newtagform{blue}{\color{mydarkblue}(}{)}

\title{For Manifold Learning, Deep Neural Networks can be Locality Sensitive Hash Functions}
\date{}
\author{
  Nishanth Dikkala\\
  \texttt{nishanthd@google.com}
  \and
  Gal Kaplun\\
  \texttt{galkaplun@g.harvard.edu\footnote{Work  partially done while interning at Google.}}
  \and
  Rina Panigrahy\\
  \texttt{rinap@google.com}
}

\begin{document}

\maketitle
\newif\ificml
\newif\ifarxv
\icmlfalse
\arxvtrue

\begin{abstract}
It is well established that training deep neural networks gives useful representations that capture essential features of the inputs. However, these representations are poorly understood in theory and practice. In the context of supervised learning an important question is whether these representations capture features informative for classification, while filtering out non-informative noisy ones. We explore a formalization of this question by considering a generative process where each class is associated with a high-dimensional manifold and different classes define different manifolds. Under this model, each input is produced using two latent vectors: (i) a ``manifold identifier" $\gamma$ and; (ii)~a ``transformation parameter"  $\theta$ that shifts examples along the surface of a manifold. E.g., $\gamma$ might represent a canonical image of a dog, and $\theta$ might stand for variations in pose, background or lighting. We provide theoretical and empirical evidence that neural representations can be viewed as LSH-like functions that map each input to an embedding that is a function of solely the informative $\gamma$ and invariant to $\theta$, effectively recovering the manifold identifier $\gamma$. An important consequence of this behavior is one-shot learning to unseen classes.
\end{abstract}

\section{Introduction}
\label{sec:intro}

\ificml
\begin{figure*}[!ht]
\begin{center}
 \begin{minipage}[t]{0.305\linewidth}\vspace{0pt}
    \centering
    \hspace{-0.5cm}\includegraphics[width=3.5cm,height=3.5cm]{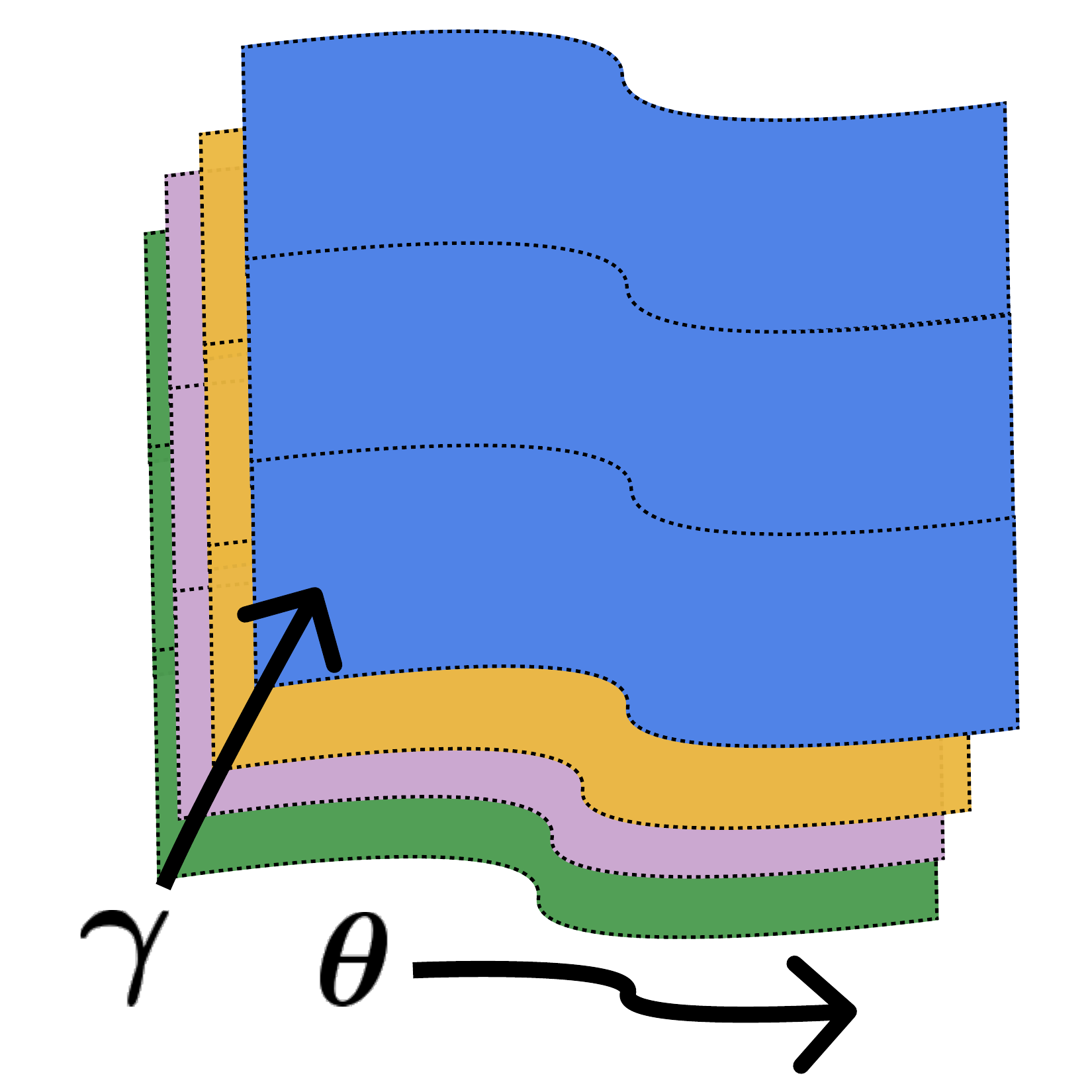}
    \vspace{0.12cm}\caption{An illustration of the data generating process. Each class is comprised of points on a simple manifold and each point is characterized by two latent parameters: $\bgamma, \btheta$. The former determines the manifold the point belongs to, while the later defines the location on the manifold.}
    \label{fig:intro1} 
\end{minipage}
\hspace{0.05cm}
\begin{minipage}[t]{0.37\textwidth}\vspace{0pt}
    \centering
    \hspace{-0.3cm}\includegraphics[width=6.5cm,height=3.5cm]{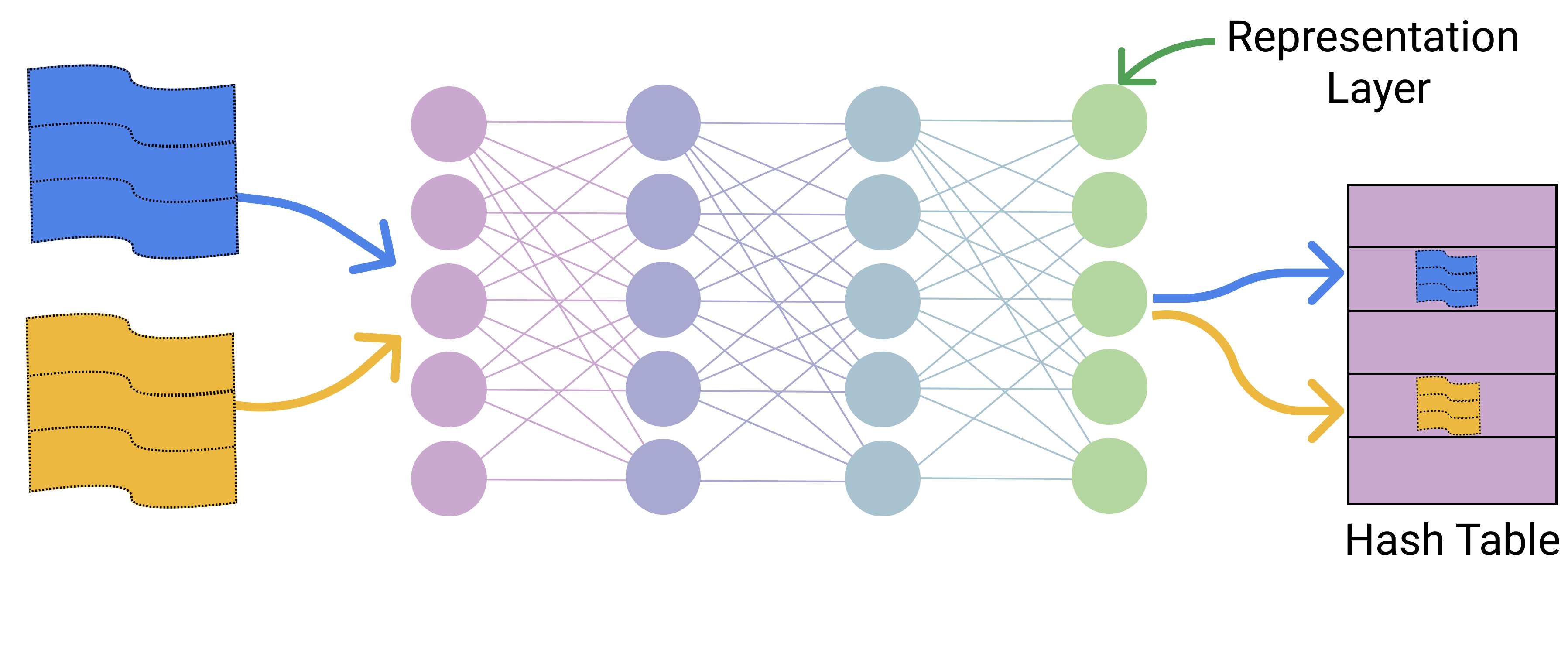}%
    \vspace{3.5pt}\caption{An illustration of a DNN as a Geometry Sensitive Hash (GSH) function. Every two points on the same manifold map to (approximately) the same representation (i.e., the penultimate layer feature map), while every two points from different manifolds go to far away representations.} 
    \label{fig:intro2} 
\end{minipage}
\hspace{0.05cm}
\begin{minipage}[t]{0.30\textwidth}\vspace{0pt}
    \centering
    \includegraphics[width=4.7cm,height=3.7cm]{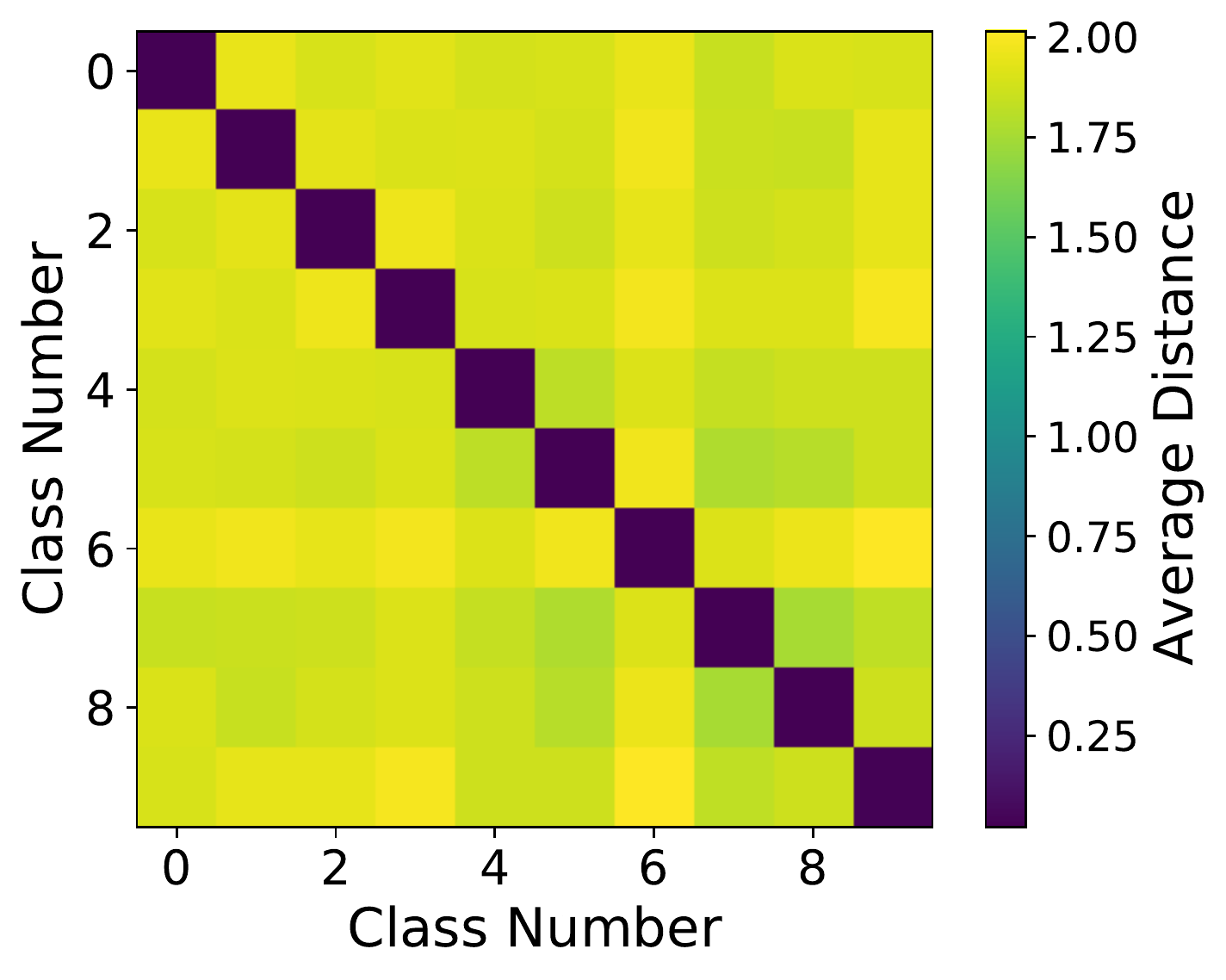}
    \vspace{-0.08cm}\caption{A confusion matrix of intra (same manifold) vs inter (different manifolds) $\ell_2$-\textbf{distances} of representations for an MLP trained on synthetic data (see \cref{sec:experiments-body} for details). Notably, the intra distances are close to zero on average, suggesting this model is a GSH function.}
    \label{fig:intro3} 
\end{minipage}
\end{center}
\end{figure*}  
\fi




Deep Neural Networks (DNNs) are commonly used for mapping complex objects to useful representations that are easily separable in the embedding space \cite{krizhevsky2012imagenet, lecun1995convolutional}. However, what features are captured by the representation and what information is stripped away remains a mystery. 
As a running example, consider a network for image classification. Each class of images can be viewed as a set of transformations (e.g., different rotations, backgrounds, poses of the object, lighting conditions) on some canonical representative object \cite{dicarlo2007untangling, bengio2012deep}. For example, in a video clip of a dog, we can think of the first frame as the canonical pose and every subsequent frame as a different point on the induced dog-manifold.
We think of all such transformations as producing points on a fixed manifold; which uniquely defines a class. Furthermore, other classes for different animals which start with different canonical images on which the {\em same} set of transforms are applied produce a collection of manifolds, one for each class, with a {\em shared} geometry.

\ifarxv
\begin{figure*}[ht]
\begin{center}
 \begin{minipage}[t]{0.30\linewidth}\vspace{0pt}
    \centering
    \hspace{-0.4cm}\includegraphics[width=3.5cm,height=3.5cm]{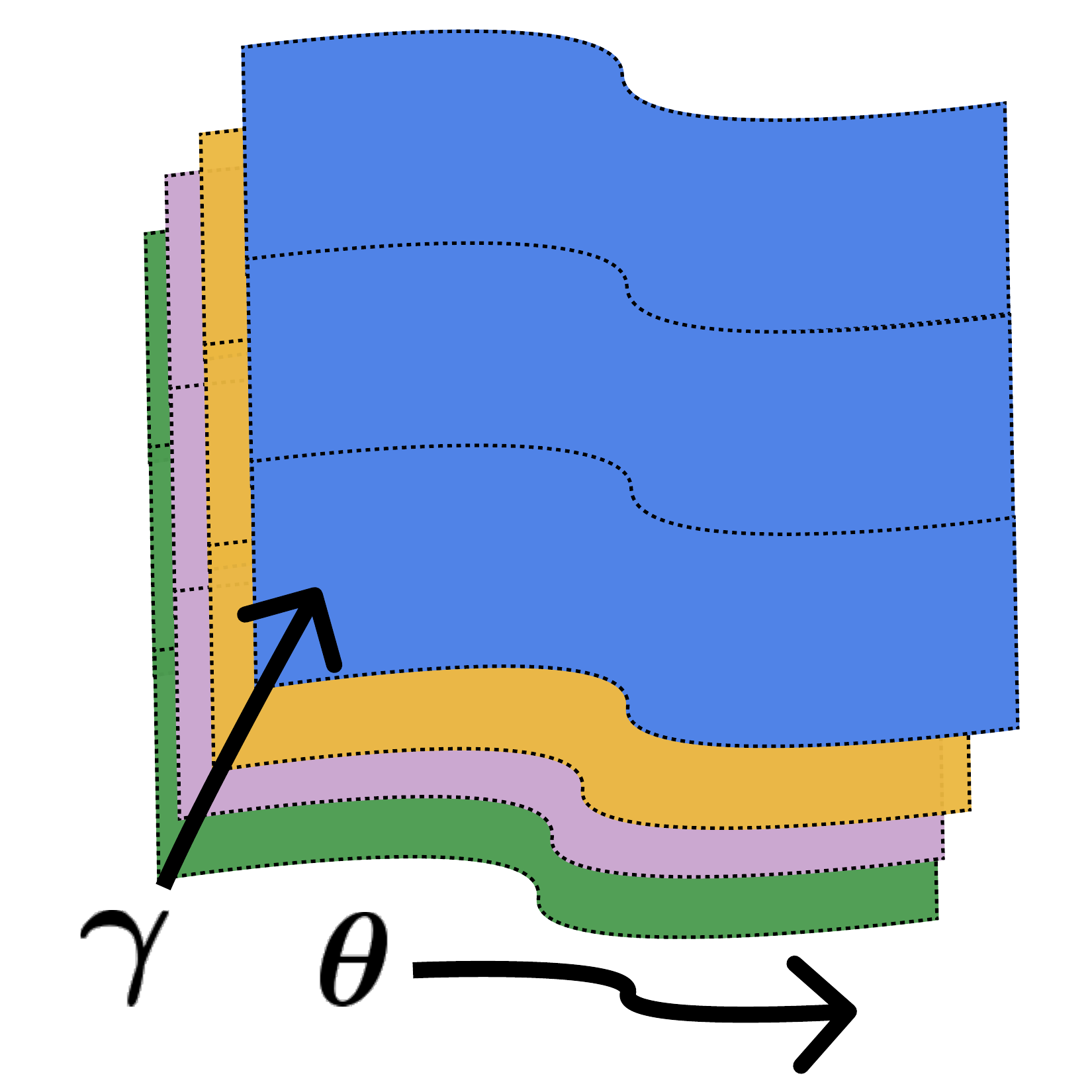}
    \caption{An illustration of the data generating process. Each class is comprised of points on a simple manifold and each point is characterized by two latent parameters: $\bgamma, \btheta$. The former determines the manifold the point belongs to, while the later defines the location on the manifold.}
    \label{fig:intro1} 
\end{minipage}~\hspace{0.3cm}
~\begin{minipage}[t]{0.34\textwidth}\vspace{0pt}
    \centering
    \hspace{-0.9cm}\includegraphics[width=6.5cm,height=3.5cm]{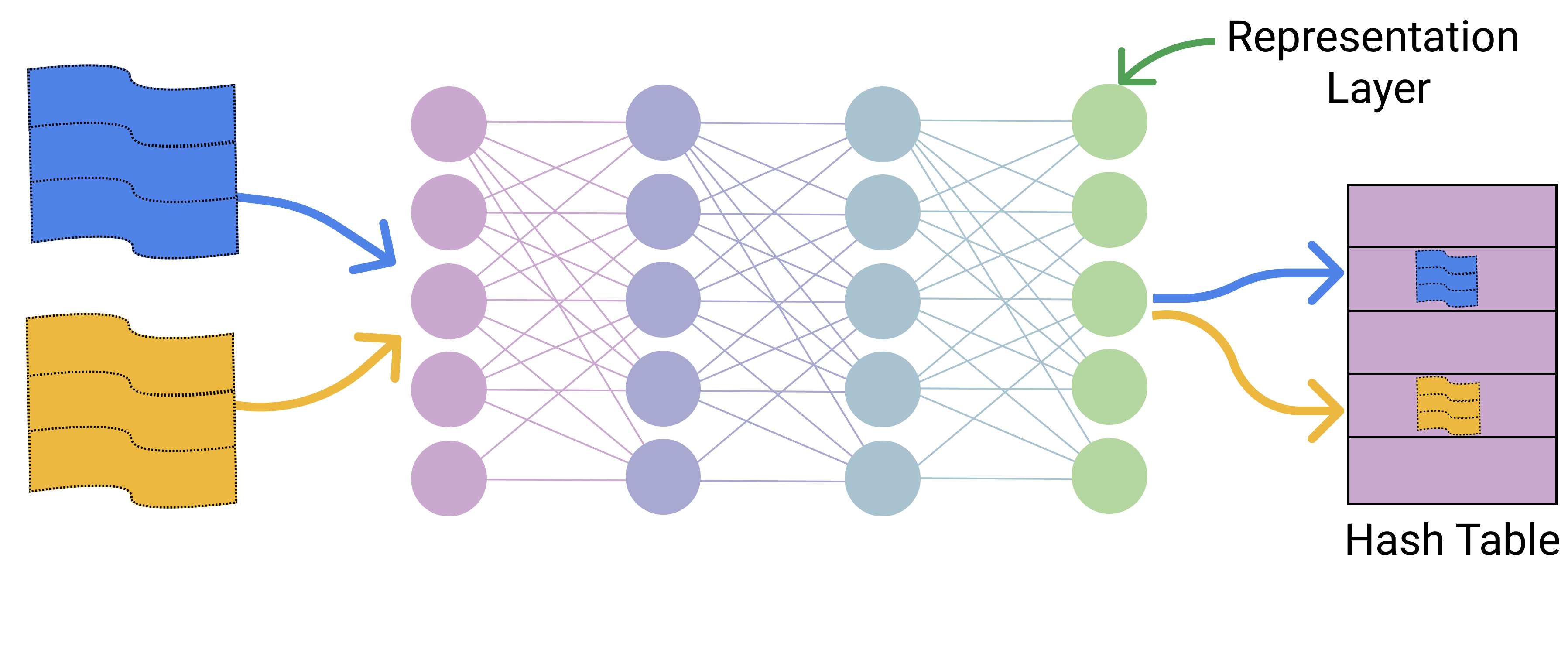}%
    \caption{An illustration of a DNN as a Geometry Sensitive Hash (GSH) function. Every two points on the same manifold map to (approximately) the same representation (i.e., the penultimate layer feature map), while every two points from different manifolds go to far away representations.} 
    \label{fig:intro2} 
\end{minipage}~\hspace{0.3cm}~
\begin{minipage}[t]{0.30\textwidth}\vspace{0pt}
    \centering
    \includegraphics[width=4.7cm,height=3.7cm]{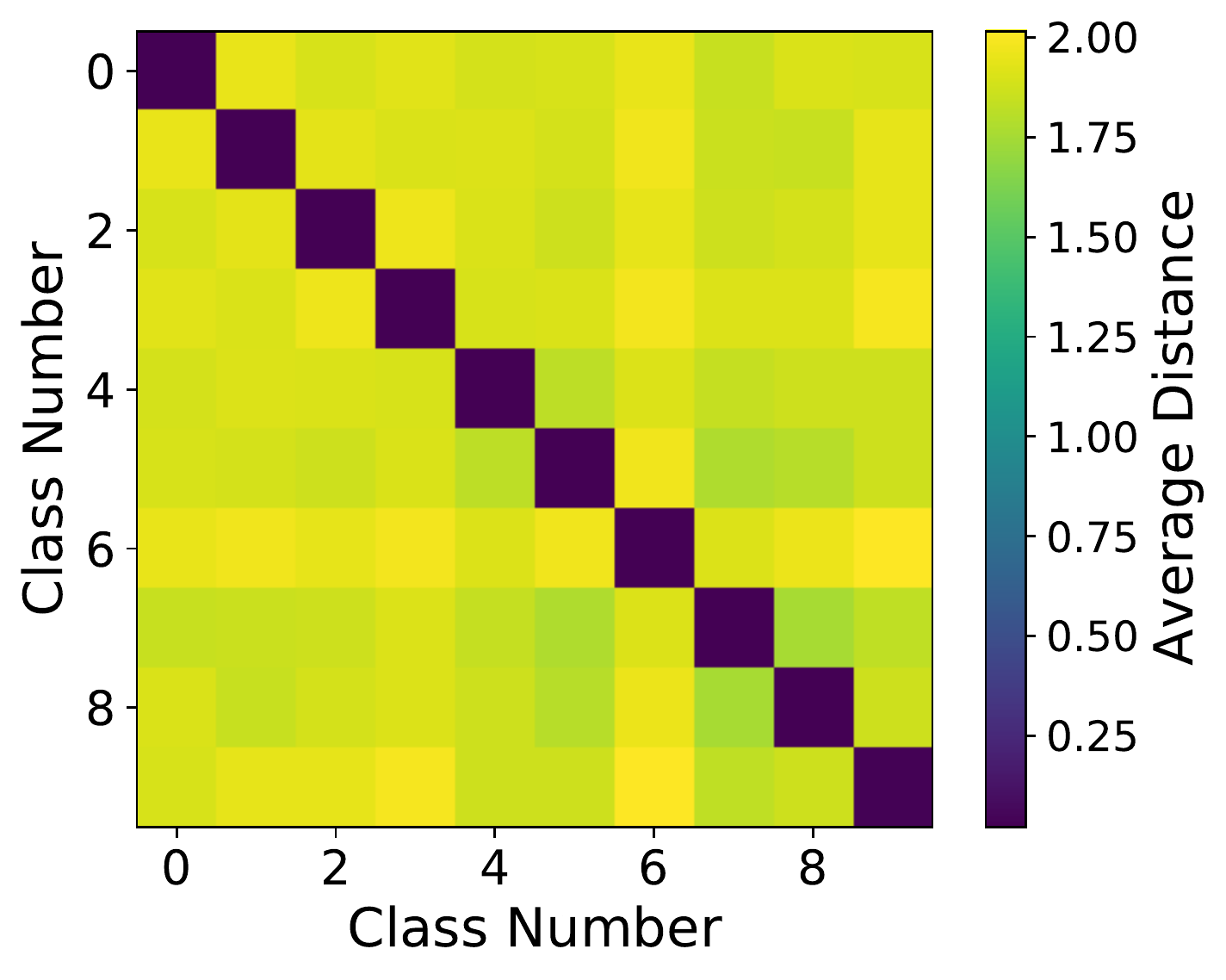}
    \vspace{-0.25cm}\caption{A confusion matrix of intra (same manifold) vs inter (different manifolds) $\ell_2$-\textbf{distances} of representations for an MLP trained on synthetic data (see \cref{sec:experiments-body} for details). Notably, the intra distances are close to zero on average, suggesting this model is a GSH function.}
    \label{fig:intro3} 
\end{minipage}
\end{center}
\end{figure*}  
\fi



In this work, we study the problem of understanding this manifold geometry in the supervised learning setting, where we have each object's class while both the canonical object and the set of transformations are unknown. Specifically, we have access to a sample $\cS=\{(\vec{x}_i, y_i)\}_{i=1}^n$ where each input $\x\in \R^d$ is drawn from a mixture distribution over $m$ manifolds $M_1,...,M_m$ sharing similar topologies (see \cref{fig:intro1}), and each label $y_i\in[m]$ corresponds to the manifold of $\x_i$. Moreover, each point $\vec{x}$ on manifold $M_i$ is characterized by two latent vectors: $\bgamma_i\in \R^s, \btheta\in \R^k$.
\begin{itemize}
    \item $\bgamma_i$ is the manifold identifier (i.e., representing the canonical object) and so there is a one-to-one correspondence between each $\bgamma_i$ and $M_i$. \vspace{0.03em}
    \item $\btheta$ is the transformation (e.g., representing the view or distortion). So if we fix $\bgamma_i$, the manifold $M_i$ can be generated by sampling different values of $\btheta$. 
\end{itemize} 

To gain intuition, we turn to the well studied example of (realizable) clustering. Here, $\bgamma_i$ represents the centroid of a cluster and $\btheta$ represents a small perturbation around the centroid, so then the manifold $M_{\bgamma}$ is comprised of all inputs $\x$ of the form $\{\bgamma+\btheta \mid \|\btheta\| \le \varepsilon\}$. In this setting, a popular algorithmic paradigm is Locality Sensitive Hashing (LSH) which maps each input to a hash  bucket, and in contrast with standard hashing (e.g., in Cryptography) tries to maximize collisions for ``similar" inputs. A good LSH function ensures that: (A) two sufficiently close or ``similar" inputs map to the same bucket and; (B) two sufficiently dissimilar inputs map to different buckets.

Beside the obvious benefit of properly clustering close points, an LSH function also allows clustering points from unseen clusters. I.e, for a new point belonging to a \emph{cluster outside the training set}  centered around $\bgamma_{m+1}$, an LSH algorithm will designate a new bucket and will map closeby points to that bucket. In Machine Learning terminology, this property is often referred to as ``few-shot learning" as a minimal number of labels from unseen manifolds are needed.  

However, it is unclear how to address the problem when the manifold geometry is more involved. In this paper, we consider a family of manifolds with a shared geometry defined by a set of analytic functions. We {\em prove} that DNNs with appropriate regularization, exhibit LSH-like behavior on this family of manifolds. More precisely, we show that the penultimate layer $r$, also known as the ``representation layer", of an appropriately trained network, will satisfy the following property:  
\begin{definition}[Geometry Sensitive Hashing (GSH), informal] We say $r$ is a GSH function with respect to a set of manifolds if (See \cref{fig:intro2} for illustration.): 
\begin{enumerate}[label=(\Alph*)]
    \item For every two points on the same manifold $\x_1, \x_2 \in M$, $\|r(\x_1) - r(\x_2)\|$ is small. 
    \item For every two points on two well separated manifolds $\x_1\in M_1 $ and $\x_2 \in M_2$, $\|r(\x_1) - r(\x_2)\|$ is large.
\end{enumerate}
\end{definition}
This suggests, that DNNs whose representations satisfy the GSH propery, capture the shared manifold geometry in a manner similar to how LSH functions capture spatial locality.
Note that having GSH in the representation is stronger than at the output layer (viewing the output layer as a feature map), which is an immediate consequence of having a small loss on test examples. GSH on the representation layer additionally implies few-shot learning for unseen manifolds (under the same shared geometry) which is more powerful. 



\textbf{Usefulness of Recovering $\bgamma$.}
An additional question one might be interested in is whether we can recover the latent vector $\bgamma$ via a simple transform of the representation computed by a GSH function. Often, $\bgamma$ represents a combination of semantic concepts such as having four legs or a tail which implies that recovering it confers interpretability to the decision-making model and encourages a modular design of systems which pass around representations computed in one task to downstream tasks. We give theoretical and empirical evidence to support the recoverability of $\bgamma$ via simple transformations on top of representations learnt by DNNs which behave as GSH functions.

\textbf{Our contributions.} Our main contributions in a nutshell:
\begin{itemize}
    \item We suggest a new generalization of Locality Sensitive Hashing---Geometry Sensitive Hashing, i.e., functions whose output is sensitive to the manifold, yet is invariant to the location along the manifold and show that properly trained DNNs will be GSH functions. Moreover, these DNNs can recover $\bgamma$ up to a linear transform thereby recovering the manifold geometry. \vspace{0.04em}
    \item We show that, under appropriate assumptions on the manifold class, the GSH property holds for representations computed by DNNs; this is not only for manifolds seen during train time, but also for manifold never seen before. This offers an explanation for why DNNs are effective one-shot learners. This could be an important first step towards better understanding Transfer Learning. Moreover, the size of our network for Transfer Learning is largely independent of the number of manifolds we wish to transfer onto, as it is a fixed size representation layer followed by a hash-table lookup.\vspace{0.04em}
    \item We empirically corroborate our findings by training DNNs on real and synthetic data (see \cref{fig:intro3}) and demonstrate that even for real datasets such as MNIST and CIFAR10, where the underlying manifolds do not satisfy the assumptions for our proof, the GSH property still holds to a certain extent.
\end{itemize}

\ifarxv \subsection{Related Work} \fi
\ificml \textbf{Related Work.} \fi 
There is a rich history of works that study classification problems as manifold learning---certain notable proposals for learning manifolds include \cite{tenenbaum2000global, belkin2006manifold} and others such as  \cite{hein2005intrinsic,hein2006uniform} study the problem of manifold density estimation. 
Deep networks are commonly used to create compact representations  \cite{bengio2013representation} of complex inputs such as text, images, objects and these representations are commonly used to compare the underlying objects and transfer to new classification problems \cite{weiss2016survey,sung2018learning}. However there is little theoretical understanding of the neural representations computed by such networks. \cite{arora2019theoretical} offer theoretical insights on contrastive learning, a popular method for unsupervised learning.
Works including \cite{maurer2016benefit,du2020few,tripuraneni2020theory} develop a theoretical understanding of transfer learning by modeling a collection of tasks with shared parameters.
Our theoretical results build on recent work expositing the benefits of wide non-linear layers \cite{daniely2016toward} and overparameterized networks \cite{arora2019fine,allen2018learning}. It is also closely related to a set of works which explore the loss landscape of linear neural networks showing that all local minima are global \cite{ge2016matrix,kawaguchi2016deep}. We use the concept of Locality Sensitive Hashing for which we refer the reader to \cite{wang2014hashing} for a survey of the area. 
Another related work to ours which has the same high level goal of computing representations invariant under noisy transformations is the work of \cite{arjovsky2019invariant}. The connection between DNNs and Hash Functions was explored before (e.g., see \ificml \citet{he2021neuralcode, wang2017survey} \fi \ifarxv \cite{he2021neuralcode, wang2017survey} \fi  and references therein). While previous works focus on empirical studies, we are able to \emph{prove} that GSH holds for certain architectures under the manifold data assumption.

\ificml \textbf{Notational Preliminaries} \fi
\ifarxv \subsection{Notational Preliminaries} \fi
\label{sec:prelim}
We use $[n]$ to denote $\{1,2, \ldots, n\}$. Boldface letters are used for vectors and capital letters mostly denote matrices. We use $\vec{x}^\top\vec{y}$ or $\langle \vec{x}, \vec{y} \rangle$ to denote the inner product of $\vec{x}$ and $\vec{y}$. We use some standard matrix norms:$\|A\|_F =  \sqrt{\sum_{i,j}
A_{ij}^2}$ is the Frobenius norm, $\|A\|_2$ is the operator or spectral norm which equals the largest singular value of $A$, $\|A\|_*$ is the nuclear norm which is the sum of the singular values. $\vec{x} \odot \vec{y}$ denotes the vector obtained by point-wise multiplication of the coordinates of $\vec{x}$ and $\vec{y}$ and a similar notation is used for entry-wise multiplication of two matrices as well. Given a matrix $A \in \mathbb{R}^{m \times n}$ such that $\svd(A) = USV^\top$, we define $A^{1/2} = US^{1/2}V^\top$. $S^{d-1}$ denotes the $d$-dimensional unit sphere.
For some additional preliminaries see Appendix~\ref{sec:apx-prelim}.
\section{A Formal Framework for GSH}
\label{sec:setting}
We consider manifolds which are subsets of points in $\mathbb{R}^d$. Every manifold $M_{\bgamma}$ has an associated latent vector $\bgamma \in \mathbb{R}^s, s \le d$ which acts as an identifier of $M_{\bgamma}$. The manifold is then defined to be the set of points $\vec{x} = \vec{f}(\bgamma, \btheta) = (f_1(\bgamma, \btheta),  \ldots, f_d(\bgamma, \btheta))$ for $\btheta\in\Theta \subseteq \mathbb{R}^{k}, k<d$. Here, the manifold generating function $\vec{f} = \{f_i(\cdot,\cdot)\}_{i=1}^d$ where the $f_i$ are all analytic functions. $\btheta$ acts as the ``shift" within the manifold. Without significant loss of generality, we assume our inputs $\vec{x}$ and $\bgamma$s are normalized and lie on  $S^{d-1}$ and $S^{s-1}$, the $d$ and $s$-dimensional unit spheres, respectively.
When the $f_i$s are all degree-1 polynomials we call the manifold a linear manifold. An example of a linear manifold is a $d-1$-dimensional hyperplane. 
Given the above generative process, we assume that there is a well-behaved analytic function to invert it.
\begin{assumption}[Invertibility]
\label{ass:main}
There is an analytic function $\vec{g}(\cdot) : \mathbb{R}^d \to \mathbb{R}^s$ with bounded norm Taylor expansion s.t. for every point $\vec{x} = \vec{f}(\bgamma, \btheta)$ on $M_{\bgamma}$, $\vec{g}(\vec{x}) = \bgamma$. 
\end{assumption}
Our definition of the norm of an analytic function is a bit technical and we defer it to Section~\ref{sec:addnl-prelim} of the appendix. For some intuition on this, the function $f(\vec{x}) = e^{\bbet_1 \cdot \x} \cdot \sin(\bbet_2 \cdot\x) + \cos(\bbet_3 \cdot \x)$ will have a constant norm if $\bbet_1, \bbet_2, \bbet_2$ all have a constant $\|\cdot\|_2$ norm.

\textbf{Train Data Generation.}
Next we describe how we get our train data. 
As described above, a set of analytic functions $\{f_i\}$ and a vector $\bgamma$ together define a manifold. We then consider a shared geometry among manifolds defined by a fixed set of $\{f_i\}$. A distribution $\cM$ over a class of manifolds $\supp(\cM)$ (given by the $\{f_i\}$) is then generated by having a set $\Gamma$ from which we sample $\bgamma$ associated with each manifold. We assume that all manifolds within $\supp(\cM)$ are well-separated. Formally, for any two manifolds $M_1, M_2 \in \supp(\cM)$, we will assume that $\vec{\bgamma}_1^\top \vec{\bgamma}_2 \le \tau$ where $\tau < 1$ is a constant\footnote{In particular this holds with high probability for randomly sampled vectors on the unit sphere (Section~\ref{sec:addnl-prelim}).}. Such a manifold distribution will be called $\tau$-separated. To describe a distribution of points over a given manifold $M$ we use the notion of a point density function $\cD(\cdot)$ which maps a manifold $M$ to a distribution $\cD(M)$ over the surface of $M$.
Training data is then generated by first drawing $m$ manifolds $M_1,\ldots, M_m \sim \cM$ at random. Then for each $l \in [m]$, $n$ samples $\{(\vec{x}_i^l, \vec{y}_i^l) \}_{i=1}^n$ are drawn from $M_l$ according to the distribution $\cD(M_l)$. Note that for convenience, we view the label $\vec{y}_i^l$ as a one-hot vector of length $m$ indicating the manifold index. The learner's goal is then to learn a function which takes in these $n\times m$ pairs of $(\vec{x},\vec{y})$ as input and is able to correctly classify which manifold a new point comes from. In other words, we wish to compute a mapping that depends on $\bgamma$ but does not have a dependence on $\btheta$.
With the above notation, we now formally define GSH.
\begin{definition}[Geometry Sensitive Hashing (GSH) ]
\label{def:hashing-property}
Given a representation function $r: \mathbb{R}^d \to \mathbb{R}^p$, and a distribution over a manifold class $\cM$, we say that $r$ satisfies the $(\eps,\rho)$-hashing property on $\cM$ with associated point density function $\cD$ if, for some $\rho > 1, \eps>0$,
\usetagform{blue}
\begin{align}
\label{eq:hash-near}
    V_M(r) = \E_{\vec{x} \sim \cD(M)}\left[ \left\| r(\vec{x}) - \E_{\vec{x} \sim \cD(M)}[r(\vec{x})]\right\|_2^2  \right] \le \epsilon, \tag{A} 
\end{align}
for all $M \in \supp(\cM)$. The above states that the variance of the representation across examples of a manifold is small.
Moreover, for two distinct $\tau$-separated manifolds $M_1$ and $M_2$ sampled from $\mathcal{M}$, the corresponding representations need to be far apart. That is,
\usetagform{blue}
\begin{align}
    \E_{\vec{x}_1 \sim \cD(M_1), \vec{x}_2 \sim \cD(M_2)}[\|r(\vec{x}_1) - r(\vec{x}_2)\|_2^2] \ge \rho\epsilon. \tag{B} \label{eq:hash-far}
\end{align}
\end{definition}
\usetagform{default}


Our main contribution is showing theoretically and empirically that deep learning on manifold data can produce a network where the representation layer is a GSH function for most manifolds from the manifold distribution.
Under an appropriate loss function and architecture (see Section~\ref{sec:theory}), we prove the following Theorem (which is an informal version of Theorems~\ref{thm:linear-main} and \ref{thm:var-reg-main}).
\begin{theorem} [(Informal) GSH holds for Most Manifolds from $\cM$]
\label{thm:informal-main}
Suppose $\cM$ is a distribution on $\tau$-separated manifolds, for some constant $\tau$. For any $\eps > 0$, there is a neural network of size $\poly(m,n,1/\eps)$ which when trained on an appropriate loss on $n$ points sampled from each of $m$ manifolds drawn from $\cM$ gives a representation which satisfies the $(\eps,\rho)$-hashing property with high probability over unseen manifolds in $\cM$, for $\rho = \Omega(1/\eps)$ when $m, n = \Theta\left(\frac{s^{O(\log(1/\eps))}}{\eps^2}\right)$.
\end{theorem}
Our network operates in the over-parameterized setting, i.e. the number of parameters is of the order of the number of train examples.
One immediate consequence of having the GSH property is transfer learning to unseen manifolds, as captured by the following.
\begin{theorem}[(Informal) GSH Property Implies One-Shot Learning]
\label{thm:hashing-implies-oneshot}
Given a distribution $\cM$ over $\tau$-separated manifolds, if a representation function $r(\cdot)$ satisfies the $(\eps,\rho)$-GSH property over $\cM$, for a small $\eps$ and a large enough $\rho$, then we have one-shot learning. That is there is a simple hash-table lookup algorithm $\cA$ such that it learns to classify inputs from manifold $M_{new} \sim \cM$ with  \textbf{just one example} with probability $\ge 1- \delta$.
\end{theorem}
In addition we show evidence for exact recoverability of $\bgamma$ in some settings.
\begin{remark}
\label{rem:gamma-recover}
GSH implies that the representation we have computed is an isomorphism to the manifold identifier $\bgamma$. We observe empirically a simple linear transform that maps this isomorphism to exactly $\bgamma$. In addition, we show theoretically as well that we are able to recover $\bgamma$ exactly albeit only for examples on our train manifolds (Section~\ref{sec:recovering-gamma}).
\end{remark}

We show Theorems~\ref{thm:informal-main} and \ref{thm:hashing-implies-oneshot} on a 3-layer NN which is described in Section~\ref{sec:theory}. We run experiments on networks with more layers and find that the GSH property holds for deeper architectures on synthetic data and, to a lesser extent, on real-world image datasets such as MNIST and CIFAR-10.

The next four sections break down the overview of our proof of Theorem~\ref{thm:informal-main}: Section~\ref{sec:theory} sets the ground for our theory, Section~\ref{sec:arch-theory} presents the relevant properties of our architecture, Section~\ref{sec:optimization-theory} analyses our loss objective to show an empirical variant of the GSH, finally Section~\ref{sec:generalization-body} is about generalizing from the empirical variant to the population variant. All four put together give us the Theorem~\ref{thm:informal-main}.
Finally in Section~\ref{sec:experiments-body} we present our experimental findings.

\section{Theoretical Results}
\label{sec:theory}
We start by describing the neural architecture for our proof.

\textbf{Our Architecture.}
We consider a 3-layer neural network $\hat{\vec{y}} = AB\sigma(C\vec{x})$, where the input $\vec{x} \in \mathbb{R}^d$ passes through a wide randomly initialized fully-connected {\em non-trainable} layer $C \in \mathbb{R}^{D \times d}$ followed by a ReLU activation $\sigma(.)$\footnote{Our results hold for more general activations. The required property  of an activation is that its dual should have an `expressive' Taylor expansion. E.g., the step function or the exponential activation also satisfy this property. See \cite{daniely2016toward}.}. Then, there are two trainable fully connected layers $A \in \mathbb{R}^{m \times T},B \in \mathbb{R}^{T \times D}$ with {\em no} non-linearity between them.
Each row of $C$ is drawn i.i.d. from $\cN(\vec{0}, \frac 1D I)$. It follows from random matrix theory that $\|C\|_2 \le 4$ w.p. $\ge 1-\exp(-O(D))$ (Section~\ref{sec:addnl-prelim}).
This choice of architecture is guided by recent results on the expressive power of over-parameterized random ReLU layers \cite{daniely2016toward, arora2019fine,allen2018learning} coupled with the fact that the loss landscape of two layer linear neural networks enjoys nice properties \cite{ge2016matrix, gunasekar2018implicit}.

\textbf{Additional Notation.} We use $\vec{z}$ to denote $\sigma(C\vec{x})$. For succinctness, we define $X_l$ to be the matrix whose columns are $\{\vec{x}_{il}\}_{i=1}^n$, $Z_l$ is the matrix whose columns are $\{\vec{z}_{il}\}_{i=1}^n$. Given the label vectors $\vec{y}_{il}$ and predictions made by our model $\vec{\hat{y}}_{il}$ we define $Y_l$ and $\hat{Y}_l$ similarly. We let $X$ be the rank-3 tensor which is obtained by stacking the matrices $X_l$ for $l \in [m]$. Tensors $Y, \hat{Y}, Z$ are defined similarly.
In many places we compute a mix of empirical averages over two distributions (i) the $m$ train manifolds (ii) the $n$ data points from each of the $m$ train manifold. Given a function $f(\vec{x})$ operating on an input from a manifold, let $ \E_{n}[f(\vec{x}_l)| \bgamma_l] = \frac{1}{n}\sum_{i=1}^n f(\vec{x}_{il})$ and given a function $g(\bgamma)$ operating on a manifold let $ \E_{m}[t(\bgamma)] = \frac{1}{m}\sum_{l=1}^m t(\vec{\bgamma}_l)$. With this additional notation, we describe our objective.

\textbf{Our Loss function.}
Given the one-hot label vectors $\vec{y}$ and the predictions $\vec{\hat y}$ made by our model we aim to minimize a weighted square loss averaged across the $m$ train manifolds.
\begin{align}
    \cL_{A,B}(Y, \hat Y) &= \frac{1}{m}\sum_{l=1}^m\E_n\left[\left\|\vec{w}_l \odot (\vec{y}_{l} - \vec{\hat{y}}_{l}) \right\|_2^2 \;\middle|\; \bgamma_l\right] \notag\\
    &= \E_m\left[\left\|W_l \odot (Y_l - \hat{Y}_l)\right\|_F^2 \right], \label{eq:weighted-square-loss}
\end{align}
where $\vec{w}_l$ is a weighting of different coordinates of $\vec{y}-\vec{\hat{y}}$ such that $\vec{w}_{lj} = 1/2$ if $j=l$ and $1/2(m-1)$ otherwise. Each example serves as a positive example for the class corresponding to its manifold and is a negative example for all the other $m-1$ classes. The weighting by $\vec{w}_l$ ensures that the total weight on the positive and negative examples is balanced and helps exclude degenerate solutions such as the all $0$s vector from achieving a low loss value. We show in Section~\ref{sec:addnl-prelim} that a small value of our weighted square loss implies a small $0/1$ classification error and vice versa.
We add $\ell_2$ regularization on the weight matrices $A$ and $B$ to this loss. The objective is then,
\begin{align}
    \cL_{A,B}(Y, \hat Y) + \|A\|_F^2 + \|B\|_F^2, \label{eq:weighted-square-loss-with-reg}
\end{align}
When we deal with non-linear manifolds which are harder to analyze, we will require an additional component in our regularization which we term \emph{variance regularization} (see Section~\ref{sec:var-reg-body}).

\textbf{Empirical Variant of Intra-Manifold Variance $V_M$.}
In the subsequent sections an empirical average of $V_M(r)$, the variance of representation $r$ over points from $M$, across manifolds will be of importance. We define it here. Given any function $r(\cdot)$ of $\vec{x} \in \mathbb{R}^d$,
\begin{align}
    \hat{V}_{mn}(r) = \E_{m} \left[\E_n\left[ \left\|r(\vec{x}_l) - \E_n[r(\vec{x}_l) ] \right\|_2^2 \middle| \bgamma_l\right] \right]
\end{align}

\textbf{A Note on Optimization Algorithms.}
Standard optimization algorithms such as gradient descent or stochastic gradient descent are theoretically shown to converge arbitrarily close to a local optimum point even for a non-convex objective. That is, they avoid second-order stable points (saddle points) with high probability \cite{ge2015escaping,jin2018accelerated,lee2019first} for Lipschitz and smooth objectives. Relying on this understanding, we focus our theoretical analysis on understanding the properties of the local minima.
One can choose the hyper-parameters of these training algorithms as a function of the Lipschitzness and smoothness properties of the training objective (see \cite{bubeck2014convex}; Appendix~\ref{sec:apx-opt-alg}).

\textbf{Proof  Overview}
We give an overview of the proof of Theorem~\ref{thm:informal-main}. A wide random ReLU layer enables us to approximately express arbitrary analytic functions $\bgamma = g(\x)$ as linear functions of the output of the ReLU layer (Lemma~\ref{claim:final-repre})---in fact we show that a wide random ReLU layer is ``equivalent" to a kernel that produces an infinite sequence of monomials in $\x$ upto an orthonormal rotation. So by approximating the desired outputs $Y$ as analytic functions of $\bgamma$ we get that $Y \approx W \sigma(C \x)$ for some $W$. Next, since we have two layers $A,B$ above the ReLU layer, it is possible to get a factorization $W = AB$ such that multiplication of $\vec{z} = \sigma(C \x)$ by $B$ drops any dependence on $\btheta$ and only depends on $\bgamma$---this ensures that for that choice of $A,B$ the representation $r(\x) = B\vec{z}$ is independent of $\btheta$ (Lemma~\ref{lem:good-ground-truth}). Further given the type of regularization we impose, it turns out to be optimal to make the output of the $B$ layer depend only on $\bgamma$ and in such a way that $\|B\|_F$, which depends on a norm bound on the inverting function $g()$, remains bounded and independent of $m, n$ (even though the number of parameters in $B$ grows with $m, n$); similarly the average norm of $A$ per output, $\|A\|_F/m$, can also be made constant. We then use Rademacher complexity arguments to show that if the number of training inputs per manifold $n$ is larger than a quantity that depends on $\|B\|_F$, then the GSH property holds not just for the training inputs but for most of  the manifold. Another set of Rademacher complexity arguments show that if $m$ is larger than a certain value that depends on $\|B\|_F$ the hashing property will generalize to most new manifolds (Lemmas~\ref{lem:common-transfer-m} and \ref{lem:hash-far-gen-m}).

\section{Properties of the Architecture}
\label{sec:arch-theory}
Recall that $\|\vec{x}\|_2 = 1$ for all inputs. We append a constant to $\x$ to get  $\vec{x}' = (\vec{x}/\sqrt{2}, 1/\sqrt{2})$. This added constant enables a more complete kernel representation of our random ReLU layer which will help our analysis.
Given~\eqref{eq:weighted-square-loss-with-reg}, we show that there exists a ground truth network which makes both the loss and the regularizer terms small. Moreover, the representation computed by this ground truth is a GSH function. This is a key component of our proof. 
\begin{lemma}[Existence of a Good Ground Truth]
\label{lem:good-ground-truth}
Suppose $\|\x_{il}\|_2 = 1$ for all $i \in [n], l \in [m]$. Then, exist ground truth matrices $A^*, B^*$ such that for any $0 < \eps \le 1/2$, \vspace{-0.030em}
\begin{enumerate}
    \item $\cL_{\hat{A},\hat{B}}(Y, \hat Y) \le \eps$,\vspace{-0.03em}
    \item $\|A^*\|_F^2 \le m$, $\|B^*\|_F^2 \le \beta = s^{O(\log(1/\eps))}$,\vspace{-0.03em}
    \item $B^*\sigma(C.)$ satisfies $(\eps,\Omega(1/\eps))$-GSH.\vspace{-0.03em} 
    \item Hidden layer width $T = O\left(\log(mn)\log(1/\delta)\eps^{-1}\right)$.\vspace{-0.045em}
\end{enumerate}
\end{lemma}
To show that the weighted square loss and the regularizer terms are small, we lean on insights from Section~\ref{sec:kernel-layer} which presents the power of having a random wide ReLU layer as our first layer.
Once we have the bounds on $\cL_{\hat{A},\hat{B}}(Y, \hat Y)$ and $\|A\|_F$, property~\eqref{eq:hash-far} for our representation follows. Our choice of $B^*$ will have a small intra-class representation variance averaged over the train manifolds giving us property~\eqref{eq:hash-near}. Finally to get a bound on the number of columns in $A^*$, we use the observation that given an $A^*$ with a large number of columns we could use a random projection to project it down to a smaller matrix without perturbing $A^*$'s output by much.

\subsection{Kernel View of a Non-Linear Random Layer}
\label{sec:kernel-layer}
In this section, we state the powerful kernel properties of a wide random ReLU layer.
The key property we show is the following.

\begin{claim}
\label{claim:final-repre}
For any $\eps, \delta>0$, and for $k \ge O(1/\eps^2)$ if the width  $D \ge \Theta\left(\frac{\sqrt{mn}\log(mn/\delta)}{\eps}\right)$, then, w.h.p. there exists an orthonormal matrix $U$, and $\|\Delta\|_F < \eps$, s.t., for the train tensor $X$ viewed as an $\mathbb{R}^{d \times mn}$ matrix, for all columns $i$, 
\begin{align*}
      \hspace{-0.4 em}\sigma(CX_i)= U \Bigg(\sqrt{\frac{1}{2\pi}}, \sqrt{\frac{1}{4}}X_i^{\otimes1}, \ldots,      O\left(\frac{1}{k^{3/2}}\right)X_i^{\otimes k}\Bigg) 
      + \Delta_i.
\end{align*}
where $\vec{x}^{\otimes j}$ is a flattened tensor power $j$ of the vector $\vec{x}$, $X_i, \Delta_i$ are the $i^{th}$ columns of $X$ and $\Delta$ respectively.
\end{claim}

Claim~\ref{claim:final-repre} implies the following lemma which says that a linear function of the output of the random ReLU layer, can approximate bounded-norm polynomials which is used in the proof of Lemma~\ref{lem:good-ground-truth} to get a $B^*$ which approximately computes the manifold inverting function $g(\vec{x})$. The formal version of Lemma~\ref{lem:linear-expressibility} is given in Section~\ref{sec:arch-app} of the appendix.
\begin{lemma}(Informal)
\label{lem:linear-expressibility}
For $\eps, \delta > 0$, and any norm bounded vector-valued analytic function $g : \mathbb{R}^d \to \R$ (for an appropriate notion of norm), w.p. $\ge 1-\delta$ we can approximate $g$ using a random ReLU kernel $\sigma(C\vec{x})$ of width $D \ge \Theta\left(\frac{\sqrt{mn}\log(mn/\delta)}{\eps}\right)$ and a bounded norm vector $\vec{a}$, so that, for each of the $mn$ inputs $\x$, $$\vspace{-0.01em} 
|g(\vec{x}) - \ba\sigma(C\vec{x})|\le \eps. \vspace{-0.05em}$$
\end{lemma}

\section{Properties of Local Minima}
\label{sec:optimization-theory}
In this section, we show that any local minimum of~\eqref{eq:weighted-square-loss-with-reg} has desirable properties.
The first is that for our minimization objective, all local minima are global. Results of this flavor can be found in earlier literature (e.g., \ificml \citet{ge2016matrix} \fi \ifarxv \cite{ge2016matrix}\fi ). We provide a proof in the supplementary material for completeness.

\begin{lemma}[All Local Minima are Global]
\label{lem:var-reg-eqv}
All local minima are global minima for the following objective, where $O(.)$ is any convex objective: \vspace{-0.02em}
\begin{align}
    \min_{A,B} O(AB) + \lambda_1\left(\|A\|_F^2\right) + \lambda_2\left(\|B\|_F^2 \right), \notag
\end{align}
 
\end{lemma}

\vspace{-0.03em}The above lemma together with Lemma~\ref{lem:good-ground-truth} implies that the desirable properties of our ground truth $A^*,B^*$ also hold at the local minima of \eqref{eq:weighted-square-loss-with-reg}. This will follow by choosing the regularization parameters $\lambda_1, \lambda_2$ appropriately. 
\begin{lemma}
\label{lem:small-train-weighted-square-loss}
At any local minima we have that the weighted square loss $\cL_{\hat{A},\hat{B}}(Y,\hat{Y}) \le 3\eps$. 
\end{lemma}

Next we need to show that the empirical variant of the GSH property holds for the representation $\hat{B}\sigma(C.)$.
Here our approaches for linear and non-linear manifolds differ. Linear manifolds enable a more direct analysis with a plain $\ell_2$-regularization. However, we need to assume certain additional conditions on the input. The result for linear manifolds acts as a warm-up to our more general result for non-linear manifolds where we have minimal assumptions but use a stronger regularizer designed to push the representation to satisfy GSH. We describe these differences in Sections~\ref{sec:linear-body} and \ref{sec:var-reg-body}. 

\subsection{GSH Property on Linear Train Manifolds}
\label{sec:linear-body}
Recall that a linear manifold is described by a set of linear functions $\{f_i\}_{i=1}^d$ which transform $\bgamma, \btheta$ to $\vec{x}$. An equivalent way of describing points on a linear manifold is: $\vec{x} = P\bgamma + Q\btheta$ for some matrices $P, Q$. Without a significant loss of generality we can assume that $P\bgamma \perp Q\btheta$ (Lemma~\ref{lem:basis-change}). Given this, we can regard as our input $\tilde{\vec{x}} = (\bgamma', \btheta')$ where $\btheta' \in \mathbb{R}^k$ and $\bgamma' \in \mathbb{R}^{d-k}$ by doing an appropriate rotation of axes. Here, $\bgamma', \btheta'$ play the role of original $\bgamma, \btheta$ respectively. As before we will assume that $\|\tilde{\vec{x}}\|_2 = 1$. 
We append a constant to $\tilde{\vec{x}}$ as before, increasing the value of it to $O(\sqrt{k})$ for a technical nuance. This constant plays the role of a bias term.
The objective for linear manifolds is then, 
\begin{align}
&\min_{A,B} \cL_{A,B}(Y, \hat{Y}) + \lambda_1\|A\|_F^2 + \lambda_2\|B\|_F^2. \label{eq:linear-loss}
\end{align}
Lemma~\ref{lem:var-reg-eqv} will imply that gradient descent on the above objective reaches the global minimum value. The first step of our argument is Lemma~\ref{lem:var-output-jensen-body} which shows that the loss decreases when the variance of the output vector across examples from a given manifold decreases. This is a simple centering argument using Jensen's inequality. 

\begin{lemma}[Centering]
\label{lem:var-output-jensen-body}
Let $f$ denote the function computed by our neural network. 
Replacing the output $\vec{\hat y} = f(\bgamma,\btheta)$ by $\vec{\hat y'} = \E_{n}[\vec{\hat y}|\bgamma]$ will reduce the (weighted) square loss:
$$\cL(Y, \hat Y') \le \cL(Y, \hat Y) - \hat{V}_{mn}(\vec{\hat y})/2(m-1)\vspace{-0.03em}$$
\end{lemma}

Lemma~\ref{lem:var-output-jensen-body} implies that a smaller variance at the output layer is beneficial. In Section~\ref{sec:linear-app}, we argue that it is in fact beneficial to have zero variance at the representation layer as well.

\ificml
\begin{figure*}[!ht]
\begin{center}
\begin{tabular} {ccc}
  \includegraphics[width=0.32\textwidth]{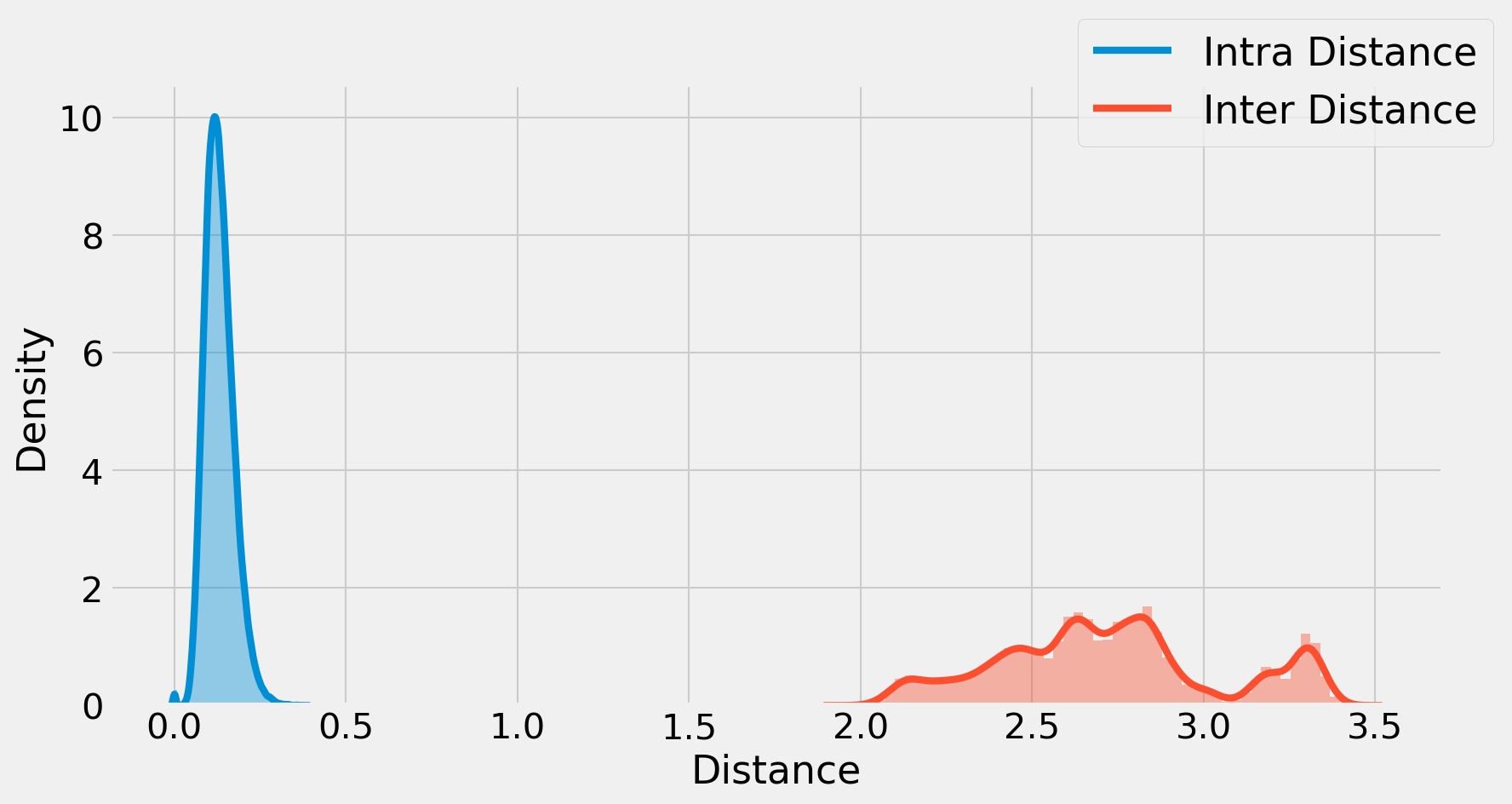} &   \includegraphics[width=0.32\textwidth]{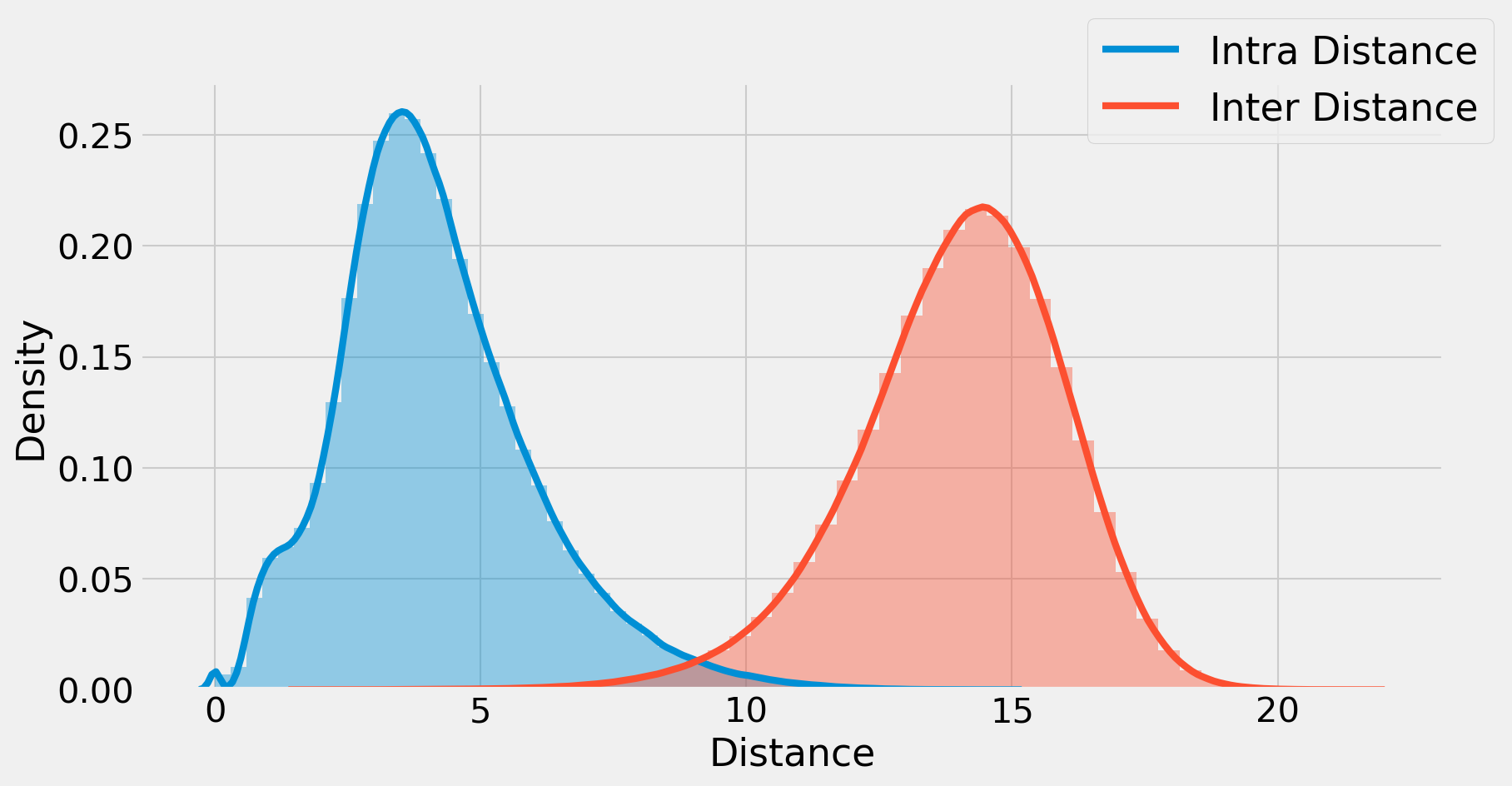} &    \includegraphics[width=0.32\textwidth]{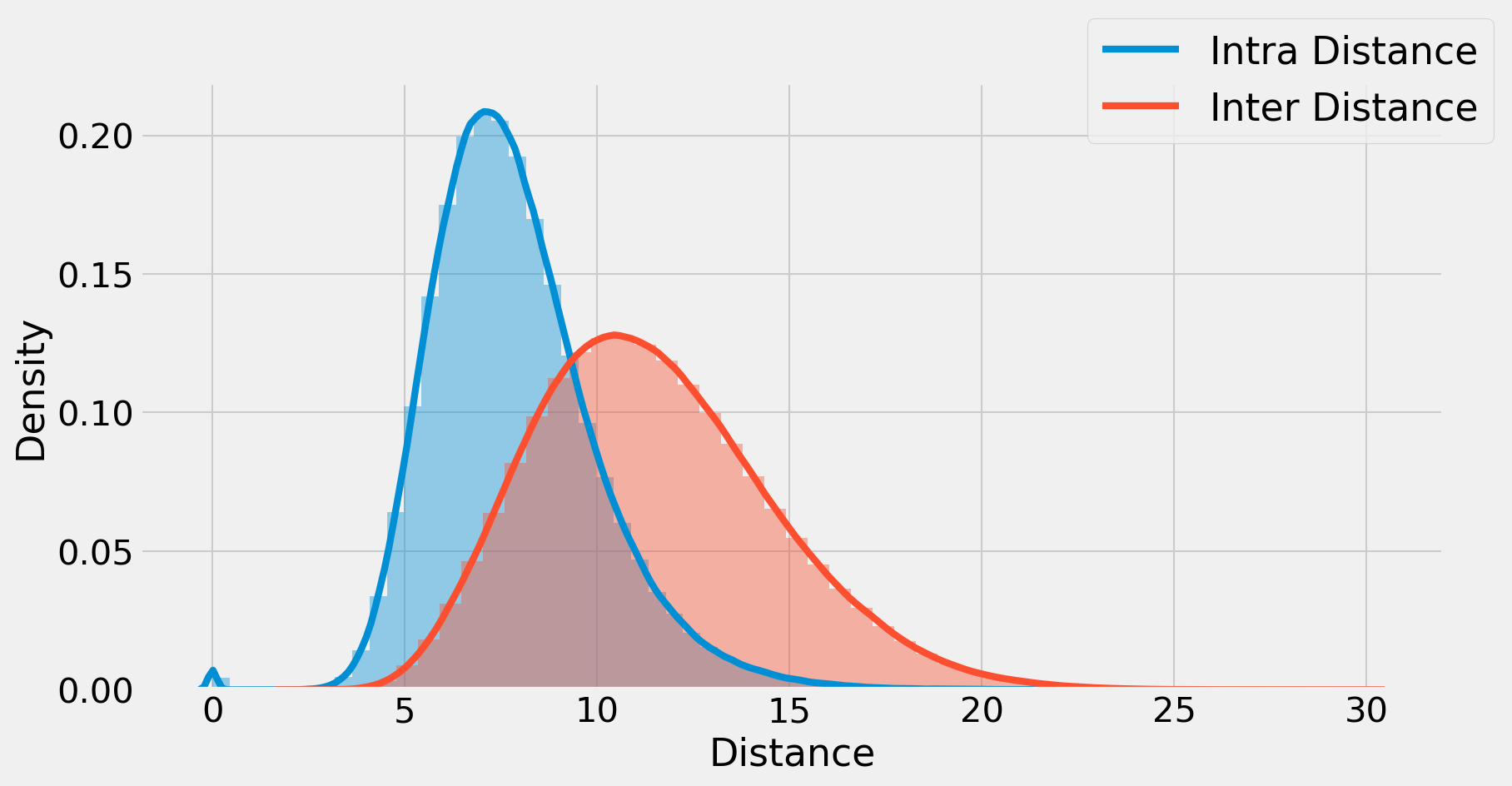} \\
  \end{tabular}
\end{center}
  \caption{A comparison of intra vs inter class distances. \textbf{Left}, we train an MLP on synthetic data (see \cref{sec:experiments-body} for experimental details) that satisfies assumption~\ref{ass:main}. On the \textbf{Middle} and \textbf{Right} we train a CNN on MNIST and  CIFAR-10. For the synthetic data the GSH property clearly holds. The intra-distance of the representation layers for networks trained on MNIST and CIFAR-10 are also significantly smaller on average than they inter distances, suggesting that even for real data, where our assumptions do not hold, a similar mechanism is at play.\vspace{-0.03em}}
  \label{fig:histograms}
\end{figure*}
\fi

Next we show Lemma~\ref{lem:weight-switch-arg-body} which lets us achieve a small variance at the representation layer by shifting weights in $B$ away from nodes corresponding to monomials which depend on $\btheta$. This change ultimately benefits the weighted square loss while also can be done in a way so that $\|A\|_F$ and $\|B\|_F$ are not impacted.
\begin{lemma}
\label{lem:weight-switch-arg-body}
Given a $B$ such that $\hat{V}_{mn}(B\vec{z}) > \omega(\eps)$, we can transform it to $B'$ with no greater Frobenius norm so that $\hat{V}_{mn}(B'\vec{z}) \le O(\eps)$.
\end{lemma}
As we saw in Claim~\ref{claim:final-repre}, the output of $\sigma(C\tilde{\vec{x}})$ can be thought of as an orthonormal transform applied onto a vector whose coordinates compute monomials of $\tilde{\vec{x}}$. 
Now we can define an association between weights of $B$ and these monomials under which we argue using Lemma~\ref{lem:var-output-jensen-body} that shifting all weights associated with monomials involving $\btheta'$ to corresponding monomials involving just $\bgamma'$ decreases the variance without increasing $\|B\|_F$, consequently improving objective~\eqref{eq:linear-loss}.
Together Lemmas~\ref{lem:var-output-jensen-body}-\ref{lem:weight-switch-arg-body} give us that at any local minima of \eqref{eq:linear-loss} the  representation $\vec{r}$ has the minimum variance possible.
\begin{lemma}
\label{lem:hash-near-train-linear}
Given any local minimum $\hat{A}, \hat{B}$ of \eqref{eq:var-reg-obj}, and given $r(\vec{x}) = \hat{B}\sigma(C\vec{x})$,  we have that $\hat{V}_{mn}(r) = O(\epsilon)$.
\end{lemma}

This will imply that at any local minimum, property \eqref{eq:hash-near} is satisfied at least on our train set. Next we need property~\eqref{eq:hash-far}. This follows as a consequence of having a small loss and a bound on $\|\hat{A}\|_F$.
\begin{lemma}
\label{lem:hash-far-train-data}
For any local minima $\hat{A}, \hat{B}$, let $r(\vec{x}) = \hat{B}\sigma(C\vec{x})$. Then,
$$\sum_{l=1}^m\sum_{j=1, j \ne l}^m\E_n\left[ \|r(\vec{x}_{l})-r(\vec{x}_{j})\|_2^2\right] \ge \Omega(m^2).\vspace{-0.1em}$$
\end{lemma}


\subsection{GSH Property on Non-linear Train Manifolds}
\label{sec:var-reg-body}
The argument in Section~\ref{sec:linear-body} does not go through for non-linear manifolds. This is because we no longer have a direct association from monomials of $\vec{x}$ to associated monomials of same degree in $\bgamma, \btheta$ as we had before. Consequently, our argument for a small representation variance at local minima (i.e., Lemma~\ref{lem:weight-switch-arg-body}) breaks down. Instead, we show the result for non-linear manifolds using a different regularizer. In addition to the $\ell_2$-regularization on the weights, we add another term which penalizes a large variance between representation vectors of points belonging to the same manifold. Note that this regularization is reminiscent of contrastive learning \cite{hadsell2006dimensionality,dosovitskiy2014discriminative,chen2020simple}, a popular technique for unsupervised representation learning.

\textbf{Variance Regularization.}
We now define the additional regularization term. Intuitively we want an empirical quantity which penalizes a high variance of the representation layer. We choose the empirical average of the variance $V_{M}$ across our train manifolds which is defined as,
\begin{align}
    V_{\text{reg}}(B\sigma(C\cdot)) = \frac{n}{n-1}\hat{V}_{mn}(B\sigma(C\cdot))\label{eq:var-reg-term}
\end{align}
The re-scaling by $n/(n-1)$ makes each term an unbiased estimator for $V_{M}(B\sigma(C\cdot))$.
We call \eqref{eq:var-reg-term} the variance regularization term. 
The final objective we minimize is,
\begin{align}
\label{eq:var-reg-obj}
    \hspace{-0.75em}\cL_{A,B}(Y, \hat{Y}) + \lambda_1\|A\|_F^2 + \lambda_2 \left(\| B\|_F^2 + V_{\text{reg}}(B\sigma(C))\right)
\end{align}
Remarkably, even though \eqref{eq:var-reg-obj} is different from what we had before we can still show that every local minimum is a global minimum (see Section~\ref{sec:opt-app} in the Appendix). 
Additionally, from the fact that the ground truth representation satisfies the GSH property, we get that under the ground truth the value of the variance regularization term is small.
Since the global minimum achieves a smaller objective than the ground truth, by choosing $\lambda_1, \lambda_2$ appropriately we get that at any local minima $V_{\mathrm{reg}}$ is small as well. 
\begin{lemma}
\label{lem:hash-near-train-non-linear}
Given any local minimum $\hat{A}, \hat{B}$ of \eqref{eq:var-reg-obj}, $V_{\mathrm{reg}}(\hat{B}\sigma(C\cdot)) \le O(\epsilon)$. 
\end{lemma}

\section{Generalization to Unseen Data}
\label{sec:generalization-body}
In this section, we present population variants for bounds on empirical quantities that we saw in Section~\ref{sec:optimization-theory}. Since the architectures for linear and non-linear manifolds are the same, the results in this section will apply to both. 
First we show that our models work well on the population loss, i.e., the test loss on new examples from $M_1,\ldots, M_m$ is small. This is a simple by-product of the weighted square loss being small. 
Next we turn our attention to property~\eqref{eq:hash-near}. 
We first show that property~\eqref{eq:hash-near} holds for all $m$ manifolds in our train set. To do this, we need to show a bound on a quantity of the form, 
$\E_m\left[ \E_{\vec{x}_l \sim \cD(M_l)}[f(\vec{x}_l)]\right] - \E_m\left[\E_n\left[f(\x_l) | {\bgamma_l} \right] \right],$
The next step is showing that for a new randomly drawn manifold, property~\eqref{eq:hash-near} holds. This involves showing a bound on a quantity of the form
$\E_{M \sim \cM}[g(M)] - \E_m\left[ g(M_l)\right].$
Both steps are shown using similar Rademacher complexity arguments. We state the final result:
\begin{lemma}[Generalization to new Manifolds]
\label{lem:common-transfer-m}
For a newly drawn manifold $M_{m+1} \sim \cM$, we have w.p. $\ge 9/10$,
\begin{align}
    V_{M_{m+1}}(\hat{B}\sigma(C(\cdot))) \le 4\epsilon, \notag
\end{align}
when $m, n \ge O\left(\frac{\beta^4\log(1/\delta)}{\epsilon^2}\right)$.\vspace{-0.1em}
\end{lemma}

\subsection{Property~\texorpdfstring{\eqref{eq:hash-far}}{(B)} Holds for most Manifolds in \texorpdfstring{$\cM$}{M}}
Now we shift our focus to showing the population variant of Lemma~\ref{lem:hash-far-train-data}.
Here, generalizing to a random new manifold drawn from $\cM$ is more tricky. Traditional uniform convergence theory deals with simple averages of a loss function evaluated on individual examples. We have a quantity which is a function evaluated on pairs of examples (pairs of manifolds in our scenario) and whose evaluations over all $O(m^2)$ pairs are averaged. Our approach hence is more involved and is described in Section~\ref{sec:gen-apx}. Our end result is Lemma~\ref{lem:hash-far-gen-m}.

\begin{lemma}
\label{lem:hash-far-gen-m}
We have w.p. $\ge 9/10$,
\begin{align*}
\vspace{-0.07em}\E_{M_1, M_2 \sim \cM}\;\E_{\substack{\vec{x}_1 \sim \cD(M_1)\\ \vec{x}_2 \sim \cD(M_2)}}\left[\|r(\vec{x}_{1})-r(\vec{x}_{2})\|_2^2 \right] \ge \Omega(1). 
\end{align*}
\end{lemma} \vspace{-0.07em}
Using the results from Sections~\ref{sec:arch-theory}-\ref{sec:generalization-body} we get Theorem~\ref{thm:informal-main} for both linear and non-linear manifolds. Using the property that at local minima $\|\hat A\|_F,\|\hat B\|_F$ are bounded (Lemma~\ref{lem:good-ground-truth}) we get the needed conditions to prove Theorem~\ref{thm:hashing-implies-oneshot} as well.\vspace{-0.03em}

\section{Experiments}
\label{sec:experiments-body}

\ifarxv
\begin{figure*}[!ht]
\begin{center}
\begin{tabular} {ccc}
  \includegraphics[width=0.32\textwidth]{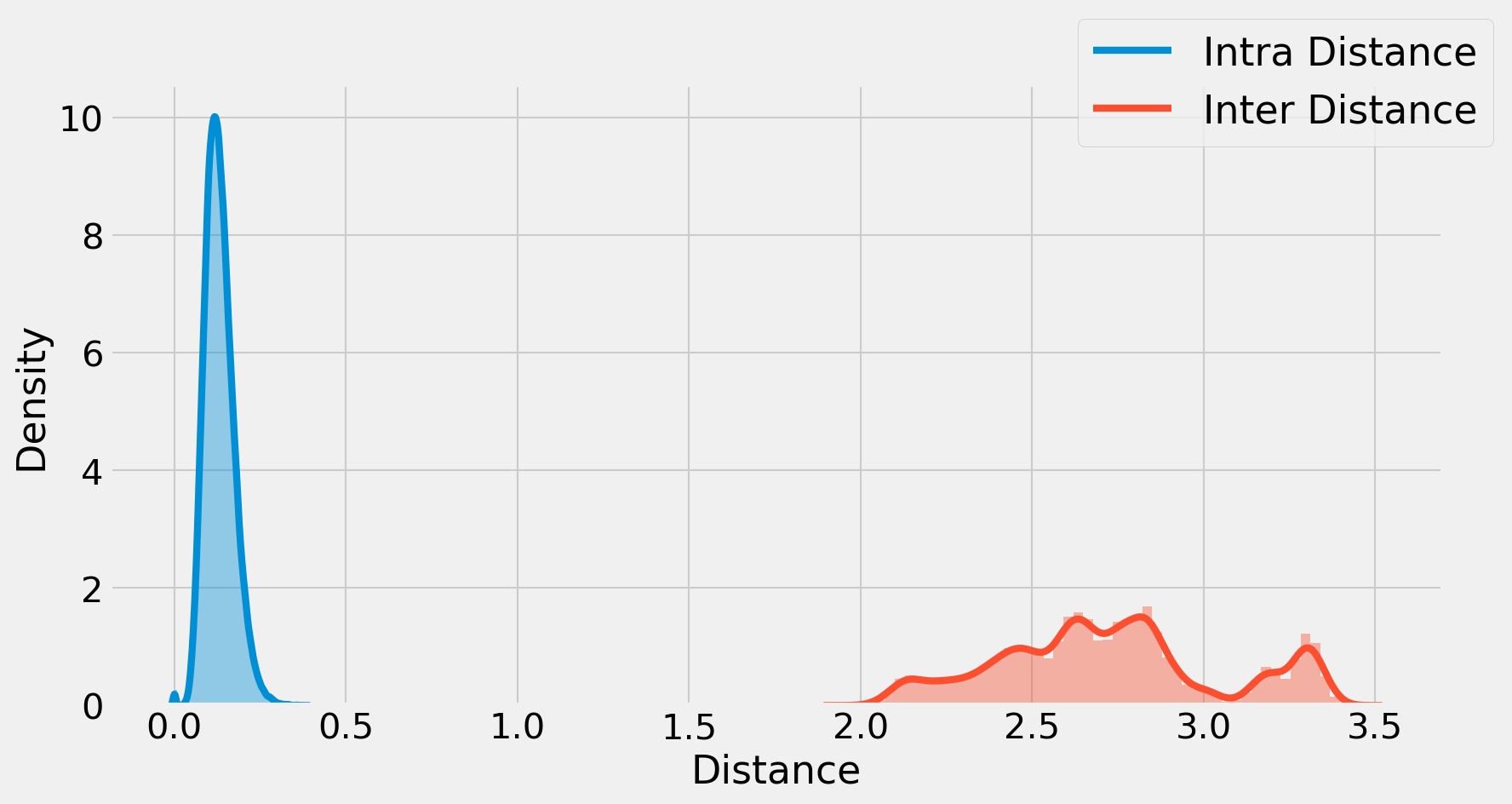} &   \includegraphics[width=0.32\textwidth]{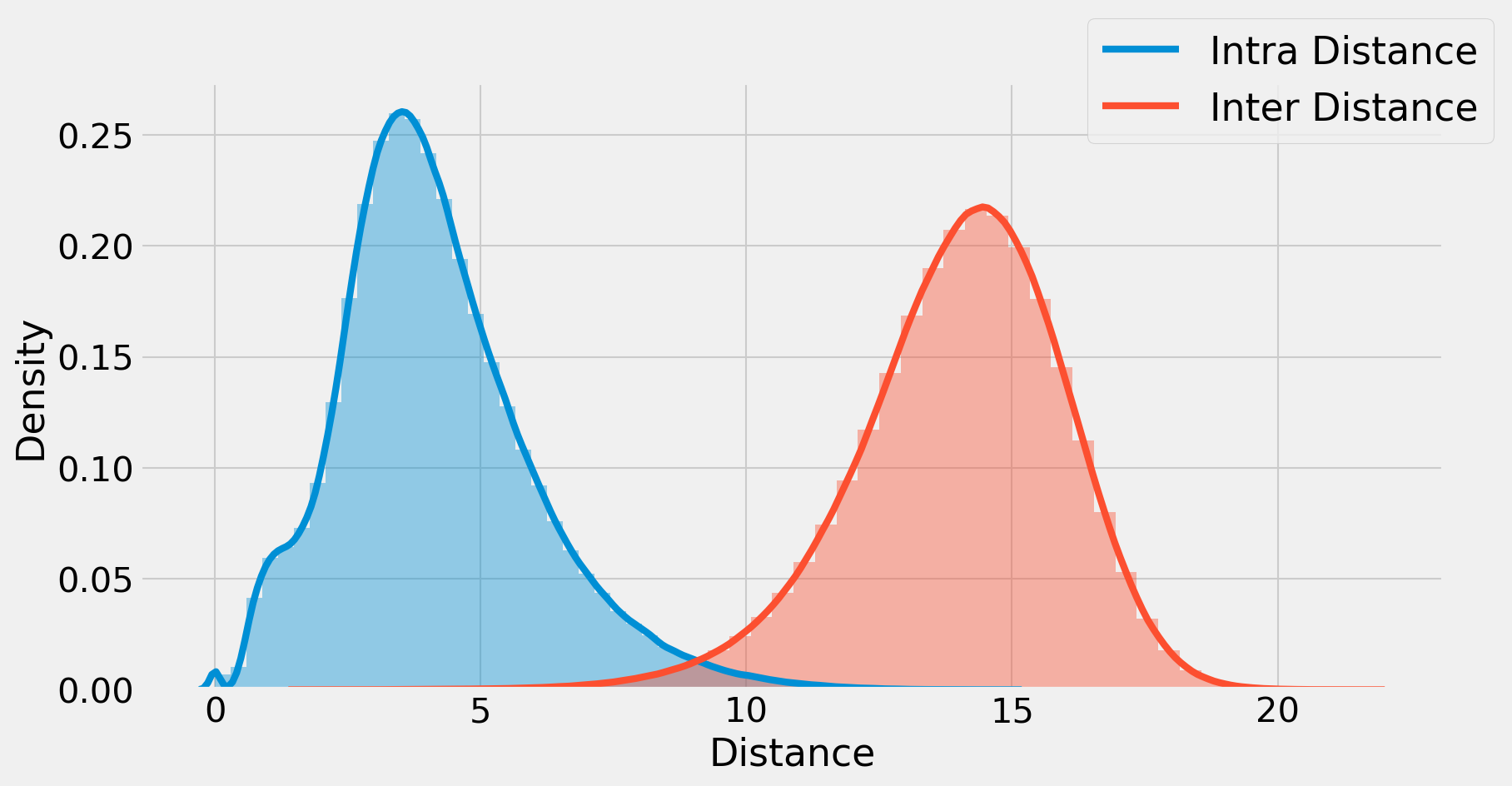} &    \includegraphics[width=0.32\textwidth]{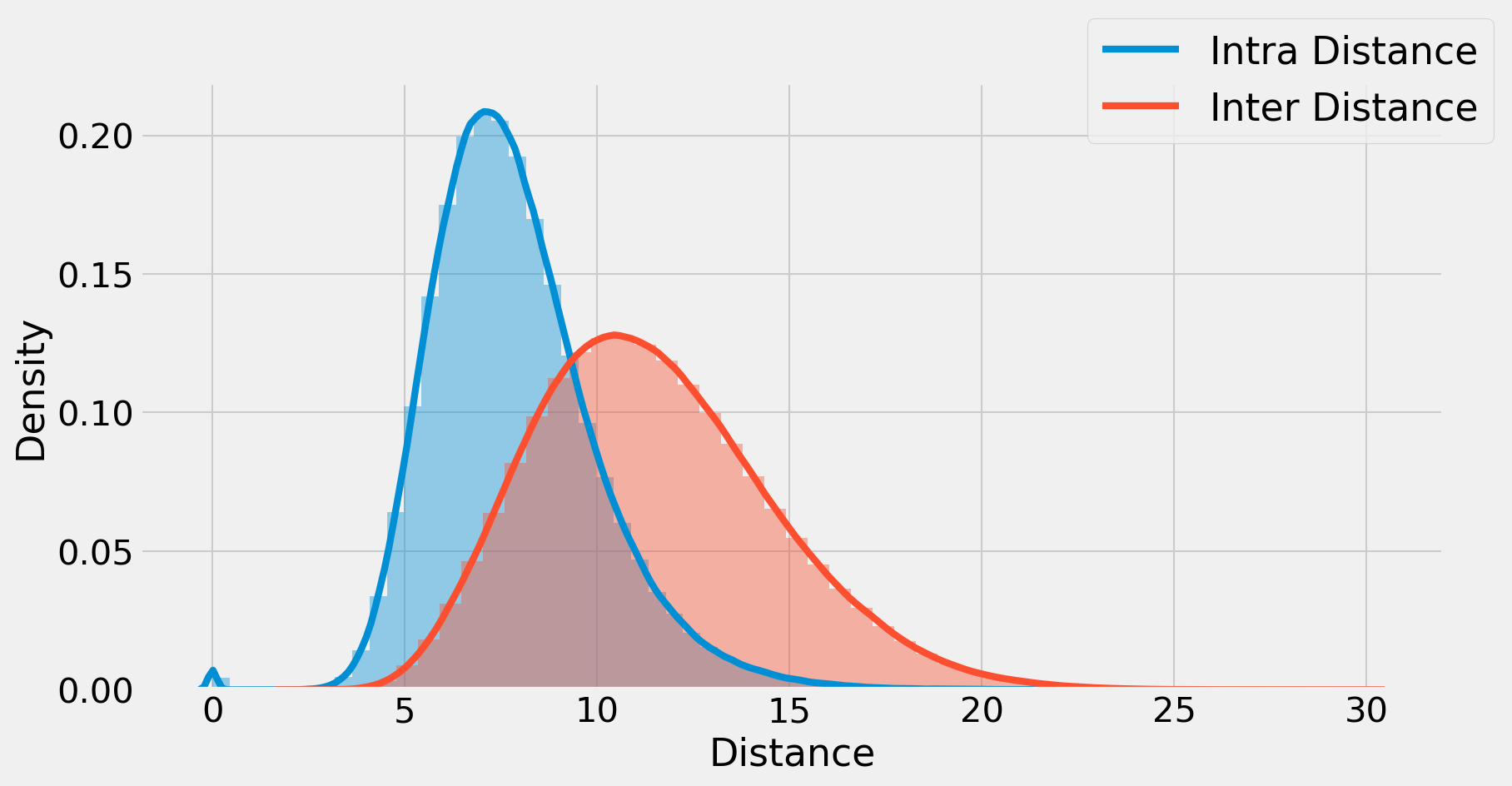} \\
  \end{tabular}
\end{center}
  \caption{A comparison of intra vs inter class distances. \textbf{Left}, we train an MLP on synthetic data (see \cref{sec:experiments-body} for experimental details) that satisfies assumption~\ref{ass:main}. On the \textbf{Middle} and \textbf{Right} we train a CNN on MNIST and  CIFAR-10. For the synthetic data the GSH property clearly holds. The intra-distance of the representation layers for networks trained on MNIST and CIFAR-10 are also significantly smaller on average than they inter distances, suggesting that even for real data, where our assumptions do not hold, a similar mechanism is at play.\vspace{-0.03em}}
  \label{fig:histograms}
\end{figure*}
\fi

 In this section, we support our theoretical results with an empirical study of the GSH property of DNNs on real and synthetic data. First, we describe our experimental setup (full details in \cref{sec:experiments-app}).
 
 \textbf{Experimental Setup}
 We separate our experiments to two groups, based on the data source. 
 \begin{itemize}
     \item \emph{Natural Images.} We train a five layer Myrtle  mCNN \cite{myrtle} on MNIST and CIFAR-10 using SGD with $\ell_2$-regularization for $50$ epochs with LR of $0.1$ then drop the LR to $0.01$ for another $100$ epochs.
     \item  \emph{Synthetic Data.} We randomly sample $\bgamma$ and $\btheta$ from a scaled Multivariate Normal so that the $\bgamma$s are well separated, then chose a function satisfying Assumption~\ref{ass:main} such as $\mathbf{f}(\cdot) = \sum_{i=1}^4 \mathbf f_i(\cdot)$ where $\mathbf f_i$ are coordinate-wise analytic functions, such as (a rotation of) $\sin,\cos, \log(0.5(1+x^2))$. So a train example becomes $\x = \textbf{f}(\bgamma, \btheta)$ and a  manifold is comprised of examples with fixed $\bgamma$ and varying $\btheta$. We train a 3-layer ReLU  MLP  with regularized $\ell_2$ loss for $200$ epochs achieving 100\% train and test accuracies.
 \end{itemize}
 \textbf{Experimental Results.} As expected, for synthetic data (see \cref{fig:histograms} left, Table~\ref{tab:main-res}), the $\rho$ is quite large even on the test data, in the range of $\rho = 10.79$-$26.8$ for the distributions we tried. This implies a strong GSH property and is consistent with our theoretical discussion. As for the real data (see middle and right panes of \ref{fig:histograms}), for MNIST the $\rho=3.36$ and for CIFAR-10 it is $\rho=1.46$ suggesting that even for distributions that do not satisfy our assumptions a-priori, the GSH holds to some extent. 
 
 \textbf{One-shot Learning and $\bgamma$ Invertibility.} We conduct two additional sets of experiments 1) We measure how well does the GSH property hold for newly sampled manifolds (i.e., few-shot learning) and; 2) whether the learnt representation is isomorphic to $\bgamma$ (i.e., are we able to invert the geometry of the manifold). For the former, we sample additional $50000$ $\bgamma^{FS}$s (FS for few-shot) and generate appropriate $\x^{FS}$s. Then, we measure the GSH property of the representation layer of the aforementioned MLP. Remarkably, even on new manifolds the GSH property strongly holds (see \cref{fig:histograms-app}) with $\rho$ in the range of $11.09$-$28.77$.
 
 For the later, we use the $\{(r(\x^{FS}_i), \bgamma^{FS}_i\}$ as a train set for a linear classifier on top of the representation produced by our MLP. In a similar fashion, we generate a test set of never before seen $\bgamma$s.  We observe (see Figure~\ref{fig:inverting}) that with enough manifolds and samples from each manifold, we are able to almost perfectly recover $\bgamma(\x)$ from $r(\x)$ with a linear function implying a (almost) linear isomorphism between the latent representation and the learnt representation, effectively recovering the geometry of the manifold.

\section{Conclusion and Discussion}
We studied the problem of supervised classification as a manifold learning problem under a specific generative process wherein the manifolds share geometry. We saw that properly trained DNNs satisfy the GSH property---by recovering the semantically meaningful latent representation $\bgamma$ while  stripping away the dependency on the classification redundant variable $\btheta$. Notably, this mechanism is not restricted to the manifolds seen during training and thus could be a  preliminary step to shed light on how Transfer Learning works in practice. Moreover, our understanding of real data distributions is limited, and further investigating generative processes such as our manifold learning is an important research direction that can illuminate real phenomena.
\newpage
\bibliography{refs}
\bibliographystyle{alpha}
\newpage

\appendix
\onecolumn
\section{Additional Preliminaries}
\label{sec:addnl-prelim}
We start the supplementary material by listing a set of additional definitions and some preliminary results. These will be for the most part statements on high-dimensional probability and linear algebra and in some cases are known from prior work or are folklore.
\label{sec:apx-prelim}
We give the definition of analytic functions next by focusing on real-valued functions.
\begin{definition}[Analytic Functions]
\label{def:analytic}
A real-valued function $f(x)$ is an analytic function on an open set $D$ if it is given locally by a convergent power series everywhere in $D$. That is, for every $x_0 \in D$,
$$f(x) = \sum_{n=0}^\infty a_n(x-x_0)^n,$$
where the coefficients $a_0,a_1,\ldots$ are real numbers and the series is convergent to $f(x)$ for $x$ in a neighborhood of $x_0$.
We also define a norm on $f$ as the two norm of the coefficient vector obtained when the above form is expanded to individual monomials.\\
Multi-variate analytic functions $f(\vec{x})$ are defined similarly with the difference being that the convergent power series is now a general multi-variate polynomial in the coordinates of $\vec{x} - \vec{x_0}$. The Taylor expansion can now be viewed to be of the form
$$f(\vec{x}) = \sum_{J} a_J\vec{x}^J$$
where $J = (j_1, \ldots, j_d)$ identifies the monomial $\vec{x}^J = x_1^{j_1}x_2^{j_2}\ldots x_d^{j_d}$.
\end{definition}

\begin{definition}[Multi-Variate Polynomials]
\label{def:multivariate-poly}
A multi-variate polynomial $p(.)$ in $\vec{x} \in \mathbb{R}^d$ of degree $k$ is defined as
$$p(\vec{x}) = \sum_{J, |J|\le k} p_J \vec{x}^J,$$
where $J=(J_1,\ldots,J_d)$ is a set of $d$ integers which identifies the monomial $\vec{x}^J = x_1^{J_1}x_2^{J_2}\ldots x_d^{J_d}$, $|J| = \sum_{i=1}^d J_i$ is the degree of the monomial and $p_J$ is the coefficient.
\end{definition}


We will show in Section~\ref{sec:arch-app} that given an infinitely wide ReLU layer we can express any analytic function by just computing a linear function of the output of the aforementioned ReLU layer. Using this knowledge, we now present our definition of norm of an analytic function we use here.
\begin{definition}
\label{def:analytic-norm}
Given a multi-variate analytic function $g : \mathbb{R}^d \to \mathbb{R}$, and an infinite width ReLU layer $\sigma(C.) : \mathbb{R}^d \to \mathbb{R}^{\infty}$, we define
$$\|g\| = \min_{\vec{a}, \vec{a}\sigma(C.) \equiv g} \|a\|_2.$$
\end{definition}
We offer more intuition along with more directly interpretable bounds on the norm of analytic function later in Section~\ref{sec:arch-app}.

\begin{definition}[Rademacher Complexity]
\label{def:rademacher}
Rademacher complexity of a function class $\cF$ is a useful quantity to understand how fast function averages for any $f \in \cF$ converge to their mean value. Formally, the empirical Rademacher complexity of $\cF$ on a sample set $S = (\vec{x_1},\ldots,\vec{x_n})$ where each sample $\vec{x_i} \sim \cD$, is defined as
$$\cR_n(\cF) = \frac{1}{n}\sum_{i=1}^n \E_{\vxi}\left[\sup_{f \in \cF} \xi_i f(\vec{x_i})\right],$$
where $\vxi = (\xi_1, \ldots, \xi_n)$ is a vector of $n$ i.i.d. Rademacher random variables (each is $+1$ w.p. $1/2$ and $-1$ w.p. $1/2$). The expected Rademacher complexity is then defined as 
$$\E[\cR_n(\cF)] = \frac{1}{n}\sum_{i=1}^n \E_{\vxi,\{\vec{x_i}\}_{i=1}^n \sim \cD^n}\left[\sup_{f \in \cF} \xi_i f(\vec{x_i})\right]$$
\end{definition}

Given the above definition of Rademacher complexity, we have the following lemma to bound the worst deviation of the population average from the corresponding sample average over all $f \in \cF$.
\begin{lemma}[Theorem 26.5 from~\cite{shalev2014understanding}]
\label{lem:rademacher-thm-shalev-textbook}
Given a function class $\cF$ of functions on inputs $\vec{x}$, if for all $f \in \cF$, and for all $\vec{x}$, $|f(\vec{x})| \le c$, we have with probability $\ge 1-\delta$,
$$\E_{\vec{x} \sim \cD}[f(\vec{x})] \le \E_n[f(\vec{x}] + 2\E[\cR_n(\cF)] + c\sqrt{\frac{2\log(2/\delta)}{n}}.$$
\end{lemma}

\begin{lemma}
\label{lem:random-matrix-spectral-norm}
Given an $D \times d$ matrix $C$ where each row is drawn from the $d$-dimensional Gaussian $\cN(\vec{0},I/D)$, we have that for a large enough constant $c_1$, if $D > \max(d, c_1)$, $\|C\|_2 \le 4$ with probability $1-\exp(-c_2D)$ for some other constant $c_2$.
\end{lemma}
\begin{proof}
Given our choice of $D$, the bound follows by a direct application of Corollary 3.11 from \cite{bandeira2016sharp} followed by simple calculations.
\end{proof}

We state the following folklore claim without proof.
\begin{claim}
\label{clm:rand-dot-product}
Let $\vec{b_1},\vec{b_2}$ be two vectors sampled uniformly at random from the $k$-dimensional ball of unit radius $S^{k-1}$. Then 
$$\Pr\left[\big|\vec{b_1}^{\top}\vec{b_2}\big|= O(1/\sqrt{k})\right] \ge 1-\frac{1}{\poly(k)}.$$
\end{claim}
Claim~\ref{clm:rand-dot-product} implies that our condition of $\tau$-separatedness is consistent with the manifold distribution $\cM$ being over a non-trivial number of manifolds.
The following claim is also folklore.
\begin{claim}[Preserving dot products with small dimensions]
\label{clm:rand-proj-jl}
Let $\vec{x_1},..,\vec{x_n}$ denote a collection of $d$-dimensional vectors of at most unit norm where $d$ is large. Let $k = O(\log n\log(1/\delta)/ \epsilon^2)$ and let $R \in \mathbb{R}^{k \times d}$ be a random projection matrix whose entries are independently sampled from $\cN(0, 1/\sqrt{k})$ that projects from $d$-dimensions down to $k$ dimensions. Then $R$ preserves all pairwise dot products $\langle \vec{x_i}, \vec{x_j}\rangle$ within additive error $\epsilon$ with probability $\ge 1-\delta$.
\end{claim}
\begin{proof}
 We have that with probability $\ge 1- \delta$, for all $i \in [n]$, $\|R\vec{x_i}\|_2 = (1 \pm \epsilon)\|\vec{x_i}\|_2)$ and $\|R(\vec{x}+\vec{y})\|_2 = (1 \pm \epsilon) \|\vec{x}+\vec{y}\|_2)$. Squaring both sides gives $\langle R\vec{x}, R\vec{y} \rangle = \langle x,y \rangle \pm O(\epsilon)$.
\end{proof}


\begin{fact}
\label{fact:unitary-preserved-frobenius}
The Frobenius norm is invariant to multiplication by orthogonal matrices. That is,
for every matrix unitary matrix $U$ and every matrix $A$,  
$$\|UA\|_F=\|A\|_F$$
\end{fact}

We next study our notion of classification loss, namely the weighted square loss we proposed in Section~\ref{sec:theory}. We relate this to the more commonly known $0/1$ classification loss now. The $0/1$ loss is defined as the fraction of mis-classified examples.
We next state a lemma which shows that if our variant of the weighted square loss is really small, then the $0/1$ loss is also small.
Given an $m$-dimensional prediction $\vec{\hat{y}}$ define $\mathrm{label}(\vec{\hat{y}}) = \argmax_{l \in [m]} \{\vec{\hat{y}}_l\}$.
\begin{lemma}
\label{lem:square-loss-zeroone-equiv}
Given $m$ train manifolds, for any $\eps>0$, let $\epsilon_1$ be such that $\frac{\eps}{4(m-1)} > \epsilon_1 > 0$. Then, we have over our train data,
\begin{align}
    \cL(Y, \hat{Y}) \le \epsilon_1 \implies \sum_{l=1}^m\sum_{i=1}^n\frac{1}{mn} \mathds{1}(\mathrm{label}(\vec{y_{il}}) = \mathrm{label}(\vec{\hat{y}_{il}})) \le \eps.
\end{align}
\end{lemma}
\begin{proof}
Let us focus on a single example $\vec{x}$ with associated true label vector $\vec{y}$ and predicted vector $\vec{\hat y}$. Suppose the label of $\vec{x}$ is $l$ for some $l \in [m]$. Then the weighted square loss being smaller than a value $\eps(\vec{x})$ implies 
\begin{align}
&\sum_{j \in [m], j\ne l} \frac{1}{2(m-1)}\hat{y}_j^2 + + \frac{1}{2}(1-\hat{y}_l)^2 \le \eps(\vec{x}) \\
& \sum_{j \in [m], j\ne l} \frac{1}{2(m-1)}\left(\hat{y}_j^2 + (1-\hat{y}_l)^2\right) \le \eps(\vec{x}) \\
&\sum_{j \in [m], j \ne l} \frac{1}{2(m-1)}\epsilon_{jl} \le \eps(\vec{x}),
\end{align}
where $\epsilon_{jl} = \hat{y}_j^2 + (1-\hat{y}_l)^2$. We have $|\hat{y}_j| \le \sqrt{\epsilon_{jl}}$ and $|1-\hat{y}_l| \le \sqrt{\epsilon_{jl}} \implies \hat{y}_l \ge 1-\sqrt{\eps_{jl}}$ which implies that if $\eps_{jl} < 1/4$, $\hat{y}_l > \hat{y}_j$. Now by Markov's inequality we have that the number of indices $j \ne l$ for which $\eps_{jl} < 1/4$ is greater than $(m-1)(1-4\eps) > (m-2)$ if $\eps(\vec{x}) < \frac{1}{4(m-1)}$. Therefore, $\mathrm{label}(\vec{\hat{y}}) = l$ if $\eps(\vec{x}) < \frac{1}{4(m-1)}$. Averaging over all examples, we have that if the average weighted square loss $\eps_1 \le \frac{\eps}{4(m-1)}$, then by another application of Markov's inequality, the function $\mathrm{label}(\vec{\hat{y}})$ will mis-classify only an $\eps$ fraction of the train examples giving the statement of the lemma.
\end{proof}

\begin{lemma}
\label{lem:frobenius-nuclear-equiv-folklore}
Let $W \in \mathbb{R}^{m \times d}$. Then, 
\begin{align}
    \|W\|_1 = \min_{\substack{A,B \\ W = AB}} \frac{1}{2}\left(\|A\|_F^2 + \|B\|_F^2\right) = \min_{\substack{A,B \\ W = AB}} \|A\|_F\|B\|_F.
\end{align}
\end{lemma}
\begin{proof}
The proof of the Lemma is folklore and we provide it for completeness sake. Using the matrix H\" older inequality and $ab\le \frac{1}{2}(a^2+b^2) $ we have, $$\|W\|_1=\|AB\|_1 \le \|A\|_F\|B\|_F \le  \frac{1}{2}(\|A\|_F^2 + \|B\|_F^2).$$
Minimizing over $A,B$ s.t. $W=AB$ gives one side of the inequality.
On the other hand, let the singular value decomposition $W=U\Sigma V^\top$. Then, for $A=U\Sigma^{1/2}$ and $B=\Sigma^{1/2}V^T$, we have that $\|A\|_F=\|B\|_F= \|\Sigma^{1/2}\|_F=\|\Sigma^{1/2}\|_F$ so, $\|A\|_F\|B\|_F=\|\Sigma^{1/2}\|_F^2=trace(\Sigma)=\|W\|_1$.
Also, $\frac{1}{2}(\|A\|_F^2 + \|B\|_F^2)=\|\Sigma^{1/2}\|_F^2$, so $\|W\|_1=\|A\|_F\|B\|_F = \frac{1}{2}(\|A\|_F^2 + \|B\|_F^2)$, so the inequality is tight.
\end{proof}

We now state a well-known claim about dual spaces.
\begin{claim}
\label{clm:trace-dual-operator}
The dual of the operator norm is the trace norm and vice versa. Given two matrices $A$ and $B$ define $\langle A, B\rangle = Tr(A^\top B) = \sum_{i,j} A_{ij}B_{ij}$. Then,
\begin{align}
&\|A\|_* = \sup_{B \; s.t. \; \|B\|_2 \le 1} \langle A, B \rangle \\
\text{and }&\|A\|_2 = \sup_{B \; s.t. \; \|B\|_* \le 1} \langle A, B \rangle.
\end{align}
\end{claim}

We next state a lemma which characterizes the structure of matrices $A, B$ at any local minima of $\min_{W=AB} \|A\|_F^2 + \|B\|_F^2$. This lemma or a variant might have been used in prior work but we couldn't find a reference. We provide a proof here for completeness.
\begin{lemma}
\label{lem:minima-sum-frobenius-charac}
Let $W \in \mathbb{R}^{m \times d}$ and let $r = \min(m,d)$.  Let $\hat{A},\hat{B}$ be the minima of the following constrained optimization:
$$\min_{W=AB} \|A\|_F^2 + \|B\|_F^2$$ then at a local minimum there is a matrix $R$ so that if $\svd(W) = U S V^T$ then $A = U S^{1/2} R, B = R^T S^{1/2} V^T$ where $R R^T = I$. (Here if $W$ is not full rank the SVD is written by truncating $U, S, V$ in  a way where $S$ is square and full rank).
\end{lemma}
\begin{proof}
First assume $W=I$ and $A,B$ are square. Then $B = A^{-1}$. So we can write $\svd(A) = U S V^T$ and $\svd(B) = U S^{-1} V^T$ where $S$ is diagonal. Now $\|A\|_F = \|U S V^T\|_F = \|S\|_F$ as multiplying by orthonormal matrix doesn't alter Frobenius norm (Fact~\ref{fact:unitary-preserved-frobenius}). But $\|S\|_F + \|S^{-1}\|_F = \sum_i (S_{ii}^2 + 1/S_{ii}^2)$, which is minimized only when $S_{ii} = \pm 1$. So $A$ becomes orthonormal. Further note that if that is not the case then it cannot be a local minima as there is a direction of change for some $S_{ii}$ that improves the objective.

Next look at the case when $W$ may not be $I$ but is full rank and $A,B$ are square Then let $\svd(W) = U S V^T$. So $AB = USV^T$. So $S^{-1/2} U^T A B V S^{-1/2} = I$. Now since $S^{-1/2} U^T A$ and $B V S^{-1/2}$ are inverses of each other we can write their SVD as $U_2 S_2 V_2^T$ and $V_2 S_2^{-1} U_2^T$. So $A = U S^{1/2} U_2 S_2 V_2^T$ and 
$B = V_2 S_2^{-1} U_2^T S^{1/2} V^T$. SO $|B|_F = |S_2^{-1} U_2^T S^{1/2}|_F$, 
and $|A|_F = |A^T|_F = |S_2 U_2^T S^{1/2}|_F$. Let $Y = U_2^T S^{1/2}$. Then $|A|_F + |B|_F = |S_2 Y|_F + |S_2^{-1} Y|_F  
= \sum_i (S_2)_{ii}^2 |Y_{i,*}|_F^2 + (1/S_2)_{ii}^2 |Y_{i,*}|_F^2
= \sum_i ((S_2)_{ii}^2 + (1/S_2)_{ii}^2) |Y_{i,*}|_F^2
$.
This is again minimized when $(S_2)_{ii} = \pm 1$ which means $S_2$ is orthonormal. So $A = U S^{1/2} R$, and 
$B = R^T S^{1/2} V$ where $R$ is orthonormal. Again note that if this is not true then some $(S_2)_{ii}$ can be perturbed and the objective can be locally improved.

The argument continues to hold if $A,B$ are not square as the only thing that changes is that $R$ now can become rectangular (with possibly more columns than rows) but still $R R^T = I$.

If $W$ is not full rank again we can apply the above argument in the subspace where $W$ has full rank (or equivalently writing SVDs in a way where the diagonal matrix $S$ is square but $U,V$ may be rectangular but still orthogonal).
\end{proof}

\begin{corollary}
\label{cor:minima-sum-frobenius-charac}
For any convex function $O(.)$, $\min_{A,B} O(AB) + (|A|_F^2 + |B|_F^2)$ can only be at a local minimum when $A = U S^{1/2} R$ and $B = R^T S^{1/2} V^T$ for some $U U^T = I$, $R R^T = I$, $V^T V = I$
\end{corollary}
\begin{proof}
This follows from the fact that otherwise the previous lemma can be used to alter $A,B$ while keeping the product $AB$ fixed and improve the regularizer part of the objective.
\end{proof}
\section{Our Theoretical Results}
In this section, we state our main theorems formally. We re-state our generative process to remind the reader.
\paragraph{Generative Process}
We consider manifolds which are subsets of points in $\mathbb{R}^d$. Every manifold $M_{\bgamma}$ has an associated latent vector $\bgamma \in \mathbb{R}^s, s \le d$ which acts as an identifier of $M_{\bgamma}$. The manifold is then defined to be the set of points $\vec{x} = \vec{f}(\bgamma, \btheta) = (f_1(\bgamma, \btheta),  \ldots, f_d(\bgamma, \btheta))$ for $\btheta\in\Theta \subseteq \mathbb{R}^{k}, k<d$. Here, the manifold generating function $\vec{f} = \{f_i(\cdot,\cdot)\}_{i=1}^d$ where the $f_i$ are all analytic functions. $\btheta$ acts as the ``shift" within the manifold. Without significant loss of generality, we assume our inputs $\vec{x}$ and $\bgamma$s are normalized and lie on the $S^{d-1}$ and $S^{s-1}$ respectively.
Given the above generative process, we assume that there is a well-behaved analytic function to invert it.
\begin{assumption}[Invertibility: Restatement of Assumption~\ref{ass:main}]
\label{ass:invertibility-app}
There is an analytic function $\vec{g}(\cdot) : \mathbb{R}^d \to \mathbb{R}^s$ with norm (Definition~\ref{def:analytic-norm}) bounded by a constant s.t. for every point $\vec{x} = \vec{f}(\bgamma, \btheta)$ on $M_{\bgamma}$, $\vec{g}(\vec{x}) = \bgamma$. 
\end{assumption}

Next we describe how we get our train data. 
As described above, a set of analytic functions $\{f_i\}$ and a vector $\bgamma$ together define a manifold. A distribution $\cM$ over a class of manifolds $\supp(\cM)$ (given by the $\{f_i\}$) is then generated by having a set $\Gamma$ from which we sample $\bgamma$ associated with each manifold. We assume that for any two manifolds $M_1, M_2 \in \supp(\cM)$, $\vec{\bgamma}_1^\top \vec{\bgamma}_2 \le \tau$ where $\tau < 1$ is a constant. To describe a distribution of points over a given manifold $M$ we use the notion of a point density function $\cD(\cdot)$ which maps a manifold $M$ to a distribution $\cD(M)$ over the surface of $M$.
Training data is then generated by first drawing $m$ manifolds $M_1,\ldots, M_m \sim \cM$ at random. Then for each $l \in [m]$, $n$ samples $\{(\vec{x}_i^l, \vec{y}_i^l) \}_{i=1}^n$ are drawn from $M_l$ according to the distribution $\cD(M_l)$. Note that we view the label $\vec{y}_i^l$ as a one-hot vector of length $m$ indicating the manifold index. 
We consider a 3-layer neural network $\hat{\vec{y}} = AB\sigma(C\vec{x})$, where the input $\vec{x} \in \mathbb{R}^d$ passes through a wide randomly initialized fully-connected {\em non-trainable} layer $C \in \mathbb{R}^{D \times d}$ followed by a ReLU activation $\sigma(.)$. Then, there are two trainable fully connected layers $A \in \mathbb{R}^{m \times T},B \in \mathbb{R}^{T \times D}$ with {\em no} non-linearity between them.
Each row of $C$ is drawn i.i.d. from $\cN(\vec{0}, \frac 1D I)$. It follows from random matrix theory that $\|C\|_2 \le 4$ w.p. $\ge 1-\exp(-O(D))$.

Theorem~\ref{thm:informal-main} is our main theoretical result which is an informal variant of the following two theorems.
\begin{theorem}[Main Theorem: GSH for Linear Manifolds]
\label{thm:linear-main}
Let $\cM$ be a distribution over $\tau$-separated linear manifolds in $\mathbb{R}^d$ such that the latent vectors all lie on $S^{s-1}$. Given inputs $\vec{x}$ such that $\|\vec{x}\|_2 = 1$, let $\vec{y} = AB\sigma(C\vec{x})$ be the output of a 3-layer neural network, where $A \in \mathbb{R}^{m \times T}, B \in \mathbb{R}^{T \times D}$ are trainable, and $C \in \mathbb{R}^{D\times d}$ is randomly initialized as described above. Suppose we are given $n$ data points from each of $m$ manifolds sampled i.i.d. from $\cM$. For any $\eps>0$, running gradient descent on
$$\cL_{A,B}(Y, \hat{Y}) + \lambda_1 (\|A\|_F^2) + \lambda_2(\|B\|_F^2)$$
yields $\hat{A}, \hat{B}$ such that with probability $\ge 1-\delta$
\begin{enumerate}
    \item $\cL_{A,B}(Y, \hat{Y}) \le O(\eps)$,
    \item The representation computed by $\tilde{B}\sigma(C.)$ satisfies $(\eps,1/\eps)$-GSH.
\end{enumerate}
for $n = \Theta\left( \frac{s^{O(\log(1/\eps))}\log(m/\delta)}{\eps^2}\right)$, $m = \Theta\left(\frac{s^{O(\log(1/\eps))}\log(2/\delta)}{\eps^2}\right)$ and $T = \Theta\left(\log(mn)\log(1/\delta)\eps^{-1}\right), D = \Theta\left(\frac{\sqrt{mn}\log(mn/\delta)}{\eps}\right)$, and $\lambda_1 = \eps/m, \lambda_2 = \eps/s^{O(\log(1/\eps))}$.
\end{theorem}

\begin{theorem}[Main Theorem: GSH for Non-Linear Manifolds]
\label{thm:var-reg-main}
Let $\cM$ be a distribution over $\tau$-separated manifolds in $\mathbb{R}^d$ such that the latent vectors all lie on $S^{s-1}$. Given inputs $\vec{x}$ such that $\|\vec{x}\|_2 = 1$, let $\vec{y} = AB\sigma(C\vec{x})$ be the output of a 3-layer neural network, where $A \in \mathbb{R}^{m \times T}, B \in \mathbb{R}^{T \times D}$ are trainable, and $C \in \mathbb{R}^{D\times d}$ is randomly initialized as described above. Suppose we are given $n$ data points from each of $m$ manifolds sampled i.i.d. from $\cM$. For any $\eps>0$, running gradient descent on
$$\cL_{A,B}(Y, \hat{Y}) + \lambda_1 (\|A\|_F^2) + \lambda_2(\|B\|_F^2)$$
yields $\hat{A}, \hat{B}$ such that with probability $\ge 1-\delta$
\begin{enumerate}
    \item $\cL_{A,B}(Y, \hat{Y}) \le O(\eps)$,
    \item The representation computed by $\tilde{B}\sigma(C.)$ satisfies $(\eps,1/\eps)$-GSH.
\end{enumerate}
for $n = \Theta\left( \frac{s^{O(\log(1/\eps))}\log(m/\delta)}{\eps^2}\right)$, $m = \Theta\left(\frac{s^{O(\log(1/\eps))}\log(2/\delta)}{\eps^2}\right)$ and $T = \Theta\left(\log(mn)\log(1/\delta)\eps^{-1}\right), D = \Theta\left(\frac{\sqrt{mn}\log(mn/\delta)}{\eps}\right)$, and $\lambda_1 = \eps/m, \lambda_2 = \eps/s^{O(\log(1/\eps))}$.
\end{theorem}

A benefit of having the hashing property is we get easy transfer learning. This was Theorem~\ref{thm:hashing-implies-oneshot} in the main body. We now re-state this theorem provide its proof below.
\begin{theorem}[GSH Implies One-Shot Learning]
\label{thm:hashing-implies-oneshot-app}
Given a distribution $\cM$ over $\tau$-separated manifolds, if a representation function $r(\cdot)$ satisfies the $(\eps,\rho)$-GSH property over $\cM$ with probability $\ge 1-\delta$, a large enough $\rho$, then we have one-shot learning. That is there is a simple hash-table lookup algorithm $\cA$ such that it learns to classify inputs from manifold $M_{new} \sim \cM$ with just one example with probability $\ge 1- \delta$.
\end{theorem}
\begin{proof}
Let $\cA$ be the following algorithm. Given a single example $\vec{x_{new}} \sim M_{new}$, we compute $r(\vec{x_{new}})$. Then given any other input $\vec{x}$, it does the following:
\begin{align}
    &\mathrm{if }\; \; \; \|r(\vec{x}) - r(\vec{x_{new}})\|_2^2 < 2\epsilon, \; \;\;\mathrm{ then }\; \;\; \vec{x} \in M_{new}, \notag \\
    &\mathrm{else }\; \;\; \vec{x} \notin M_{new}. \notag
\end{align}
Since $r(\cdot)$ satisfies the $(\epsilon, \rho)$-GSH w.p. $\ge 1-\delta$, for $\rho ge 2$, we have that $\cA$ misclassifies an input $\vec{x}$ only with probability $\le \delta$.
\end{proof}

Next, it remains to prove Theorems~\ref{thm:linear-main} and \ref{thm:var-reg-main}. We split the proofs over multiple sections. Section~\ref{sec:arch-app} studies the properties of our architecture which is expressive enough to enable us to learn the geometry of the manifold surfaces.
Section~\ref{sec:opt-app} analyses our loss objective to show an empirical variant of the GSH for both linear and non-linear manifolds. Finally Section~\ref{sec:gen-apx} is about generalizing from the empirical variant to the population variant. All three put together give us Theorems~\ref{thm:linear-main} and \ref{thm:var-reg-main}.

\section{Kernel Function view of Random layer with activation \texorpdfstring{$\sigma$}{s}}
\label{sec:arch-app}
We start by looking at some properties of a wide random ReLU layer. At a high level, our goal is to show that a random ReLU layer computes a transform of the input which is highly expressive. Formally, we will show that a linear function of the feature representation computed by the random ReLU layer can approximately express `well-behaved' analytic functions. Our formalization of what we mean by `well-behaved' is a bit technical and relies on the understanding we develop of the transformation an input goes through via a random ReLU layer. We develop this understanding via a sequence of lemmas.

The first is the following simple lemma which focuses on a single node of a random ReLU layer and defines a kernel on the implicit feature space computed by a ReLU using the dual activation function of ReLU.
\begin{lemma}(Random ReLU Kernel)
\label{lem:random-relu-kernel}
For any $\vec{x}, \vec{y}\in \mathbb{R}^d$ and $\vec{r}$ drawn from the $d$-dimensional normal distribution,
$$\E_{\vec{r}}\left[\sigma(\vec{r}^\top \vec{x})\sigma(\vec{r}^\top \vec{y}) \right] = K(\vec{x},\vec{y}) = \|\vec{x}\|_2\|\vec{y}\|_2\hat \sigma\left(\frac{\vec{x}^\top \vec{y}}{\|\vec{x}\|_2\|\vec{y}\|_2}\right).$$
where $\hat \sigma(\eta)  = \frac{\sqrt{1-\eta^2} + (\pi - \cos^{-1}(\eta))\eta}{2\pi}$.
\end{lemma}
\begin{proof}
The result follows by noting the form of the dual activation of ReLU from Table 1 of \cite{daniely2016toward} together with the observation that for any unit vectors $\vec{u}, \vec{v}$, the joint distribution of $(\vec{r}^\top \vec{u}, \vec{r}^\top \vec{v})$ is a multivariate Gaussian with mean 0 and covariance,
$$\text{cov}(\vec{r}^\top \vec{u}, \vec{r}^\top \vec{v}) = \left(\begin{matrix}\vec{u}^\top \\ \vec{v}^\top\end{matrix}\right)\text{cov}(\vec{r})(\vec{u}, \vec{v}) =
\left(\begin{matrix}
  \|\vec{u}\|_2^2 & \vec{u}^\top \vec{v}\\
  \vec{u}^\top \vec{v} & \|\vec{v}\|_2^2
\end{matrix}\right) = \left(\begin{matrix} 1 & \eta \\ \eta & 1
\end{matrix}
\right),$$
when $\vec{u} = \vec{x}/\|\vec{x}\|_2$ and $\vec{v} = \vec{y}/\|\vec{y}\|_2$.
\end{proof}

We let $N = mn$ denote the total number of samples we have from all our train manifolds. Recall that $X$ denotes a rank-3 tensor of size $m \times d \times n$ obtained by stacking the $X_l$ matrices for $l \in [m]$. In the rest of this section, we override notation and flatten $X$ to be a $d \times N$ matrix. Given the kernel function $K(.)$ from Lemma~\ref{lem:random-relu-kernel}, we let $K(X,X)$ be the $N \times N$ kernel matrix whose $(i,j)^{th}$ entry is $\sigma(CX_i)^\top \sigma(CX_j)$ (where $X_i$ is the $i^{th}$ column of $X$).
Next, we have the following result which shows that with high probability, for any two inputs among our $N$ train inputs, the inner product of the feature representation given at the end of a random ReLU layer is close to the kernel evaluation on this pair of inputs.
\begin{lemma}
\label{lem:hidden-nodes-rep}
Let $N = mn$ and let $D= \Theta\left(\frac{\sqrt{N}\log(2N^2/\delta)}{\eps}\right)$. Then letting $Z_{D}= \sigma (C \cdot X)$ where $Z_{D} \in \mathbb{R}^{D \times N}$, we have,
\begin{align}
    \E\left[Z_{D}^\top Z_{D}\right] = K(X, X)
\end{align}
where $X$ is the train set (and each column is of norm $1$), and $K(X, X)$ is the Random ReLU kernel given in Lemma~\ref{lem:random-relu-kernel}. 
Moreover, for any $\eps, \delta>0$, w.p. $\ge 1-\delta$,
\begin{align*}
    \|Z_{D}^\top Z_{D} - K(X, X)\|_F \le \eps
\end{align*}
\end{lemma}
\begin{proof}
For the first part of the lemma, let any two indices $j,k$ and let $\vec{x_i}$ and $\vec{x_j}$ be the appropriate columns of $X$. Then, $\E\left[Z_{D}^\top Z_{D}\right]$ in coordinate $j, k$ is, 
\begin{align*}
    \E\left[Z_{D}^\top Z_{D}\right]_{j,k} = \E\left[\sum_{t=1}^{D} \sigma (\vec{r_t}^\top \vec{x_j})\sigma (\vec{r_t}^\top \vec{x_k})\right] = K(\vec{x_j},\vec{x_k})
\end{align*}
Where $r_t$ is the $t^{th}$ (random) row of $C$ and we use linearity of expectation and Lemma~\ref{lem:random-relu-kernel}.
For the second part, write $\sigma(\vec{r}_t^\top \vec{x}_j)\sigma(\vec{r}_t^\top \vec{x}_k)\sim \frac{1}{D}\sigma(Y_t)\sigma(Z_t)$, where $Y_t,Z_t$ are jointly distributed as $\mathcal N\left(0, \begin{pmatrix}1 & \rho\\ \rho & 1 \end{pmatrix}\right)$. Now we note that $\sigma(Y_t)\sigma(Z_t)$ is a sub-exponential random variable (e.g., see \cite{vershynin2019high}), either by noticing that it is a multiplication of sub-Gaussians or directly by taking any $a>0$ and writing,
\begin{align*}
    \Pr[\sigma(Y_t)
    \sigma(Z_t) > a ] \le \Pr[|Y_t|>\sqrt{a})] + \Pr[|Z_t| > \sqrt a] \le 4 \exp(-a/2)
\end{align*}
And so, for all $a>-\mu$,
\begin{align*}
    \Pr[\sigma(Y_t)\sigma(Z_t) - \mu > a ] \le  4 \exp\left(-\frac{(a+\mu)}{2}\right)
\end{align*}
Thus, using a property of the sub-exponential family, there exists some universal constant $c>0$, 
\begin{align*}
    \Pr\left[\left|\frac{1}{D}\sum_{t=1}^{D}\E[\sigma(Y_t)\sigma(Z_t)] - \mu\right| > \eps \right] \le 2 \exp\left(\frac{-D\eps}{c}\right)
\end{align*}
So by taking $D = \Theta\left(\frac{\sqrt{N}\log(2n^2/\delta)}{\eps}\right)$ and using the union bound over all $N^2$ coordinates, we have w.p. $1-\delta$, $\|Z_{D}^\top Z_{D}-K(X,X)\|_\infty \le \frac{\eps}{\sqrt N}$ and thus, 
$$\|Z_{D}^\top Z_{D}-K(X,X)\|_F\le\eps.$$
\end{proof}

Next, we show a linear algebraic result which argues that if two sets of vectors have the same set of inner products amongst them, then they must be semi-orthogonal transforms of each other. Recall that a rectangular matrix with orthogonal columns (or rows) is called semi-orthogonal.
\begin{lemma}
\label{lem:ortho-transform}
Let $X \in \mathbb{R}^{D \times n}$ and $Y \in \mathbb{R}^{D \times n}$ and let $X^{\top}X = Y^{\top}Y$, assuming $D\ge n$. Then there exists a semi-orthogonal matrix $U$ with orthogonal columns such that $X = UY$.
\end{lemma}
\begin{proof}
If $Y$ is invertible then let $U = X Y^{-1}$. Then clearly $UY = X$ and 
$$U^{\top} U = (Y^{-1})^\top X^\top X Y^{-1} = (Y^{-1})^\top Y^\top Y Y^{-1} = I.$$
Now if $Y$ is not invertible, then first note that $X^\top X$ and $X$ have the same null space (as they have the same right singular vectors and the singular values for the former are squares of those of the latter), and since $X^\top X = Y^\top Y$, $X$ and $Y$ have the same null space. Write $U' = XY^\dagger$ where $Y^\dagger$ is the pseudoinverse of $Y$ and let $V$ be the identity transformation on $\ker(Y)$ (and $0$ everywhere else). Then, we claim that $U=U' + V$ is an orthogonal matrix such that $X = UY$. The main point is that $U'^\top U'$ is an identity operator outside $\ker{Y}$ and inside $\ker Y$, $V$ is an identity operator. To see that $X=UY$, note that for every $\vec{x}\in \R^n$, write $\vec{x}=\hat{\vec{x}} + \vec{x}^\perp$, decomposing $\vec{x}$ to $\text{span } Y \oplus \ker{Y}$. Then, we have,
$$UY\vec{x}=(XY^\dagger Y + VY)(\hat{\vec{x}} + \vec{x}^\perp)= XY^\dagger Y\hat{\vec{x}} = X\hat{\vec{x}}= X\vec{x},$$ 
where we used $X\vec{x}^\perp = Y \vec{x}^\perp = 0$, $V\vec{y}=0$ for $\vec{y}=Y\hat{\vec{x}}\in \text{span } Y $ and $Y^\dagger Y=I$ on $\text{span }Y $. So, $UY$ and $X$ agree as transformations on all of $\R^n$ and therefore are the same. Now, 
$U^\top U= (U'+V)^\top (U'+V)=U'^\top U' +V^\top V$ as $U'\perp V$. But $U'^\top U'$ is the identity on $\text{span } Y$ (and $0$ elsewhere) and $V^\top V$ is the identity on $\ker Y$ (and $0$ elsewhere) so $U^\top U = I$.
\end{proof}

When $X^TX$ is only approximately equal to $Y^TY$, a weaker variant of Lemma~\ref{lem:ortho-transform} still holds.
\begin{lemma}
\label{lem:sqrt_close}
If there a sequence of matrices $X_i, Y_i \in \mathbb{R}^{D_i \times n}$ so that $X_i^\top X_i, Y_i^\top Y_i  \rightarrow A$ as $D_i \to \infty$ then $X_i = U_i Y_i + \Delta_i$  where $U_i$ are orthonormal and $\|\Delta_i\|_F \rightarrow 0$. Precisely if $\|X_i^\top X_i - A\|_F \le \epsilon$ and $\|Y_i^\top Y_i - A\|_F \le \epsilon$, then $\|\Delta_i\|_F \le 2 \sqrt \eps$.
Although we assumed $X_i, Y_i$ have the same number of rows, if they were different we could pad the smaller matrix with zero vectors to get them to be the same shape. 
\end{lemma}
\begin{proof}
Let $(.)^{1/2}$ denote the matrix square-root operator which is defined as follows: $A^{1/2} = U\Sigma^{1/2}V^\top$ where $svd(A) = U\Sigma V^\top$. Note that this operator is continuous. Let $B_i = (X_i^\top X_i)^{1/2}$ and $C_i = (Y_i^\top Y_i)^{1/2}$. Let us pad $B_i, C_i$ with zero rows so that they are both of dimension $D_i \times n$ then by continuity of the square root of a matrix, if  $\Delta'_i = B_i - C_i$, $\|\Delta'_i\|_F \rightarrow 0$.
Note that $B_i^\top B_i = X_i^\top X_i$. Then, from Lemma~\ref{lem:ortho-transform} we have that $X_i = P_i B_i$ where the $P_i$ are orthonormal. Similarly, we have $Y_i = Q_i C_i$ where the $Q_i$ are orthonormal. So $X_i = P_i B_i = P_i Q_i^\top Q_i B_i = P_i Q_i^\top Q_i (C_i + \Delta'_i) = P_iQ_i^\top Q_iC_i + \Delta_i = P_iQ_i^\top Y_i + \Delta_i$. Finally, $\|\Delta_i\|_F = \|P_i Q_i^\top Q_i\Delta'_i\|_F = \|\Delta'_i\|_F$ from Fact~\ref{fact:unitary-preserved-frobenius}. Therefore, $\|\Delta_i\|_F \to 0$.
Hence we have $X_i = U_i Y_i + \Delta_i$ where $U_i = P_i$ and $\|\Delta_i\|_F \to 0$.

To make this precise, we note that for two $n \times n$ square matrices $U, V$, 
$\|U^{1/2}-V^{1/2}\|_2 \le n^{-1/2} \|U-V\|_2$ \cite{carlsson2018perturbation} and so $\|U^{1/2}-V^{1/2}\|_F \le \|U-V\|_F$ . So $\|B_i - A^{1/2}\|_F \le \sqrt \eps$ and $\|C_i - A^{1/2}\|_F \le \sqrt \eps $. and so $\|\Delta'_i\|_F \le 2 \sqrt \eps$ and hence $\|\Delta_i\|_F \le 2 \sqrt \eps$.

\end{proof}


Now, let $\hat \sigma(\eta)  = \frac{1}{2\pi} + \frac{1}{4}\eta + \frac{1}{4\pi} \eta^2 + \frac{1}{48\pi} \eta^3 + \ldots = q_0 + q_1 \eta + q_2 \eta^2 + q_3 \eta^3 \ldots$ denote the Taylor series expansion of $\hat \sigma$, the dual activation of ReLU defined in Lemma~\ref{lem:random-relu-kernel}.
Note that $q_k$ decays as $O\left(\frac{1}{k^{3/2}}\right)$. So for $\eta \le 1$ we can approximate this series within $\eps$ error as long as we use at least the first $O(1/\eps^{2/3})$ terms.

We will now argue, using Lemmas~\ref{lem:ortho-transform} and \ref{lem:sqrt_close}, that the output of the random ReLU layer can be viewed with good probability as approximately an orthogonal linear transformation applied on a power series $\phi(\vec{x})$, where $\phi(\vec{x}) = \left(\sqrt{q_0}, \sqrt{q_1}\vec{x}, \sqrt{q_2}\vec{x}^{\otimes 2}, \sqrt{q_3}\vec{x}^{\otimes 4}, \ldots \right)$, an infinite dimensional vector where $\vec{x}^{\otimes i}$ is a flattened tensor power $i$ of the vector $\vec{x}$.
Let $\phi_k(\vec{x}) = \left(\sqrt{q_0}, \sqrt{q_1}X_i, \sqrt{q_2}X_i^{\otimes 2}, \sqrt{q_3}X_i^{\otimes 3}, \ldots \sqrt{q_k}X_k^{\otimes k} \right)$ denote the truncation of $\phi(\vec{x})$ up to the $k^{th}$ tensor powers.
The following Lemma allows us to think of a random ReLU layer of high enough width as kernel layer that outputs a sequence of monomials in its inputs. 


\begin{corollary}
\label{cor:final-repre}
For all $\eps, \delta>0$, all $k \ge O((N/\eps)^{2/3})$ if the width  $D$ of the random ReLU layer is at least $\Theta\left(\frac{\sqrt{N}\log(2(N)^2/\delta)}{\eps}\right)$, then, $w.p. \ge 1-\delta$ there exists an semi-orthonormal matrix $U \in \mathbb{R}^{D \times O(d^k)}$, and $\Delta \in \mathbb{R}^{D \times N}, \|\Delta\|_F < 2 \sqrt \eps$ such that, for the train matrix $X \in \mathbb{R}^{d \times N}$, for all $i$,
\begin{align}
     \sigma(CX_i) = U \phi_k(\vec{x})  + \Delta_i. 
\end{align}
where $X_i$ is the $i^{th}$ column of $X$ and $\Delta_i$ the $i^{th}$ column of $\Delta$.
\end{corollary}
\begin{proof}
For two input vectors $\vec{x},\vec{y}$, we have,
\begin{align*}
    \langle\phi(\vec{x}), \phi(\vec{y})\rangle = q_0 + q_1 \langle \vec{x}, \vec{y}\rangle + q_2 \langle \vec{x}^{\otimes 2},  \vec{y}^{\otimes 2}\rangle + q_3 \langle \vec{x}^{\otimes 4},  \vec{y}^{\otimes 4}\rangle + \ldots
\end{align*}
For any $J=(J_1,\dots, J_d) \in \mathbb{N}^d$, write a monomial $x^J=x_1^{J_1}\dots x_d^{J_d}$ and define $|J|=\sum_k J_k$. By definition, $\vec{x}^{\otimes i}$ is the vector of all monomials of the form $(x^J;\ |J|=i)$ and so,
$$\langle \vec{x}^{\otimes i}, \vec{y}^{\otimes i}\rangle = \sum_{|J|=i} x^J y^J  = \langle \vec{x}, \vec{y}\rangle^i,$$
where the last equality is just rearranging the terms of the power of the dot product. Therefore, we can write $\langle\phi(\vec{x}), \phi(\vec{y})\rangle = K(\vec{x},\vec{y})$ and $\phi(X)^\top \phi(X) = K(X, X)$. 
Now, since $\|\vec{x}\|_2 \le 1$, $\langle \phi_k(\vec{x}), \phi_k(\vec{y}) \rangle$ is a $O(1/k^{3/2})$ approximation to $\hat{\sigma}(\vec{x}^\top \vec{y})$ for all pairs $\vec{x}, \vec{y}$. Hence, we have that for $k = O(\left(N/\eps\right)^{2/3})$, $\|\phi_k(X)^\top\phi_k(X) - K(X,X)\|_F \le \eps$. Moreover from Lemma~\ref{lem:hidden-nodes-rep} we have that w.p. $\ge 1-\delta$, $\|Z_D^\top Z_D - K(X,X)\|_F \le \eps$ for our chosen width $D$.
Now we can use Lemma~\ref{lem:sqrt_close} to conclude that there exists a semi-orthogonal matrix $U \in \mathbb{R}^{D \times O(d^k)}$ and an error matrix $\Delta \in \mathbb{R}^{D \times N}$, such that,  $$\sigma(CX) = U \cdot\phi_k(X)  + \Delta$$ and $\|\Delta\|_F<2\sqrt{\eps}$.
\end{proof}

The following Lemma quantifies the norm of a function $p(\vec{x})$ given as a Taylor series when expressed in terms of a random ReLU kernel. We will assume, without essential loss of generality, that in the Taylor series of the random representation $\phi(\vec{x})$, for every monomial $x^J$ the corresponding coefficient $q_J$ is non-zero. This is because by adding a constant to our input with subsequent renormalization, i.e. $\vec{x'} = (\vec{x}/\sqrt{2},1/\sqrt{2})$ we can use as kernel $K'$ where $K'(\vec{x},\vec{y})=K(\vec{x},\vec{y})+1$ wherein all the monomials exist as the Taylor series of $\hat \sigma$ is non-negative (also see~\ificml\citet{agarwala2021one}\fi\ifarxv\cite{agarwal2017finding}\fi, Corollary {\color{mydarkblue} 3}, and also Lemma {\color{mydarkblue} 9} in there for a matching lower bound for expressing $p(x)$ in terms of a wide random ReLU layer for a certain distribution of inputs).

\begin{lemma}
\label{lem:linear-expressibility-app}
For any $\eps, \delta>0$ and multi-variate polynomial $p(\vec{x})= \sum_J p_J x^J$ w.p. $\ge 1-\delta$ we can approximate $p$ via the application of a random ReLU kernel of large enough width followed by a dot product with a vector $\vec{a}$, i.e. $\vec{a} \sigma(C\vec{x})$, so that $|p(\vec{x}) - \vec{a}\sigma(C\vec{x})|\le \eps$ for any $\vec{x}$ in our train samples and $\|\vec{a}\|_2^2 = \sum_J p_J^2 /q_J$ where $q_J$ is the coefficient of the monomial $\vec{x}^J\vec{y}^J$ in $\hat \sigma(\vec{x}^\top \vec{y})$.
\end{lemma}
\begin{proof}
This follows from Corollary~\ref{cor:final-repre} and taking $\vec{a}U$ to be the vector of the coefficients of $p$ divided by the appropriate coefficients of the Taylor series of $\phi(\vec{x})$. To ensure that every monomial has a non-zero coefficient in the Taylor series of the representation $\phi(.)$, we add a bias term to our input as described in the paragraph above.
\end{proof}

\subsection{Formalizing Bounded-Norm Analytic Functions: The \texorpdfstring{$q$}{q}-Norm}
Given the understanding developed so far, we now define a norm of an analytic functions which formalizes the intuition that we want our inverting analytic function $g(.)$ from Assumption~\ref{ass:main} to be expressible approximately using a wide enough random ReLU layer. We use Lemma~\ref{lem:linear-expressibility-app} to define a notion of norm for any analytic function $g$. Given the vector $\vec{q}$ of coefficients of the series $\phi(\vec{x})$, we will define $\|g\|_q$ to be the norm of $g$'s approximate representation using an infinitely wide random ReLU layer. That is given an infinite dimensional vector $\vec{a}$ and an infinitely wide random ReLU layer, let 
\begin{align}
\|g\|_q = \min_{\vec{a}, \vec{a}\sigma(C.) = g} \|\vec{a}\|_2.
\end{align}
We call $\|g\|_q$ the $q$-norm of $g$.
We can see that $\|g\|^2_q \le \sum_J g_J^2 /q_J$ where $g_J$ are coefficients of monomials in the representation of $g$ and $q_J$ are the coefficients of the Taylor series of $\phi(.)$.
We next present Lemmas which will show that for most natural well-behaved analytic functions which to not blow up to $\pm \infty$ the $q$-norm is bounded (see Remark~\ref{rem:qnorm-examples}).

The following lemma from\ificml\citet{agarwala2021one}\fi\ifarxv\cite{agarwal2017finding}\fi bounds $\|g\|_q$ for univariate functions -- there the notation $\sqrt {M_g}$ was used for $\|g\|_q$ instead just as in \cite{arora2019fine}.
\begin{theorem}\label{thm:univar}\ificml\citet{agarwala2021one}\fi\ifarxv\cite{agarwal2017finding}\fi
Let $g(y)$ be a function analytic around $0$, with radius of convergence
$R_{g}$.
Define the \emph{auxiliary function} $\tilde{g}(y)$ by the power series
\begin{equation}
\tilde{g}(y) = \sum_{k=0}^{\infty} |a_{k}| y^k
\end{equation}
where the $a_k$ are the power series coefficients of $g(y)$. Then the function
$g(\bbet\cdot\x)$ satisfies,
\begin{equation}
\|g\|_q \le \bnorm \tilde{g}'(\bnorm)+\tilde{g}(0)
\end{equation}
if the norm $\bnorm\equiv\|\bbet\|_{2}$ is less than $R_{g}$.
\end{theorem}

The tilde function is the notion of complexity which relates to the $q$-norm. Informally, the tilde function makes all coefficients in the Taylor series positive. The $q$-norm is essentially upper bounded by the value of the derivative of function at $1$ (in other words, the L1 norm of the coefficients in the Taylor series). For a multivariate function $g(\x)$, we define its tilde function $\tilde{g}(y)$ by substituting any inner product term in $\x$ by a univariate $y$. The above theorem can then also be generalized to multivariate analytic functions:

\begin{theorem}\ificml\citet{agarwala2021one}\fi\ifarxv\cite{agarwal2017finding}\fi
\label{thm:multivar}

Let $g(\x)$ be a function with multivariate power series representation:
\begin{equation}
g(\x) = \sum_{k} \sum_{v\in V_k} a_{v} \prod_{i=1}^{k} (\bbet_{v,i}\cdot\x)
\end{equation}
where the elements of $V_k$ index the $k$th order
terms of the power series. We define $\tilde{g}(y) = \sum_{k} \tilde{a}_{k} y^k$
with coefficients 
\begin{equation}
\tilde{a}_{k} = \sum_{v\in V_{k}} |a_{v}|\prod_{i=1}^{k}\bnorm_{v,i}.
\end{equation}

If the power series of $\tilde{g}(y)$ converges at $y=1$ then  $\|g\|_q \le \tilde{g}'(1)+\tilde{g}(0)$.
\end{theorem}

Let $g^+(\vec{x})$ denote the same Taylor series as $g(\vec{x})$ but where all coefficients have been replaced by their absolute value. Let $\|g\|_{qu}$ denote the {\em upper bound} $\tilde g'(1) +  \tilde g'(0)$ as in Theorem~\ref{thm:multivar} which ensures that $\|g\|_{q} \le \|g\|_{qu}$. The following claim is evident from the expression for $\|g\|_{qu}$.

\begin{claim}
The $q$-norm of an analytic function $g$ satisfies the following properties.
\begin{itemize}

\item $\|g\|_{q} \le \|g\|_{qu}$

\item $\|g\|_{qu} =  \|g^+\|_{qu}$.

\item $\|g^+_1 + g^+_2\|_{qu} = \|g^+_1\|_{qu} + \|g^+_2\|_{qu}$.

\end{itemize}
\end{claim}

\begin{corollary}\label{cor:power_q_norm}
If for $s$ functions $g_1(\x),.,g_s(\x)$ functions $\tilde{g_i}'(1) \le O(1), \tilde{g_i}(1) \le O(1)$, then $\|(\sum_i g_i(\x))^{c}\|_{qu} \le c(O(s))^c$
\end{corollary}
\begin{proof}
Let $f(\x) = (\sum_i g_i(\x))^{c}$. Then 
$\|f\|_q \le \tilde{f}'(1)+\tilde{f}(0)$ where $\tilde{f}(y) = (\sum_i \tilde{g}_i(y))^{c}$. So $\tilde{f'}(1) = c(\sum_i \tilde{g}_i(1))^{c-1}(\sum_i \tilde{g'}_i(1)) = c(O(s))^{c-1} O(s) = c(O(s))^c$.
And $\tilde{f}(0) \le \tilde{f}(1) \le (O(s))^c$.
\end{proof}

\begin{remark}
\label{rem:qnorm-examples}
Most analytic functions which do not blow up to $\pm \infty$ and are Lipschitz and smooth will have a bounded $q$-norm according to our definition.
As a concrete example to gain intuition into $q$-norms of analytic functions, the function $f(\vec{x}) = e^{\bbet_1 \cdot \x} \cdot sin(\bbet_2 \cdot\x) + cos(\bbet_3 \cdot \x)$ has constant $q$-norm if $\bbet_1, \bbet_2, \bbet_2$ all have a constant norm.
\end{remark}

\section{Properties of Local Minima}
In the previous section, we have seen that the representation computed by a random ReLU layer is expressive enough to approximate `well-behaved' analytic functions. In this section we will leverage this understanding to show that (a) there are good ground truth weight matrices $A^*, B^*$ which learn to classify our train manifolds well while satisfying the GSH property, (b) and consequently any local minima of our optimization will also be a good classifier for our train data and satisfy the GSH.
We start with point (a). We will assume the that the $g()$ function satisfies the conditions of Corollary~\ref{cor:power_q_norm}. 

\label{sec:opt-app}
\begin{lemma}[Existence of Good Ground Truth]
\label{lem:good-ground-truth-app}
Given our 3-layer architecture, there exist ground truth matrices $A^*, B^*$ such that for any $\epsilon_1, \epsilon_2>0$, with probability $\ge 1-\delta$,
\begin{enumerate}
    \item $\cL_{A^*, B^*}(Y, \hat{Y}) \le \epsilon_1$,
    \item $\|A^*\|_F^2 \le m$,
    \item $\|B^*\|_F^2 \le s^{O(\log(1/\epsilon_1))}$, 
    \item $\hat{V}_{mn}(B^*\sigma(C.)) \le \epsilon_2$.
\end{enumerate}
\end{lemma}
\begin{proof}
The desired output $\vec{y}$ is a non-continuous function whose outputs are either $0$ or $1$.
We will approximate each coordinate of the output $\vec{y}$ by a continuous polynomial. First we recall that for any two distinct $\vec{\gamma_1}$ and $\vec{\gamma_2}$ from our distribution $\cM$ we have $\langle \vec{\gamma_1} , \vec{\gamma_2}\rangle \le \frac{1}{\sqrt{s}}$ by assumption of $\tau$-separatedness. 
For any $\epsilon > 0$, define
\begin{align}
    \mu(\bu,\bv) =  \langle \bu, \bv\rangle^{c\log(1/\epsilon)}.
\end{align}
where $\bu, \bv$ are vectors representing two possible values of $\bgamma$ and $c\ge 1/\log(1/\tau)$ is a constant chosen so that $c\log(1/\epsilon)$ is an integer. Then we have, $\mu(\bu,\bv) = 1$ if and only if $\bu = \bv$ and if $\langle \bu ,\bv \rangle \le \tau$, then $\mu(\bu, \bv) \le \epsilon$.


Hence, for $\vec{x_l}$ sampled from manifold $M_l$ we have that $g(\vec{x_l}) = \vec{\bgamma_l}$ and,
$$\mu(\vec{\bgamma_j}, g(\vec{x_l})) = \begin{cases}
\eps \text{ (at most}) & l\ne j,\\
1& l=j\\
\end{cases}$$
Let $\vec{y_l}^* = (\mu(\vec{\bgamma_1}, g(\vec{x_l})), \mu(\vec{\bgamma_2}, g(\vec{x_l})), \ldots, \mu(\vec{\bgamma_m}, g(\vec{x_l})))$.
Then we have that the weighted square loss term corresponding to $\vec{x_l}$ $\| \vec{w_l} \odot (\vec{y_l} - \vec{y_l}^* )\|_2^2 \le \eps^2/2$.
Based on Corollary~\ref{cor:final-repre} without loss of generality, we assume that the random ReLU layer outputs the monomials $\Phi(\vec{x})$ in co-ordinates of $\bbx$ \footnote{In reality there is an additional orthogonal matrix $U$ but we can define $B^*_2 = B^*U$ and subsume it in our ground truth.}. We will now find matrices $A^*, B^*$ so that, $A^* B^*\sigma(C\vec{x_l}) = \vec{y_l}^*$ approximately. Then bullet \emph{1} of the Lemma will immediately follow.

To do so, we express,
\begin{align}
    \mu(\bu, \bv) = \langle\psi(\bu), \psi(\bv)\rangle,
\end{align}
where $\psi(\bu)$ and $\psi(\bv)$ are bounded-norm vectors.
We do this using the binomial expansion of $\langle \bu, \bv\rangle^{c\log(1/\epsilon)}$. We can write it as a weighted sum of monomials where each monomial is a product of two similar monomials in $\bu$ and $\bv$. We can enumerate these monomials by their degree distribution. Let $J=(J_1,\dots, J_d) \in \mathbb{N}^d$ denote the degree distribution of a monomial in $d$ variables. We will use the notation $x^J=x_1^{J_1}\dots x_d^{J_d}$ to denote such a monomial over $\bbx$. Then $|J|=\sum_k J_k$ is the degree of the monomial.
The expanded expression for $\mu(\bu,\bv)$ can be written as $\sum_{J: |J|=c\log(1/\epsilon)} a_J u^J v^J$. This in turn can be written as a dot product of two vectors whose dimension equals the total number of monomials of degree $c\log(1/\epsilon)$ in $s$ variables, which is $\binom{c\log(1/\eps) + s-1}{s-1} = O\left(s^{c\log(1/\eps)}\right)$. So precisely, $\mu(\bu,\bv) = \psi(\bu) \cdot \psi(\bv)$ where $\psi(\bu)$ is a vector whose coordinates can be indexed by the different values of $J$ and the value at the $J^{th}$ coordinate is $(\psi(\bu))_J =  \sqrt{a_J} u^J$. Clearly then $\langle\psi(\bu), \psi(\bv)\rangle = \sum_J a_J u^J v^J = \mu(\bu,\bv)$.

We will now describe the matrices $A^*, B^*$. For now, assume that the random ReLU kernel $\sigma(C.)$ is of infinite width. We will choose the width of the hidden layer (number of rows in $A^*$) to be exactly the number of different values of $J$. This width can be reduced to $O(\log(mn)/\eps^2)$ at the expense of an additional $\eps$ error per output coordinate of $A^*$ as shown in Lemma~\ref{lem:bounded-width}. Given this width, we simply set the $l^{th}$ row $A^*_l = \psi(\vec{\bgamma_l})$.
Then $B^*$ is chosen such that the output of the hidden layer $r = B^* \sigma(C\vec{x_l}) \approx \psi(g(\vec{x_l}))$. To see that such a $B^*$ exists, note that we need the $J^{th}$ coordinate of $r$, $r_J = \sqrt{a_J} (g(\vec{x_l}))^J$. Since $g(\vec{x_l})$ is analytic with a bounded norm, the $\sqrt{a_J} (g(\vec{x_l}))^J$ are also bounded-norm analytic functions in $\vec{x_l}$ and so by Lemma~\ref{lem:linear-expressibility} these can be expressed using a linear transform of $\sigma(C\vec{x_l})$ (as the width goes to infinity). So $B^*_J$ is chosen such that $B^*_J \cdot \sigma(C\bbx) = \sqrt{a_J} (g(\bbx))^J$. Now let us look at the Frobenius norms of $A^*, B^*$ constructed above. First $\|A^*\|_F^2 = \sum_{l=1}^m \|A_l^*\|_2^2=m$, since,
$$\|A^*_l\|_2^2 = \langle\psi(\vec{\bgamma_l}), \psi(\vec{\bgamma_l})\rangle = g(\vec{\bgamma_l}, \vec{\bgamma_l}) = 1.$$ 
Next, we can use Lemma~\ref{lem:linear-expressibility} to express the norm of $B^*$ as $\|B^*\|^2_F=\sum_J a_J \|g(\vec{x})^J\|^2_{\vec{q}}$ where the $q_J$s are the coefficients of $\Phi(\vec{x})$.  
Note that this is independent of $m$ and given $g(\bbx), \delta$ it only depends on $\epsilon$ therefore we can write $\|B^*\|_F^2=T(\eps)$ where $T$ is only a function of $\eps$. Note that $T(\eps) = \sum_J a_J \|g(\x)^J\|^2_{\vec{q}} \le \sum_J a_J \|g^+ (\x)^J\|^2_{\vec{qu}} \le \|\sum_J  a_J g^+ (\x)^J\|^2_{\vec{qu}} \le \|(\sum g^+_i (\x))^{c\log(1/\epsilon)}\|^2_{\vec{qu}}$. By Corollary~\ref{cor:power_q_norm} this is at most $s^{O(\log(1/\eps))}$.

Moving to the bound on $V_{\text{reg}}$, this is easy to see once we note that $B^*$ is such that for any $l \in [m]$, for all $i \in [n]$, $B^*\sigma(C\x_{il}) \approx p(\x_{il}) = \vec{\bgamma_l}$ and hence has very low intra-manifold variance.

So far we assumed the random ReLU layer to be monomials $\Phi(\vec{x})$ according to infinite width kernel. Now, we argue that if we use a large enough width $D$, then by Corollary~\ref{cor:final-repre} there is an orthogonal matrix $U$ so that $\sigma(C\vec{x})$ is approximately $U\Phi(\bbx)$. If we choose $D$ so that $\|\sigma(C\vec{x}) - U\Phi(\bbx)\|_2$ is at most $\eps/T(\eps)$ then $B^*\Phi(\vec{x})$ will differ from $B^*U\sigma(C\vec{x})$ by at most Frobenius error $\eps$ on any of the $n$ inputs; this will result in at most additive error $\eps$ at each of the outputs in $Y_i$ (since each row of $A^*$ has norm at most $1$. This is done by setting $D = O(\sqrt n T(\eps)^2 \log (n/\delta)/\eps^2)$.


\end{proof}

\begin{lemma}[Bounding the Width of the Hidden layer]
\label{lem:bounded-width}
Given any $\epsilon_1, \epsilon_2 > 0$, and $A^*,B^*$ of Lemma~\ref{lem:good-ground-truth-app}, we can construct new $A', B'$ with number of columns in $A'$ (and number of rows in $B'$) equal to $O(\log(mn)\log(1/\delta)/\epsilon_1)$, such that 
\begin{enumerate}
    \item $\cL_{A^*, B^*}(Y, \hat{Y}) \le \epsilon_1$,
    \item $\|A'\|_F^2 \le m$, $\|B'\|_F^2 \le s^{O(\log(1/\epsilon_1))}$,
    \item $\hat{V}_{mn}(B'\sigma(C.)) \le \epsilon_2$.
\end{enumerate}
Note that now we have the small loss guarantee only on our train examples and not over any new samples from our manifolds.
\end{lemma}
\begin{proof}
Let the original width of the hidden layer (number of columns in $A^*$) be $w$. From Lemma~\ref{clm:rand-proj-jl}, we have that randomly projecting both $A^*$ and $B^*$ down to $O(\log(mn)\log(1/\delta)/\epsilon^2)$ dimensions preserves all the dot products between the normalized rows of $A^*$ and normalized columns of $B^*$ up to an additive error $\epsilon$ with probability $\ge 1-\delta$. In addition we have that $\|A^*_l\| = 1$ for all $l \in [m]$. So we can replace $A^*B^*$ by $A'B' = (A^*R^\top)(RB^*)$ where $R$ is the random projection matrix and get that for each input $\vec{x_{il}}$, $\|A'B'\sigma(C\vec{x_{il}}) - A^*B^*\sigma(C\vec{x_{il}})\|_{\infty} \le \epsilon b$ where $b$ is the maximum norm of the rows of $B^*$.
As an aside, we note that a similar random projection can be applied on top of the random ReLU layer $\sigma(C.)$ as well to get a random ReLU layer followed by a random projection neither of which are trained and resulting in a smaller width ReLU layer.
\end{proof}

Next, we recall that our objective is of the form
\begin{align}
    \min_{A,B} \cL_{A,B}(Y, \hat Y) + \|A\|_F^2 + \|B\|_F^2 \label{eq:general-objective-app}
\end{align}
We will argue that the nice properties we saw holding for $A^*, B^*$ also hold for any global minima of our optimization \eqref{eq:general-objective-app}.
This is because of the following lemma.
\begin{lemma}[Multi-Objective Optimization]
\label{lem:mutli-obj-opt}
Given a multi-objective minimization where we want to minimize a set of non-negative functions $O_i(\theta)$ for $i=1,\ldots,q$ and there exists a solution $\theta^*$ such that $O_i(\theta^*) \le OPT_i$. Then, we have that 
\begin{align}
    \min_{\theta} \sum_{i=1}^q \frac{O_i(\theta)}{OPT_i} \notag
\end{align}
produces $\hat{\theta}$ such that for each $i$, $O_i(\hat \theta) \le q OPT_i$ at any global minimum.
\end{lemma}
\begin{proof}
Note that at global minimum 
\begin{align}
&\sum_{i=1}^q \frac{O_i(\theta)}{OPT_i} \le \sum_{i=1}^q \frac{O_i(\theta^*)}{OPT_i} \le \sum_{i=1}^q 1 = q. \notag
\end{align}
Since $O_i$ are non-negative functions we have $O_i(\theta) \le q OPT_i$.
\end{proof}
Lemma~\ref{lem:mutli-obj-opt} will guide our choice of regularization parameters $\lambda_1, \lambda_2$.

\begin{lemma}
\label{lem:global-opt-reg-params}
Let $\hat{A}^*, \hat{B}^*$ denote the global optimum of \eqref{eq:general-objective-app}. Then, for $\lambda_1 = \eps_1/m, \lambda_2 = \eps_1/s^{O(\log(1/\eps_1))}$, we have
\begin{align}
    &\cL_{\hat{A}^*,\hat{B}^*}(Y, \hat Y) \le 3\epsilon_1, \\
    &\|\hat{A}^*\|_F \le 3m, \|\hat{B}^*\|_F \le s^{O(\log(1/\eps_1))}.
\end{align}
\end{lemma}
\begin{proof}
Recall that for ground truth $A^*, B^*$ from Lemma~\ref{lem:good-ground-truth-app} we have that $\cL_{A^*,B^*}(Y, \hat Y) \le \epsilon_1$, $\|A^*\|_F^2 \le m$ and $\|B\|_F^2 \le s^{O(\log(1/\eps_1))}$. Therefore, setting $\lambda_1 = \frac{\epsilon_1}{m}$ and $\lambda_2 = \frac{\epsilon_1}{s^{O(\log(1/\eps_1))}}$, we get from Lemma~\ref{lem:mutli-obj-opt} that at global minimum $\hat{A}^*, \hat{B}^*$
\begin{align}
    &\cL_{\hat{A}^*,\hat{B}^*}(Y, \hat Y) \le 3\epsilon_1, \\
    &\|\hat{A}^*\|_F \le 3m, \|\hat{B}^*\|_F \le 3s^{O(\log(1/\eps_1))}.
\end{align}
\end{proof}
Note that the chosen values of $\lambda_1, \lambda_2$ will influence the number of steps gradient descent will need to run to reach a local optimum.

Since our objective is non-convex, it is not clear how good a local optimum we reach will be. However, for our particular architecture, it turns out that every local minimum is a global minimum.
\begin{lemma}[Equivalence to Nuclear Norm Regularized Convex Minimization]
\label{lem:frobenius-nuclear-reg-eqv}
For any convex objective function $O()$, in the minimization
\begin{align}
    \min_{A,B} O(AB) + \lambda\left(\|A\|_F^2 + \|B\|_F^2 \right), \tag{P1} \label{eq:frobenius-min}
\end{align}
all local minima are global minima and the above minimization is equivalent to the following convex minimization
\begin{align}
    \min_{A, B} O(AB) + 2\lambda \left(\|AB\|_{*} \right) \equiv \min_{W} O(W) + 2\lambda \left(\|W\|_{1} \right). \tag{P2} \label{eq:nuclear-min}
\end{align}
\end{lemma}
\begin{proof}
From Lemma~\ref{lem:frobenius-nuclear-equiv-folklore}, it follows that the global minimum of \eqref{eq:frobenius-min} and \eqref{eq:nuclear-min} have the same value. Note that the latter minimization is convex and hence any local minima is global. We now show that all local minima of \eqref{eq:frobenius-min} are global as well even though it is potentially a non-convex objective. Let $OPT$ denote the value of the global minimum of either objective and let $A_1, B_1$ be a local minima of \eqref{eq:frobenius-min}. Suppose for the sake of contradiction that $O(A_1B_1) + \lambda(\|A_1\|_F^2 + \|B_1\|_F^2) > OPT$. Then it must be the case that $\|A_1\|_F^2 + \|B_1\|_F^2 = 2\|A_1B_1\|_*$ as otherwise by Lemma~\ref{lem:frobenius-nuclear-equiv-folklore} we will be able to improve the objective by keeping $A_1B_1$ a constant and reducing $\|A_1\|_F^2 + \|B_1\|_F^2$ (note that the sum of Frobenius norms given a fixed product of $AB$ is a convex minimization problem). Therefore we have that $O(A_1B_1) + 2\lambda(\|A_1B_1\|_*) > OPT$. Since \eqref{eq:nuclear-min} is a convex problem, this implies that for any $\epsilon > 0$, within an $\epsilon$-sized ball around $W_1 = A_1B_1$ there exists $W_2$ such that $O(A_1B_1) + 2\lambda(\|A_1B_1\|_*) > O(W_2) + 2\lambda(\|W_2\|_*)$. Let $A_2 = B_2 = W_2^{1/2}$. Then we have that $2\|A_2B_2\|_* = \|A_2\|_F^2 + \|B_2\|_F^2$ and hence $O(A_2B_2) + 2\lambda(\|A_2\|_F^2 + \|B_2\|_F^2) < O(A_1B_1) + 2\lambda(\|A_1\|_F^2 + \|B_1\|_F^2)$ which is a contradiction to the statement that $A_1, B_1$ is a local minima of \eqref{eq:frobenius-min}.
\end{proof}

\begin{corollary}[Generalization of Lemma~\ref{lem:frobenius-nuclear-reg-eqv}]
\label{cor:all-local-global}
For any convex objective function $O()$,
\begin{align}
    \min_{A,B} O(AB) + \left(\lambda_1 \|A\|_F^2 + \lambda_2 \|B\|_F^2 \right), \notag
\end{align}
all local minima are global minima and is equivalent to the following convex objective
\begin{align}
    \min_{A,B} O(AB) + 2\sqrt{ \lambda_1 \lambda_2 }\left(\|AB\|_{*} \right). \notag
\end{align}
\end{corollary}
\begin{proof}
The lemma follows by replacing $A, B$ in the previous lemma by $\sqrt {\lambda_1 / \lambda_2 } A, \sqrt {\lambda_2 / \lambda_1 } B$ respectively and setting $\lambda$ to $\sqrt{ \lambda_1 \lambda_2 }$
\end{proof}

Corollary~\ref{cor:all-local-global} will imply that at any local minimum $\hat{A}, \hat{B}$ we have a small value for our weighted square loss. This is because $\cL_{A,B}(Y, \hat{Y})$ is convex in $AB$. 
Next, we will show that an empirical variant of the GSH property holds for the representation $\hat{B}\sigma(C.)$ obtained at any local minimum.
Here our approaches for linear and non-linear manifolds differ. Linear manifolds enable a more direct analysis with a plain $\ell_2$-regularization. However, we need to assume certain additional conditions on the input. The result for linear manifolds acts as a warm-up to our more general result for non-linear manifolds where we have minimal assumptions but end up having to use a stronger regularizer designed to push the representation to satisfy GSH. We describe these differences in Sections~\ref{sec:linear-app} and \ref{sec:var-reg-app}.

\subsection{GSH on Train Data for Linear Manifolds}
\label{sec:linear-app}
Here we will show that we can train our 3-layer non-linear neural network on input data from linear manifolds, to get GSH. To get an intuitive understanding of why this is the case, we first recall that by passing an input vector $\vec{x}$ through a random ReLU layer, we get approximately all possible monomials of $\vec{x}$ and its higher tensor powers (Corollary~\ref{cor:final-repre}). Now, we will show that by passing in a dummy constant as part of the input, the regularization on the weights $A$ and $B$ enforces that weights corresponding to certain monomials of $\vec{x}$ are zero at any minima. These weights being zero will imply Property~\ref{eq:hash-near} of the hashing property. The second part of the hashing property will follow due to a similar reasoning as in Section~\ref{sec:var-reg-app}. With this high level intuition in mind, we proceed with the formal proof.

A linear manifold with a latent vector $\bgamma$ can be represented by the set $\{\vec{x} = P\btheta + Q\bgamma\}$ for some matrices $P$ and $Q$. Moreover, without a significant loss of generality we can assume that $\bgamma$ is such that $Q\bgamma \perp P\btheta$ for all $\btheta \in \cS^{k-1}$ (as otherwise, we can project $Q\bgamma$ onto the subspace perpendicular to $P\btheta$).
The objective function we minimize is \eqref{eq:general-objective-app}.


We begin the proof by first performing a transformation on the input that will simplify the presentation. 
\begin{lemma}
\label{lem:basis-change}
Given a point $\vec{x} = P\btheta + Q\bgamma$ where $P\btheta \perp Q\bgamma$, there exists an orthogonal matrix $U_2$
$$\sigma(C\vec{x}) = U_2\phi(\bgamma', \btheta') + \delta,$$
where $\bgamma' = RQ\bgamma$ and $\btheta' = RP\btheta$ and $U$ and $\phi(.)$ are as defined in Corollary~\ref{cor:final-repre}.
\end{lemma}
\begin{proof}
Since $P\btheta \perp Q\bgamma$, a rotation of the bases transforms $Q\bgamma$ to a vector with non-zero entries only in the first $d-k$ coordinates, and $P\btheta$ to lie in a subspace which contains vectors with non-zero entries only in the last $k$ coordinates. This is made feasible since the rank of the space spanned by $\btheta$ is $\le k$. Denote the vectors obtained after these transformations by $\bgamma'$ and $\btheta'$. We drop the 0 entries to get $\bgamma' \in \mathbb{R}^{d-k}$ and $\btheta' \in \mathbb{R}^{k}$.
Therefore, $\vec{x} = R_2(\bgamma', \btheta')$ for some rotation matrix $R_2$. Note that rotation matrices are orthogonal. 
Now $\sigma(C\vec{x}) = \sigma(CR_2(\bgamma', \btheta'))$.
$CR$ is also a random matrix distributed according to $\cN\left(0,\frac{I}{D}\right)$ and hence Corollary~\ref{cor:final-repre} applies to it as well giving us the statement of the Lemma.
\end{proof}

In light of Lemma~\ref{lem:basis-change}, we can assume that our neural net gets as input $\vec{\tilde{x}} = (\bgamma', \btheta')$ where $\bgamma' \in \mathbb{R}^{(d-k)}$ and $\btheta' \in \mathbb{R}^k$ without loss of generality as after passing through the random ReLU layer all that differs between the two views is the orthogonal matrix $U$ which is applied to $\phi(\vec{x})$. In addition to the constant $1/\sqrt{2}$ appending to our input originally, we append a constant $\frac{\sqrt{k(d+1)}}{\sqrt{(d+k)}}$ as well to $\vec{\tilde{x}}$ before passing it to our neural network as this will help us argue GSH.

\subsubsection{Property~\texorpdfstring{\eqref{eq:hash-near}}{(A)} of GSH for Linear Manifolds}
A key part of our argument for why we can get neural nets to behave as hash functions over manifold data is the observation that at the output layer having a small variance over points from the same manifold benefits the primary component of the loss. 
The following lemma formalizes the above intuition focusing on a single manifold. Note that this result holds for non-linear manifolds too.

\begin{lemma}[Centering]
\label{lem:mean-reduces-square-loss}
Let $M$ be one of the train manifolds with associated latent vector $\bgamma$. For each $\vec{x} \sim \cD(M)$, replacing $\vec{\hat y}$ by $\vec{\hat y'} = \E_{n}[\vec{\hat y}|\bgamma] = \frac{1}{n}\sum_{i=1}^n\vec{\hat y_i}$ will reduce the (weighted) square loss term corresponding to $M$
$$\cL(Y, \hat{Y}) = \frac{1}{n}\sum_{i=1}^n\sum_{l=1}^m \| \vec{w_{l}} \odot (\vec{\hat{y}_{il}} - \vec{y_{il}})\|_2^2$$
by at least $\hat{V}_{mn}(\vec{\hat y})$.
\end{lemma}
\begin{proof}
We start by focusing on a single manifold with latent vector $\bgamma$. We drop the conditioning on $\bgamma$ to simplify the proof.
Note that if there is no weighting of different coordinates of $\vec{y}$ according to $\vec{w_l}$, then
\begin{align}
\E_{n} [\|\vec{y} - \vec{\hat{y}}\|_2^2] &= \E_{n}[\|\vec{y} - \vec{\hat y'} + \vec{\hat y'} - \vec{\hat y}\|_2^2]
= \|\vec{y} - \vec{\hat y'}\|_2^2 + \E_{n}[\|\vec{\hat y'} - \vec{\hat y}\|_2^2] + 2\E_{n}[(\vec{y} - \vec{\hat y'})^\top(\vec{\hat y'} - \vec{\hat y})] \notag \\
&= \|\vec{y} - \vec{\hat y'}\|_2^2 + \E_{n}[\|\vec{\hat y'} - \vec{\hat y}\|_2^2] + 2(\vec{y} - \vec{\hat y'})^\top\E_{n}[\vec{\hat y'} - \vec{\hat y}] \notag\\
&= \|\vec{y} - \vec{\hat y'}\|_2^2 + \hat{V}_n(\vec{\hat y} | \bgamma) + 0.
\end{align}

So the value of $(1/m) \E_{n} [\|\vec{y} - \vec{\hat y}\|_2^2]$ reduces by at least $\hat{V}_n(\vec{\hat y} | \bgamma))/m$ upon replacing $\vec{\hat y}$ by its average value per manifold $\vec{\hat y'}$.
This holds even when there is weighting according to the $\vec{w}$ matrix as it only depends on $\bgamma$ and doesn't vary based on $\btheta$.

\begin{align}
\E_{n} [\|\vec{w}_{\bgamma} \odot (\vec{y} - \vec{\hat y})\|_2^2] &= \E_{n}[\|\vec{w}_{\bgamma}\odot (\vec{y} - \vec{\hat y'} + \vec{\hat y'} - \vec{\hat y})\|_2^2] \notag \\
&= \|\vec{w}_{\bgamma}\odot (\vec{y} - \vec{\hat y'})\|_2^2 + \E_{n}[\|\vec{w}_{\bgamma}\odot (\vec{\hat y'} - \vec{\hat y})\|_2^2] + 2\E_{n}[\left(\vec{w}_{\bgamma}\odot (\vec{y} - \vec{\hat y'})\right)^\top\left(\vec{w}_{\bgamma} \odot (\vec{\hat y'} - \vec{\hat y})\right)] \notag\\
&= \|\vec{w}_{\bgamma}\odot (\vec{y} - \vec{\hat y'})\|_2^2 + \E_{n}[\|\vec{w}_{\bgamma}\odot (\vec{\hat y'} - \vec{\hat y})\|_2^2] + (\vec{w}_{\bgamma}\odot (\vec{y} - \vec{\hat y'}))^\top\left(\vec{w}_{\bgamma}\odot \E_{n}[\vec{\hat y'} - \vec{\hat y} | \bgamma]\right) \notag\\
&= \|\vec{w}_{\bgamma}\odot (\vec{y} - \vec{\hat y'})\|_2^2] + \E_{n}[\|\vec{w}_{\bgamma}\odot (\vec{\hat y'} - \vec{\hat y})\|_2^2] + 0. \notag
\end{align}

Thus even in the weighted case the weighted square loss gets reduced by at least $\E_{n}[\|\vec{w}_{\bgamma}\odot (\vec{\hat y'} - \vec{\hat y})\|_2^2]$. But since $\vec{w}_{\bgamma}$ is at least $1/\sqrt{2(m-1)}$ per coordinate, this is at least $\hat{V}_n(\vec{\hat y} | \bgamma)/(2(m-1))$. Summing up this reduction over all the $m$ manifolds, we get the lemma.
\end{proof}

The following lemma will show that it is in fact beneficial to have a zero variance at the representation layer itself rather than just at the output layer. This also holds generally across linear and non-linear manifolds.
\begin{lemma}
\label{lem:var-rep-output-relation}
$\hat{V}_{mn}(\vec{\hat y}) = 0$ if and only if the variance of the representation layer $\hat{V}_{mn}(\vec{r}) = 0$ where $\vec{r} = B \sigma(C\vec{x})$.
Further $\hat{V}_{mn}(\vec{\hat y}) \ge (\lambda_2 / \lambda_1)(\hat{V}_{mn}(\vec{r}))^2/4$ where $\lambda_1, \lambda_2$ are the regularization weights.
\end{lemma}
\begin{proof}
Recall that for a non-square matrix $W = USV^\top$, we define the square root as $W^{1/2} = US^{1/2}V^\top$.
From Corollary~\ref{cor:minima-sum-frobenius-charac}, we have that at local minima of \eqref{eq:linear-loss}, if $\vec{\hat y} = W \vec{z}$ then without loss of generality $\vec{r} = W^{1/2} \vec{z}$ (upto orthonormal rotation and scaling). Let $\vec{z'} = \vec{z} - \E_{n}[\vec{z}]$ where the mean value of $z$ per manifold has already been subtracted. Let $Z'$ denote the matrix of all such $\vec{z'}$ scaled by $1/\sqrt n$. Since $\vec{\hat y}, \vec{r}$ are linear transforms of $\vec{z}$, it is not hard to see that $\hat{V}_{mn}(\vec{\hat y}) = \|W Z'\|_F^2$ and similarly $\hat{V}_{mn}(\vec{r}) = \|W^{1/2}Z'\|_F^2$. Now, if $\|W^{1/2}Z'\|_F = 0$, then $\|W Z'\|_F = 0$ clearly. 
To see the other direction, we first observe that multiplying $Z'$ by a matrix $W$ is the same as taking dot products of the columns of $Z'$ with the right singular vectors of $W$ and scaling the result by the singular values of $W$. Now, the right singular vectors of $W$ and $W^{1/2}$ are the same and the singular value of $W = 0$ iff the corresponding singular value of $W^{1/2}$ is 0 as well.
Therefore, if $\|WZ'\|_F = 0$, then so is $\|W^{1/2} Z'\|_F$.

Furthermore, the singular values of $W^{1/2}$ are square roots of the singular values of $W$. So $\|W^{1/2} Z'\|_F$ can be non-zero only if and only if $Z'$ has component along a singular vector of $W^{1/2}$ with non-zero singular value and the same must be true for $\|W Z'\|_F$ as well. Note that since $W$ has $m$ rows there are at most $m$ singular values $\sigma_1,\ldots,\sigma_m$. Let $c_1,\ldots,c_m$ be the total norm squared of $Z'$ along the right singular vectors, that is if $W = U D V^\top$, then $c_i = \|V_{i,*}^\top Z'\|_2^2$. Since for any $\vec{z}$, $\|z\|_2 \le \|C\|_2\alpha$, same must be true for $\vec{z'}$. Hence, $\sum c_i \le \|C\|_2^2\alpha^2$. Now, $\hat{V}_{mn}(\vec{\hat y}) = \sum \sigma_i^2 c_i$ and $\hat{V}_{mn}(\vec{r}) = \sum \sigma_i c_i$. Now in the latter, the sum from those singular values that are at most $\hat{V}_{mn}(\vec{r})/2$ is at most $\hat{V}_{mn}(\vec{r})/2$ and so the rest must be coming from singular values larger than $\hat{V}_{mn}(\vec{r})/2$. Since the singular vectors are getting squared we have $\hat{V}_{mn}(\vec{\hat y}) \ge (\hat{V}_{mn}(\vec{r})/2) (\hat{V}_{mn}(\vec{r})/2)  = (\hat{V}_{mn}(\vec{r}))^2/4$. If the weights of $\|A\|^2_F, \|B\|^2_F$ are $\lambda_1, \lambda_2$ then $B = \sqrt{\frac{\lambda_1}{\lambda_2}} W^{1/2}$ and the statement of the lemma follows.
\end{proof}

Next we show that if the intra-manifold variance of the representation $\vec{r} = B\sigma(C\vec{x})$ is large for a certain weight matrix $B$ then replacing $\vec{r}$ by its mean value per manifold leads to a reduction in the intra-manifold variance down to a small value. Moreover, this can be achieved by using a $B'$ such that $\|B'\|_F \le \|B\|_F$. 
The main idea is more easily exposited by first assuming we have an infinite width random ReLU layer. Hence, we first state the following lemma which shows that we can push the intra-manifold variance of the representation all the way down to 0 if $D \to \infty$.
\begin{lemma}
\label{lem:weight-switch-arg-inf-width}
For $D \to \infty$, given as input $\vec{\tilde{x}'} = (\bgamma', \btheta', \frac{\sqrt{k(d+1)}}{\sqrt{(d+k)}})$, and given a $B$ we can transform it to $B'$ with no greater Frobenius norm so that $\hat{V}_{mn}(B'\vec{z}) = 0$.
\end{lemma}
\begin{proof}
First let us assume that instead of the constant $\frac{\sqrt{k(d+1)}}{\sqrt{(d+k)}}$ being appended to the input, we have $k$ 1s appended. We will later argue that the features computed by the ReLU layer are equivalent for both these cases.
Let $\sigma(C\vec{x}) = U_2\phi(\bgamma', \btheta') + \bdelta$. When $D \to \infty$, we have from Corollary~\ref{cor:final-repre} that $\bdelta \to 0$. Consider the matrix $B_2 = BU_2$. Then $B_2$ can be viewed as a linear mapping from different monomials in the representation computed by the random ReLU layer to a new representation space. Crucially, since $\phi(\bgamma', \btheta')$ comprises of monomials in $\bgamma', \btheta'$, we can view each column of $B_2$ as the set of weights corresponding to a particular monomial in $\vec{\tilde{x}}$. Let $M(\bgamma', \btheta')$ be one such monomial.
Suppose the intra-manifold variance of $\vec{r}$ is larger than $0$, then $B_2$ matrix must have non-zero weight on nodes which correspond to monomials $M(\bgamma',\btheta')$ that depend on $\btheta'$. By Lemmas~\ref{lem:mean-reduces-square-loss} and \ref{lem:var-rep-output-relation}, replacing the terms depending on $\btheta'$ by their expected value over $\btheta'$ reduces $\hat{V}_{mn}(\vec{\hat y})$ and consequently also the square loss term. Since we have the $\vec{1}_k$ vector concatenated to $\tilde{x}$, for each monomial $M(\bgamma', \btheta')$ there is a unique corresponding monomial $M(\bgamma',\vec{1})$ with the same coefficient as $M(\bgamma', \btheta')$. This is because whatever combination of coordinates of $\bgamma'$ and $\btheta'$ and their powers are chosen in $M(\bgamma', \btheta')$ we can choose the same combination of coordinates and powers to get $M(\bgamma', \vec{1}_k)$ where we have replaced $\btheta'$ with $\vec{1}_k$. And note that since we have assumed $\bgamma',\btheta'$ vectors to lie within the unit sphere, $\E_{n}[M(\bgamma',\btheta')] = c' M(\bgamma',\vec{1})$ where $c' \le 1$. This implies that shifting the weights of $B_2$ from terms corresponding to the monomials which depend on $\btheta'$ to those corresponding to monomials which have no $\btheta'$ dependence should strictly decrease the square loss. 
Moreover, this shift ensures that $\|B_2\|_F$ is not increased. Since $B = B_2U_2^\top$, we have that $\|B\|_F = \|B_2\|_F$ and hence $\|B\|_F$ does not increase as well. $A$ has remained unmodified throughout this process and hence we have managed to strictly decrease the loss by decreasing $\hat{V}_{mn}(B\vec{z})$ to $0$ at the same time.

Now we argue that, appending $k$ 1s to our input produces the same output as appending a single scalar $\frac{\sqrt{k(d+1)}}{\sqrt{(d+k)}}$ where the original dimension of input be $d$. Recall that the weight matrix for the ReLU layer is randomly initialized with each row being drawn from $\cN(0,I/D)$. Consider the output being computed at any single node after the ReLU. In the first case, the contribution of the $k$ 1s to the output is $\sum_{i=1}^k c_{i}$ where each $c_{i} \sim \cN(0, 1/(D(d+k)))$. This is equivalent to a single $c \sim \cN(0, k/(D(d+k)))$. So by appending a constant of value $\frac{\sqrt{k(d+1)}}{\sqrt{(d+k)}}$ the same effect will be achieved.
\end{proof}

Next we show that the insights of Lemma~\ref{lem:weight-switch-arg-inf-width} continue to hold approximately for a finite $D$ as long as it is large enough.
\begin{lemma}
\label{lem:weight-switch-arg}
For $D \ge O(\sqrt{nm}\log(mn/\delta) / \eps)$, given as input $\vec{\tilde{x}'} = (\bgamma', \btheta', \frac{\sqrt{k(d+1)}}{\sqrt{(d+k)}})$, and given a $B$ we can transform it to $B'$ with no greater Frobenius norm so that $\hat{V}_{mn}(B'\vec{z}) \le 4\eps$.
\end{lemma}
\begin{proof}
Let $\sigma(C\vec{x}) = U_2\phi(\bgamma', \btheta') + \bdelta$. When $D \ge O(\sqrt{nm}\log(mn/\delta) / \eps)$, we have from Corollary~\ref{cor:final-repre} that $\|\bdelta\|_2 \le 2\sqrt{\eps}$.
This will imply that the intra-manifold variance of $\vec{r}$ when $B_2$ has zero weight on nodes with a $\btheta'$ dependence is at most $\|2\bdelta\|_2^2/4 = (4\sqrt{\eps})^2/4 = 4\eps$. The rest of the argument proceeds similar to before. Now at each output node of the ReLU layer, we compute a monomial $M(\bgamma', \btheta') + \delta_i$ where $\delta_i$ is the $i^{th}$ coordinate of norm-bounded noise. If $\hat{V}_{mn}(B\sigma(C.)) > 4\eps$, then shifting the weights of $B$ in a similar manner as we did in Lemma~\ref{lem:weight-switch-arg-inf-width} will yield a $B'$ such that $\hat{V}_{mn}(B\sigma(C.)) \le 4\eps$ while maintaining $\|B\|_F = \|B'\|_F$. This ultimately leads to a decrease in the overall objective value.
\end{proof}


Lemma~\ref{lem:weight-switch-arg} gives the following.
\begin{lemma}
\label{lem:linear-var-zero}
At any minima of the objective function~\eqref{eq:general-objective-app} over points taken from the distribution of $\bgamma,\btheta$, $\hat{V}_{mn}(\vec{r}) \le 4\eps$.
\end{lemma}
\begin{proof}
Assume that at some minima $\hat{V}_{mn}(\vec{r}) > 4\eps$. Then by Lemma~\ref{lem:var-rep-output-relation}$, \hat{V}_{mn}(\vec{\hat y}) \ge (\lambda_2/\lambda_1)4\eps^2 = 4m\eps^2/s^{O(\log(1/\eps))}$. Now by replacing $B$ by $B'$ as described in Lemma~\ref{lem:weight-switch-arg}, $\hat{V}_{mn}(\vec{r})$ is pushed to a value smaller than $4\eps$ which implies that the output $\vec{\hat y}$ will be replaced by its approximate mean $\vec{\hat y'}$ per manifold which by Lemma~\ref{lem:mean-reduces-square-loss} reduces the weighted square loss. $A$ remains unchanged and $B$'s Frobenius norm has not increased. So the value of the minimization objective overall has reduced which is a contradiction to Lemma~\ref{lem:frobenius-nuclear-reg-eqv} which states that all minima are global in our setting.
\end{proof}
Lemma~\ref{lem:linear-var-zero} implies that property~\eqref{eq:hash-near} holds on the train examples from the train manifolds.


\subsubsection{Property \texorpdfstring{\eqref{eq:hash-far}}{(B)} of GSH for Linear Manifolds}
Next we prove a bunch of Lemmas for showing property \eqref{eq:hash-far} of GSH for linear manifolds. In this section, without loss of generality we will assume that $m$ is even. If it is not, we drop the samples from the $m^{th}$ manifold and set the new value of $m$ to be $m-1$.
We will use the fact that our train loss is small. The following Lemmas will argue that when the train loss is small, the average of the inter-manifold representation distance over all pairs of our train manifolds is small.
\begin{lemma}
\label{lem:small-loss-rep-far-1}
Let $\vec{a} \in \mathbb{R}^T$ such that $\|\vec{a}\|_2 \le \delta$. Let $\vec{b_1}, \vec{b_2} \in \mathbb{R}^T$ be such that
$$\frac{1}{2}\left(\vec{a}^{\top}\vec{b_1}\right)^2 + \frac{1}{2}\left(1 - \vec{a}^{\top}\vec{b_2} \right)^2 \le \epsilon.$$
Then $$\|\vec{b_1} - \vec{b_2}\|_2^2 \ge \frac{1-4\sqrt{\epsilon}}{\delta^2}.$$
\end{lemma}
\begin{proof}
Let $(\vec{a}^{\top}\vec{b_1})^2 = 2\epsilon_1$ and let $\epsilon_2 = \epsilon - \epsilon_1$. Then,
\begin{align}
    &\left\lvert \vec{a}^{\top}\vec{b_1} \right\rvert \le \sqrt{2\epsilon_1} \; \; \text{and} \; \; \left\lvert 1 - \vec{a}^{\top}\vec{b_2} \right\rvert \le \sqrt{2\epsilon_2} \\
    \implies & \lvert \vec{a}^{\top}(\vec{b_2} - \vec{b_1})\rvert \ge 1 - \sqrt{2\epsilon_1} - \sqrt{2\epsilon_2} \ge 1 - 2\sqrt{\epsilon} \\
    \implies & \|\vec{a}\|_2 \|\vec{b_2} - \vec{b_1}\|_2 \ge 1-\sqrt{2\epsilon} \implies \|\vec{b_2} - \vec{b_1}\|_2 \ge \frac{1-2\sqrt{\epsilon}}{\delta} \\
    \implies &\|\vec{b_2} - \vec{b_1}\|_2^2 \ge \frac{1 + 4\epsilon -4\sqrt{\epsilon}}{\delta^2} \ge \frac{1-4\sqrt{\epsilon}}{\delta^2}.
\end{align}
where we used that $\sqrt{2(a + b)} \ge \sqrt{a} + \sqrt{b}$.
\end{proof}

\begin{lemma}
\label{lem:small-loss-rep-far-2}
For $l \in [m]$, let 
\begin{align}
&H(l) = \frac{1}{n(m-1)}\sum_{j \ne l} \sum_{i=1}^n\| B\sigma(C\vec{x_{il}}) - B\sigma(C\vec{x_{ij}})\|_2^2, \\
&\epsilon_l = \frac{1}{n(m-1)}\sum_{i=1}^n\left[ \sum_{j \ne l} \epsilon_{ijl} \right],\\
&\text{where } \epsilon_{ijl} = \frac{1}{2}\left(A_lB\sigma(C\vec{x_{ij}})\right)^2 + \frac{1}{2}\left(1 - A_lB\sigma(C\vec{x_{il}} \right)^2.
\end{align}
We have,
$$H(l) \ge \frac{1-4\sqrt{\epsilon_l}}{\|A_l\|_2^2}.$$
\end{lemma}
\begin{proof}
From Lemma~\ref{lem:small-loss-rep-far-1} we have
\begin{align}
    &\|B\sigma(C\vec{x_{il}}) - B\sigma(C\vec{x_{ij}}) \|_2^2 \ge \frac{1 - 4\sqrt{\epsilon_{ijl}}}{\|A_l\|_2^2} \\
    \implies &\frac{1}{n(m-1)}\sum_{i=1}^n\sum_{j \ne l} \|B\sigma(C\vec{x_{il}}) - B\sigma(C\vec{x_{ij}}) \|_2^2  \ge \frac{1}{n(m-1)}\frac{n(m-1) - 4n(m-1)\sqrt{\epsilon_l}}{\|A_l\|_2^2} \\
    &~~= \frac{1 - 4\sqrt{\epsilon_l}}{\|A_l\|_2^2},
\end{align}
where we have used that $\sum_{k=1}^c \sqrt{a_k} \le \sqrt{c}\cdot \sqrt{\sum_{k=1}^n a_k}$.
\end{proof}

\begin{lemma}[Small Weighted Square Loss Implies Distant Representations]
\label{lem:small-loss-rep-far-3}
For $l \in [m]$, let 
\begin{align}
&H(l) = \frac{1}{n(m-1)}\sum_{j \ne l} \sum_{i=1}^n\| B\sigma(C\vec{x_{il}}) - B\sigma(C\vec{x_{ij}})\|_2^2 \\
&\epsilon = \cL_{A, B}(Y, \hat{Y}),
\end{align}
Then, 
$$\frac{1}{m}\sum_{l=1}^m H(l) \ge \frac{m(1 - O(\sqrt{\epsilon}))}{O(\|A\|_F^2)}.$$
\end{lemma}
\begin{proof}
First note that $\epsilon = \sum_{l=1}^m \epsilon_l / m$ and let $\delta = \sum_{l=1}^m \|A_l\|_2^2 / m$. From Markov's inequality we have that there exists $S_1 \subseteq [m]$ such that $|S_1| = \lceil 9m/10 \rceil$ and $\forall l \in S_1, \epsilon_l \le 10\epsilon$. Similarly there exists $S_2 \subseteq [m]$, $|S_2| = \lceil 9m/10 \rceil$ such that $\forall l \in S_2, \|A_l\|_2^2 \le 10\delta$. Note that $|S_1 \cap S_2| \ge \lceil 8m/10 \rceil$.
We have,
\begin{align}
    \frac{1}{m} \sum_{l=1}^m H(l) \ge \frac{1}{m} \sum_{l \in |S_1 \cap S_2|} H(l) \ge \frac{8}{10} \cdot \frac{1-4\sqrt{10\epsilon}}{10\delta} = \frac{m(1-O(\sqrt{\epsilon}))}{O(\|A\|_F^2)},
\end{align}
where we used Lemma~\ref{lem:small-loss-rep-far-2}.
\end{proof}

Now since at any local minimum $\cL_{\hat{A},\hat{B}}(Y, \hat{Y}) \le \eps$ and $\|\hat{A}\|_F^2 \le 3m$, Lemma~\ref{lem:small-loss-rep-far-3} gives property~\eqref{eq:hash-far} on the train data.

\subsection{GSH on Train Data for Non-Linear Manifolds With Intra-Class Variance Regularization}
\label{sec:var-reg-app}
For non-linear manifolds, the argument we had in Section~\ref{sec:linear-app} does not go through as is. This is because we no longer have as nice a mapping from monomials of $\vec{x}$ to associated monomials of similar degree in $\bgamma, \btheta$ as we had before. In particular, our argument for Lemma~\ref{lem:weight-switch-arg} breaks down. Instead, we show a more general result over non-linear manifolds in this section via the means of a different form of regularizer than before.

Recall that we add to our objective the following variance regularization term
$$V_{\text{reg}}(B\sigma(C\cdot)) = \frac{n}{n-1}\hat{V}_{mn}(B\sigma(C\cdot))$$
The final objective we minimize is,
\begin{align}
\label{eq:var-reg-obj-app}
    \hspace{-0.75em}\cL_{A,B}(Y, \hat{Y}) + \lambda_1\|A\|_F^2 + \lambda_2 \left(\| B\|_F^2 + V_{\text{reg}}(B\sigma(C))\right)
\end{align}

We prove Theorem~\ref{thm:var-reg-main} via a series of Lemmas which follow. We begin by showing that for this new minimization objective, the ground truth matrices $A^*, B^*$ from Lemma~\ref{lem:good-ground-truth-app} still give good properties.
\begin{lemma}[Good Ground Truth for Non-Linear Manifolds]
\label{lem:var-reg-ground-truth}
Given the ground truth weight matrices $A^*, B^*$ of Lemma~\ref{lem:good-ground-truth-app}, we have that
$$V_{\text{reg}}(B^*\sigma(C.)) \le 2\epsilon_2.$$
\end{lemma}
\begin{proof}
This follows immediately from Lemma~\ref{lem:good-ground-truth-app}, point 4 since $V_{\text{reg}}(B^*\sigma(C.)) = \frac{n}{n-1}\hat{V}_{mn}(B^*\sigma(C.))$.
\end{proof}

Next, we show that, remarkably, even for our new objective the property that all local minima are global still holds.
\begin{lemma}
\label{lem:var-reg-global-min}
Consider \eqref{eq:var-reg-obj-app}. Every local minimum is a global minimum.
\end{lemma}
\begin{proof}
We will map our minimization to an appropriate form whereafter we can apply Lemma~\ref{lem:frobenius-nuclear-reg-eqv} to argue that it is equivalent to a weighted nuclear norm minimization in $AB$ which is convex in $AB$. 
Let $Z'' = (I, Z')$ where we stacked the identity's columns at the front of $Z'$. Note that $Z''$ is full row rank.
Consider the SVD of $Z'' = USV^{\top}$ and consider the truncated form $Z''_{trunc} = US_{trunc}$. $Z''_{trunc}$ is a square matrix which is invertible and
\begin{align}
    \|B\|_F^2 + \|BZ'\|_F^2 = \|BZ''\|_F^2 = \|BZ''_{trunc}\|_F^2.
\end{align}
By letting $B'' = BZ''_{trunc}$, we can re-write \eqref{eq:var-reg-obj-app} in the following form now.
\begin{align}
    \label{eq:var-reg-opt-rewritten}
    \min_{A, B''} \sum_{l=1}^m\frac{1}{m}\|W_l \odot (Y_l - AB''M_l) \|_F^2 + \lambda_1 \|A\|_F^2 +\lambda_2 \|B''\|_F^2,
\end{align}
where $M_l = S_{trunc}^{-1}U^{\top}Z_l$ which is a minimization of a convex function of $AB''$ with Frobenius norm regularization which by the application of Lemma~\ref{lem:frobenius-nuclear-reg-eqv} gives us the desired mapping to a convex function minimization with nuclear norm regularization which is a convex objective and gradient descent on this objective will achieve the global minimum which we can argue is also what is achieved by gradient descent on our objective.
\end{proof}

\begin{lemma}
\label{lem:good-local-minima-var-reg}
We get that at any local minimum $\hat A, \hat B$ of \eqref{eq:var-reg-obj-app} for $\lambda_1 = \eps_1/m$, $\lambda_2 = \eps_1/s^{O(\log(1/\eps_1))}$, we have
\begin{align}
    &\cL_{\hat{A},\hat{B}}(Y, \hat Y) \le 4\epsilon_1, \\
    &\|\hat{A}\|_F \le 4m, \|\hat{B}\|_F \le 4s^{O(\log(1/\eps_1))}, \\
    & V_{\text{reg}}(B^*\sigma(C.)) \le 8\epsilon_2.
\end{align}
\end{lemma}
\begin{proof}
Following a line of argument similar to the proof of Lemma~\ref{lem:global-opt-reg-params} we get that at global minimum the bounds stated in the Lemma are satisfies (by using Lemma~\ref{lem:mutli-obj-opt}). Lemma~\ref{lem:var-reg-global-min} gives us that all local minima are global and hence the statement of the current lemma follows.
\end{proof}

Next we show that that each part of GSH holds on the train data.
\subsubsection{Property \texorpdfstring{\eqref{eq:hash-near}}{(A)} of GSH}
Since at any local minimum, we have that  $V_{\text{reg}}(\hat{B}\sigma(C.)) = \frac{n}{n-1}\hat{V}_{mn}(\hat{B}\sigma(C.))$ and $V_{\text{reg}}(\hat{B}\sigma(C.)) \le 8\epsilon_2$ from Lemma~\ref{lem:var-reg-ground-truth} we get that $\hat{V}_{mn}(\hat{B}\sigma(C.)) \le 8\eps_2$ as well immediately giving property~\eqref{eq:hash-near} of the GSH.

\subsubsection{Property~\texorpdfstring{\eqref{eq:hash-far}}{(B)} of GSH}
Recall the argument for showing property~\ref{eq:hash-far} for linear manifolds from Section~\ref{sec:linear-app}. Lemmas~\ref{lem:small-loss-rep-far-1}-\ref{lem:small-loss-rep-far-3} imply that the average inter-manifold representation distance is larger than a constant as long as $\cL_{\hat{A},\hat{B}}(Y, \hat{Y})$ is small and $\|\hat{A}\|_F$ is small. These two properties still hold in the current setting for non-linear manifolds. Hence we immediately get that property~\eqref{eq:hash-far} holds on the train data for non-linear manifolds as well.

\section{Generalization Bounds: Proofs for Section~\ref{sec:generalization-body}}
\label{sec:gen-apx}
In this section, we present population variants for bounds on empirical quantities that we saw in Section~\ref{sec:opt-app}. Since the architectures for linear and non-linear manifolds are the same, the results in this section will apply to both. We rely on the technique of uniform convergence bounds which are shown via Rademacher complexity.

A recurring general function class whose Rademacher complexity we will repeatedly use is the following
$$\cF = \left\{ f(\vec{x}) = \|P\vec{x} + \vec{b}\|_2^2 \middle| \|\vec{x}\|_2 \le \alpha, \|P\|_F \le \beta, \|\vec{b}\|_2 \le b\right\}.$$
We present a Rademacher complexity bound for this class.
\begin{lemma}
\label{lem:common-rademacher}
Given the above function class $\cF$, we have
$$\cR_n(\cF) \le \frac{2b\beta\alpha + \beta^2\alpha^2}{\sqrt{n}}.$$
\end{lemma}
\begin{proof}
First, $\|P\vec{x} + \vec{b}\|_2^2 = \|P\vec{x}\|_2^2 + \|\vec{b}\|_2^2 + 2\vec{b}^\top P\vec{x}$. Since $\sup_{\theta}(f_{\theta}(\vec{x}) + g_{\theta}(\vec{x})) \le \sup_{\theta}f_{\theta}(\vec{x}) + \sup_{\theta}g_{\theta}(\vec{x})$ we have
\begin{align}
    \cR_n(\cF) &= \frac{1}{n}\E_{\vxi}\left[\sup_{f \in \cF} \sum_{i=1}^n \xi_i f(\vec{x_i})\right] \\
    &\le \underbrace{\frac{1}{n}\E_{\vxi}\left[\sup_{P: \|P\|_F \le \beta} \sum_{i=1}^n \xi_i \|P\vec{x_i}\|_2^2 \right]}_{(A)} + \underbrace{\frac{1}{n}\E_{\vxi}\left[\sup_{\vec{b}: \|\vec{b}\|_2 \le b} \sum_{i=1}^n \xi_i \|\vec{b}\|_2^2 \right]}_{(B)} + \underbrace{\frac{1}{n}\E_{\vxi}\left[\sup_{\substack{P,\vec{b}: \|P\|_F \le \beta,\\ \|\vec{b}\|_2 \le b}} \sum_{i=1}^n \xi_i 2\vec{b}^\top P\vec{x_i} \right]}_{(C)}.
\end{align}
We bound each term separately. $(B)$ is clearly $0$. $(C)$ is the Rademacher complexity of a class of linear functions of $\vec{x}$ which is known to be bounded by $\|\vec{b}^\top P\|_2\alpha/\sqrt{n}$ (Lemma 26.10 of \cite{shalev2014understanding}) which can be bounded by $2b\beta\alpha/\sqrt{n}$. It remains to bound $(A)$. Here we use that,
\begin{align}
\|P\vec{x}\|_2^2 = P^{\top}P \odot \vec{x}\vec{x}^{\top}. \label{eq:rad-1}
\end{align}
Moreover, we have $\|P^{\top}P\|_* = \|P\|_F^2$.
Now, 
\begin{align}
    (A) &= \frac{1}{n}\E_{\vxi}\left[\sup_{\|P\|_F \le \beta} \sum_{i=1}^n \xi_i \|P\vec{x_i}\|_2^2 \right] \notag\\
    &= \frac{1}{n}\E_{\vxi}\left[\sup_{\|P\|_F \le \beta} \sum_{i=1}^n \xi_i \left( P^{\top}P \odot \vec{x_i}\vec{x_i}^{\top}\right)  \right] \label{eq:rad-2}\\
    &\le \frac{1}{n}\E_{\vxi}\left[\sup_{\|W\|_* \le \beta^2} \sum_{i=1}^n \xi_i  \left( W \odot \vec{x_i}\vec{x_i}^{\top} \right)  \right] \label{eq:rad-3}\\
    &= \frac{1}{n}\E_{\vxi}\left[\sup_{\|W\|_* \le \beta^2}  \left( W \odot \sum_{i=1}^n \xi_i \vec{x_i}\vec{x_i}^{\top} \right)  \right] \notag\\
    &= \frac{\beta^2}{n}\E_{\vxi}\left[ \left\|\sum_{i=1}^n \xi_i \vec{x_i}\vec{x_i}^{\top}\right\|_2  \right] \label{eq:rad-5}\\
    &\le \frac{\beta^2}{n}\E_{\vxi}\left[ \left\|\sum_{i=1}^n \xi_i \vec{x_i}\vec{x_i}^{\top}\right\|_F  \right] \label{eq:rad-6}\\
    & \le \frac{\beta^2}{n}\sqrt{\E_{\vxi}\left[ \left\|\sum_{i=1}^n \xi_i \vec{x_i}\vec{x_i}^{\top}\right\|_F^2  \right]} \label{eq:rad-7}\\
    &= \frac{\beta^2}{n}\sqrt{\sum_{i=1}^n \|\vec{x_i}\vec{x_i}^{\top}\|_F^2} \label{eq:rad-8}\\
    &= \frac{\beta^2}{n}\sqrt{\sum_{i=1}^n \|\vec{x_i}\|_2^4} \le \frac{\beta^2\alpha^2}{\sqrt{n}}, \notag
\end{align}
where \eqref{eq:rad-2} uses \eqref{eq:rad-1} and \eqref{eq:rad-3} is obtained by replacing $P^\top P$ with a matrix $W$ and using that $\|P\|_F \le \beta \implies \|P^\top P\|_* = \|P\|_F^2 \le \beta^2$. \eqref{eq:rad-5} uses Claim~\ref{clm:trace-dual-operator}, \eqref{eq:rad-6} follows because for any matrix $A, \|A\|_2 \le \|A\|_F$, and \eqref{eq:rad-7} follows from Jensen's inequality. Finally \eqref{eq:rad-8} follows by expanding the Frobenius norm and noting that the $\xi_i$ are independent and $\E[\xi_i \xi_j] = 0$ for $i \ne j$.
Combining the bounds for $(A),(B)$ and $(C)$ we get the statement of the lemma.
\end{proof}

\subsection{Small Weighted Square Loss on Test Samples from Train Manifolds}
The first generalization bound is to show that on the $m$ train manifolds, our learnt network achieves a small test error as measured by the weighted square loss. That is, the network has actually learnt to classify inputs from the $m$ train manifolds.
A simple uniform convergence argument coupled with Lemma~\ref{lem:small-train-weighted-square-loss} gives us that the expected weighted square loss over unseen samples from our train manifolds is small as well for large enough $n$. 
\begin{lemma}
\label{lem:small-test-loss-gen-n}
At a local minimum $\hat{A},\hat{B}$ we have, with probability $\ge 1-\delta$,
$$\frac{1}{m}\sum_{l=1}^m\E_{\vec{x}_l \sim \cD(M_l)}\left[ \|\vec{w_l} \odot (\vec{y_l} - \vec{\hat{y}_l})\|_2^2  \right] \le 2\epsilon,$$
for $n \ge \Theta\left( \frac{s^{O(\log(1/\eps))}\log(m/\delta)}{\eps^2}\right)$.
\end{lemma}
\begin{proof}
Let $\beta = s^{O(\log(1/\eps)}$ be the bound we have on $\|\hat{B}\|_F$. Fix a train manifold with index $l$. Let $\hat{A}_l$ denote the $l^{th}$ row of $\hat{A}$ and let $\|\hat{A}_l\|_2 \le a_l$. We know that $\sum_l a_l^2 \le O(m)$. Let $\E_n[\|\vec{w_l} \odot (\vec{y_l} - \vec{\hat{y}_l})\|_2^2 \le \eps_l$. From Lemma 26.5 of \cite{shalev2014understanding} we have that with probability $\ge 1-\delta/m$,
\begin{align}
    \E_{\vec{x}_l \sim \cD(M_l)}\left[ \|\vec{w_l} \odot (\vec{y_l} - \vec{\hat{y}_l})\|_2^2  \right] \le \E_n\left[ \|\vec{w_l} \odot (\vec{y_l} - \vec{\hat{y}_l})\|_2^2  \right] + 2\E[\cR_n(\cF)] + c\sqrt{\frac{2\log(m/\delta)}{n}},
\end{align}
In our case, $c = O(a_l^2\beta^2\|C\|_2^2)$ as $\|\vec{w_l} \odot (\vec{y_l} - \vec{\hat{y}_l})\|_2^2 \le O(a_l^2\beta^2\|C\|_2^2)$ and $\cR_n(\cF_l)$ is the Rademacher complexity of the function class
$$\cF_l = \left\{ f(\vec{x_l}) = \left\|\vec{w_l} \odot (\vec{y_l} - AB\sigma(C\vec{x_l})) \right\|_2^2 \; \; \middle| \; \; \|A\|_F^2 \le O(m), \|B\|_F \le \beta\right\}$$
Now $\|\vec{w_l} \odot (\vec{y_l} - AB\sigma(C\vec{x_l})) \|_2^2$ is of the form $\|P\vec{z} + \vec{b}\|_2^2$ for $P = \vec{w_l}\odot AB$, $\vec{b} = -\vec{w_l}\odot \vec{y_l}$ and $\vec{z} = \sigma(C\vec{x_l})$. Since, $\|P\|_F^2 \le O(a_l^2\beta^2), \|\vec{b}\|_2 = 1/2$ and $\|\sigma(C\vec{x_l})\|_2 \le \|C\|_2$, from Lemma~\ref{lem:common-rademacher} we have that 
\begin{align}
    \cR_n(\cF_l) &\le O\left(\frac{a_l^2\beta^2\|C\|_2^2}{\sqrt{n}}\right). \notag
\end{align}
Therefore, we have that, with probability $\ge 1-\delta/m$,
\begin{align}
    \E_{\vec{x}_l \sim \cD(M_l)}\left[ \|\vec{w_l} \odot (\vec{y_l} - \vec{\hat{y}_l})\|_2^2  \right] &\le \E_n\left[ \|\vec{w_l} \odot (\vec{y_l} - \vec{\hat{y}_l})\|_2^2  \right] + O\left(\frac{a_l^2\beta^2\|C\|_2^2}{\sqrt{n}}\right) + O(a_l^2\beta^2\|C\|_2^2)\sqrt{\frac{2\log(m/\delta)}{n}}.
\end{align}
Averaging over all $m$ train manifolds and taking a union bound, we have with probability $\ge 1-\delta$,
\begin{align}
    \frac{1}{m}\sum_{l=1}^m\E_{\vec{x}_l \sim \cD(M_l)}\left[ \|\vec{w_l} \odot (\vec{y_l} - \vec{\hat{y}_l})\|_2^2  \right] \le \eps + O\left(\frac{\beta^2\|C\|_2^2}{\sqrt{n}}\right) + O(\beta^2\|C\|_2^2)\sqrt{\frac{2\log(m/\delta)}{n}}. \label{eq:rad-20}
\end{align}
Note that we have used that $\sum_{l=1}^ma_l^2 = O(m)$ here. \eqref{eq:rad-20} will imply the statement of the lemma for 
$$n = \Theta\left( \frac{\beta^4\|C\|_2^4\log(m/\delta)}{\eps^2}\right) = \Theta\left( \frac{s^{O(\log(1/\eps))}\log(m/\delta)}{\eps^2}\right),$$
with probability $\ge 1-\delta/2$. Here we used that for the $D$ we have chosen $\|C\|_2$ is bounded by a constant with very high probability (Lemma~\ref{lem:random-matrix-spectral-norm}).
\end{proof}

\subsection{Generalization for Property~\texorpdfstring{\eqref{eq:hash-near}}{(6)}}
We first show that the variance regularization term $V_{\text{reg}}$ is an unbiased estimator for the intra-manifold variance of our representation.
\begin{lemma}[Unbiased Variance Estimation]
\label{lem:unb-var-est}
Consider the variance regularization term~\eqref{eq:var-reg-term}. For the set of train manifolds $M_1, \ldots, M_m$, we have,
$$\E[\eqref{eq:var-reg-term}] =  \frac{1}{m}\sum_{l=1}^m V_{M_l}(B\sigma(C.)).$$
\end{lemma}
\begin{proof}
Let us focus on a single manifold $l$. Let $\E[\|B\bz_l\|_2^2] = \mu_l$ and let $\|B\E[\bz_l]\|_2^2 = \kappa_l$.
\begin{align}
    \frac{1}{(n-1)}\E\left[\sum_{i=1}^n\|B\hat{\bz}_{il} \|_2^2 \right] &= \frac{1}{n-1}\sum_{i=1}^n \left(\overbrace{\E\left[\|B\bz_{il} \|_2^2\right]}^{(A)}  + \overbrace{\E\left[\|B\E_n[\bz_l]\|_2^2 \right]}^{(B)} - \overbrace{2\E\left[\bz_{il}^{\top}B^{\top}B\E_n[\bz_l] \right]}^{(C)} \right).
\end{align}
Now $(A) = \mu_l$. We obtain the expressions for $(B)$ and $(C)$.
\begin{align}
    (B) &= \frac{1}{n^2}\E \left[ \left\| \sum_{j=1}^n B\bz_{jl} \right\|_2^2 \right] \\
    &= \frac{1}{n^2}\sum_{j=1}^n\E[\|B\bz_{jl}\|_2^2] + \frac{1}{n^2}\sum_{j_1 \ne j_2} \E\left[ \bz_{j_1l}^{\top}B^{\top}B\bz_{j_2l}\right] \\
    &= \frac{\mu_l}{n} + \frac{n-1}{n}\E[\bz_l]^{\top}B^{\top}B\E[\bz_l] \\
    &= \frac{\mu_l}{n} + \frac{n-1}{n}\|B\E[\bz_l] \|_2^2 = \frac{\mu_l}{n} + \frac{\kappa_l(n-1)}{n}.
\end{align}
Finally,
\begin{align}
    (C) &= \frac{2}{n}\E\left[ \sum_{j=1}^n \bz_{il}^{\top}B^{\top}B\bz_{jl} \right] \\
    &= \frac{2}{n}\E[\bz_{il}^{\top}B^{\top}B\bz_{il}] + \frac{2}{n}\sum_{j, j\ne i} \E[\bz_{il}^{\top}B^{\top}B\bz_{jl}] \\
    &= \frac{2\mu_l}{n} + \frac{2(n-1)\kappa_l}{n}.
\end{align}
Adding all together we get,
\begin{align}
    \frac{1}{(n-1)}\E\left[\sum_{i=1}^n\|B\hat{\bz}_{il} \|_2^2 \right] = \mu_l - \kappa_l.
\end{align}
It is easy to see via a similar calculation that,
\begin{align}
    &V_{M_l}(B\sigma(C.)) = \mu_l - \kappa_l \\
    \implies & \frac{1}{(n-1)m}\E\left[\sum_{l=1}^m\sum_{i=1}^n\|B\hat{\bz}_{il} \|_2^2 \right] = \frac{1}{m}\sum_{l=1}^m V_{M_l}(B\sigma(C.)).
\end{align}
\end{proof}

We next show that having a small variance regularization term over the train data implies that Property~\eqref{eq:hash-near} of GSH holds with high probability.
\begin{lemma}[Generalization of Property~\eqref{eq:hash-near} to Unseen Points from Train Manifolds]
\label{lem:linear-gen-n}
Recall that at local minimum, we have found a $\hat{B}$ so that for $r(\vec{x}) = \hat{B}\sigma(C\vec{x})$,
$$\frac{1}{m}\sum_{l=1}^m\frac{1}{n-1}\sum_{i=1}^n \|r(\vec{x_{il}}) - \E_n[r(\vec{x_{l}})]\|_2^2 = 0.$$
Then we have,
\begin{align}
    \Pr\left[\sum_{l=1}^m\frac{1}{m}V_{M_l}(r(.)) \le 2\epsilon \right] \ge 1-\delta,
\end{align}
where the probability is taken over the sampling of the $n$ input examples from each of the $m$ train manifolds.
\end{lemma}
\begin{proof}
From Lemma~\ref{lem:unb-var-est} we have that
$$\E\left[\frac{1}{n-1}\sum_{i=1}^n\|r(\vec{x_{il}}) - \E_n[r(\vec{x_{l}})]\|_2^2\right] = V_{M_l}(r(.)).$$

Given a vector $\vec{h} \in \mathbb{R}^D$, consider the family of functions defined as follows.
\begin{align}
    \mathcal{F}_{\vec{h}} = \left\{ f_{B, \vec{h}}:\mathbb{R}^d \to \mathbb{R}  \;\big|\ f_{B, \vec{h}}(x) = \|B\left(\sigma(C\vec{x}) - \vec{h}\right)\|_2^2, \|\vec{x}\|_2 \le \alpha, \|B\|_F \le \beta, \|\vec{h}\|_2 \le \|C\|_2\alpha \right\}.
\end{align}
By Theorem 26.5 from \cite{shalev2014understanding}, we have that for each manifold $M_l$, for every $f_{B,\vec{h}} \in \mathcal{F}_{\vec{h}}$,
\begin{align}
    \E_{M_l}\left[ f_{B,\vec{h}}(\vec{x})\right] \le \tilde{\E_{M_l}}\left[ f_{B,\vec{h}}(\vec{x}) \right] + 2\mathcal{R}_n(\mathcal{F}_{\vec{h}}) + c\sqrt{\frac{2\log\left(m/\delta \right)}{n}},
\end{align}
with probability $1-\delta/m$. Here $|f_{B,\vec{h}}| \le c$. For the function family we have considered taking $c = 4\beta^2\|C\|_2^2\alpha^2$ suffices. Note that for $\vec{h}(l) = \tilde{\E}_{M_l}[\sigma(C\vec{x})]$, $\E_{M_l}\left[ f_{\tilde{B},\vec{h}(l)}(x) \right] = V_{M_l}(\tilde{B}\sigma(C.))]$ and $\sum_{l=1}^m \tilde{\E}_{M_l}\left[ f_{B,\vec{h}(l)}(\vec{x}) \right] = m\eqref{eq:var-reg-term} \le m\epsilon$.
An upper bound on $\mathcal{R}_n(\mathcal{F}_{\vec{h}(l)})$ will lead us to an upper bound on $\sum_{l=1}^m V_{M_l}(\tilde{B}\sigma(C.))]$ as we shall see later. We proceed to bound $\mathcal{R}_n(\mathcal{F}_{\vec{h}(l)})$.
For $\vec{x}\sim \cD(M_l)$, let  $\sigma(C\vec{x}) - \vec{h}(l) = \vec{z}$.
From Lemma~\ref{lem:common-rademacher}, we have
\begin{align}
    \mathcal{R}_n(\mathcal{F}_{\vec{h}(l)}) \le \frac{3\beta^2\|C\|_2^2\alpha^2}{\sqrt{n}}. \notag
\end{align}

Therefore, we get that
\begin{align}
    \sum_{l=1}^m V_{M_l}(\hat{B}\sigma(C.)) &\le \sum_{l=1}^m \hat{\E_{M_l}}\left[ f_{B,h(l)}(x) \right] + 6m\beta^2\|C\|_2^2\alpha^2\sqrt{\frac{1}{n}} + 4m\beta^2\|C\|_2^2\alpha^2\sqrt{\frac{\log\left(m/\delta \right)}{n}} \notag\\
    &\le m\epsilon + 6m\beta^2\|C\|_2^2\alpha^2\sqrt{\frac{1}{n}} + 4m\beta^2\|C\|_2^2\alpha^2\sqrt{\frac{\log\left(m/\delta \right)}{n}} \notag\\
    &\le m\epsilon + 10m\beta^2\|C\|_2^2\alpha^2\sqrt{\frac{\log\left(m/\delta \right)}{n}}.\label{eq:var-reg-gen-eq20}
\end{align}
with probability $1-\delta$. Note that we applied a union bound over the statement for each manifold bounding the probability of large deviation in any one of the manifolds by $m \cdot \delta / m = \delta$.
By choosing $n \ge \Theta\left(\frac{\beta^4\alpha^4\|C\|_2^4\log(m/\delta)}{\eps^2}\right)$ in ~\eqref{eq:var-reg-gen-eq20}, we get that with probability $1-\delta$,
\begin{align}
    \frac{1}{m}\sum_{l=1}^m V_{M_l}(\hat{B}\sigma(C.)) \le 2\epsilon. \notag
\end{align}
Since $\|\vec{x_l}\|_2 \le 1$, we can choose $\alpha = 1$. Moreover from Lemma~\ref{lem:random-matrix-spectral-norm} we can assume $\|C\|_2$ is bounded by a constant with very high probability. In addition we have that at local minima $\|\hat{B}\|_F \le s^{O(\log(1/\eps))}$. Hence, we get the statement of the lemma for
$$n \ge \Theta\left(\frac{s^{O(\log(1/\eps))}\log(m/\delta)}{\eps^2} \right).$$
\end{proof}

\begin{lemma}[Generalization of the Property \eqref{eq:hash-near} to new Manifolds]
\label{lem:gen-transfer-app}
Given $\hat{A}, \hat{B}$ which are at a local minima of either objective \eqref{eq:general-objective-app} or \eqref{eq:var-reg-obj-app}, for a fresh sample $M_{m+1} \sim \mathcal{M}$, we have
\begin{align}
    \E_{\vec{x} \sim \cD(M_{m+1})}\left[\left\|\hat{B}\vec{z'} \right\|_2^2 \right] \le O(\epsilon),
\end{align}
with probability $\ge 9/10$ over the draw of $M_{m+1}$ from $\mathcal{M}$ when 
$$m \ge \Theta\left(\frac{s^{O(\log(1/\eps))}\log(2/\delta)}{\eps^2}\right).$$
Here $\vec{z'} = \sigma(C\vec{x})-\E_{\vec{x} \sim \cD(M_{m+1})}[\sigma(C\vec{x})]$.
\end{lemma}
\begin{proof}
Consider the following function class which maps a manifold to a non-negative real value:
$$\mathcal{F} = \left\{f_B : M \to \mathbb{R}, f_B(M) = \E_{\vec{x} \sim \cD(M)}\left[\|B\vec{z'}\|_2^2\right] \; \; \vert \; \; \| B \|_F \le \beta, \|\vec{x}\|_2 \le \alpha  \right\}.$$
By the property of Rademacher complexity (Theorem 26.2 from \cite{shalev2014understanding}), we have with probability $\ge 1-\delta$, for any $f \in \mathcal{F}$,
\begin{align}
    \E_{M_{m+1} \sim \mathcal{M}}\left[f(M_{m+1})\right] \le \E_m[f(M_l)] + 2\E[\mathcal{R}_m(\mathcal{F})] + c\sqrt{\frac{\log(2/\delta)}{m}},
\end{align}
with probability $1-\delta$. Here $c$ is such that $|f(M)| \le c$ for all $M$. Choosing $c=4\beta^2\|C\|_2^2 \alpha^2$ suffices.
We need to bound $\mathcal{R}_m(\mathcal{F})$. We could appeal to Lemma~\ref{lem:common-rademacher} again but since now the function class $\cF$ has functions of manifolds instead of vectors $\vec{x}$ we present the full argument here for clarity. The argument follows in a very similar vein to that of Lemma~\ref{lem:common-rademacher}. Now we let $\vxi = (\xi_1, \ldots, \xi_m)$ denote a vector of $m$ i.i.d. Rademacher random variables.
\begin{align}
    \mathcal{R}_m(\mathcal{F}) &= \frac{1}{m}\E_{\vxi}\left[\sup_{f_B \in \cF} \sum_{l=1}^m \xi_l f_B(M_l) \right] \notag\\
    &= \frac{1}{m}\E_{\vxi}\left[\sup_{f_B \in \cF} \sum_{l=1}^m \xi_l\E_{\vec{x} \sim \cD(M_l)} \left[ \|B\vec{z'}\|_2^2 \right] \right] \notag\\
    &= \frac{1}{m}\E_{\vxi}\left[\sup_{\|B\|_F \le \beta} \sum_{l=1}^m \xi_l\E_{\vec{x} \sim \cD(M_l)} \left[ \left\langle B^{\top}B,  \vec{z'}\vec{z'}^{\top} \right\rangle \right] \right]  \notag\\
    &= \frac{1}{m}\E_{\vxi}\left[\sup_{\|B\|_F \le \beta}  \left\langle B^{\top}B,  \sum_{l=1}^m \xi_l\E_{\vec{x} \sim \cD(M_l)} \left[\vec{z'}\vec{z'}^{\top}  \right] \right\rangle\right] \notag\\
    &\le \frac{1}{m}\E_{\vxi}\left[\sup_{\|W\|_* \le \beta^2}  \left\langle W,  \sum_{l=1}^m \xi_l\E_{\vec{x} \sim \cD(M_l)} \left[\vec{z'}\vec{z'}^{\top}  \right] \right\rangle\right] \notag\\
    &= \frac{\beta^2}{m}\E_{\vxi}\left[\left\| \sum_{l=1}^m \xi_l\E_{\vec{x} \sim \cD(M_l)} \left[\vec{z'}\vec{z'}^{\top}  \right] \right\|_2\right] \notag\\
    &\le \frac{\beta^2}{m}\sqrt{\E_{\vxi}\left[\left\| \sum_{l=1}^m \xi_l\E_{\vec{x} \sim \cD(M_l)} \left[\vec{z'}\vec{z'}^{\top}  \right] \right\|_F^2\right]} \notag\\
    &= \frac{\beta^2}{m}\sqrt{\sum_{l=1}^m \left\|\E_{\vec{x} \sim \cD(M_l)} \left[\vec{z'}\vec{z'}^{\top}  \right]\right\|_F^2} \notag\\
    &\le \frac{\beta^2}{m}\sqrt{m\cdot \max\left\|\vec{z'}\right\|_2^4} \notag\\
    &\le \frac{4\beta^2\alpha^2\|C\|_2^2}{\sqrt{m}}. \notag
\end{align}
Therefore,
\begin{align}
    \E_{M_{m+1} \sim \mathcal{M}}\left[\E_{x \sim \cD(M_{m+1})}\left[\|\hat{B}z' \|_2^2 \right]\right] &\le \epsilon + \frac{8\beta^2\alpha^2\|C\|_2^2}{\sqrt{m}} + \frac{4\beta^2\alpha^2\|C\|_2^2\sqrt{\log(2/\delta)}}{\sqrt{m}} \notag\\
    &\le \epsilon + \frac{12\beta^2\alpha^2\|C\|_2^2\sqrt{\log(2/\delta)}}{\sqrt{m}} \notag\\ 
    &\le 2\epsilon, \notag
\end{align}
for 
$$m \ge \frac{144\beta^4\alpha^4\|C\|_2^4\log(2/\delta)}{\epsilon^2} = \Theta\left(\frac{s^{O(\log(1/\eps))}\log(2/\delta)}{\eps^2}\right)$$
Finally, by Markov's inequality, we have that with probability $\ge 9/10$,
\begin{align}
    \E_{x \sim \cD(M_{m+1})}\left[\|\hat{B}\vec{z'} \|_2^2 \right] \le 20\epsilon. \notag
\end{align}
\end{proof}

\subsection{Generalization for Property~\texorpdfstring{\eqref{eq:hash-far}}{(B)}}
We first show that a population variant of Lemma~\ref{lem:small-loss-rep-far-3} over $n$ holds with high probability.
\begin{lemma}
\label{lem:small-loss-rep-far-gen-n}
With probability $\ge 1-\delta$,
$$\frac{1}{m(m-1)} \sum_{l=1}^m \sum_{j \ne l}\E_{\substack{\vec{x_l} \sim \cD(M_l),\\ \vec{x_j}\sim \cD(M_j)}}\| \hat{B}\sigma(C\vec{x_l}) - \hat{B}\sigma(C\vec{x_j})\|_2^2  \ge \frac{m(1 - O(\sqrt{\epsilon}))}{O(\|A\|_F^2)} - f(\log(m)/\delta)/\sqrt{n}.$$
\end{lemma}
\begin{proof}
Recall the definition of $H(l)$ from Lemma~\ref{lem:small-loss-rep-far-3}. Consider the sum $S = \frac{1}{m}\sum_{l=1}^m H(l)$. From this summation, consider the $n$ terms corresponding to a pair of manifolds $l \ne j$. Denote the sum over these $n$ terms by $S_{l,j}$. We can argue generalization for $S_{l,j}$ using uniform convergence theory. 
In particular, we have that with probability $1-\delta/(m(m-1))$,
\begin{align}
     &S_{l,j} - \frac{1}{m(m-1)}\E_{\substack{\vec{x_l} \sim \cD(M_l),\\ \vec{x_j}\sim \cD(M_j)}}\| \hat{B}\sigma(C\vec{x_l}) - \hat{B}\sigma(C\vec{x_j})\|_2^2 \le 2\frac{\E[\cR_n(\cF)]}{m(m-1)} + 4\beta^2\alpha^2\|C\|_2^2\frac{\sqrt{\log(2m(m-1)/\delta)}}{\sqrt{n}m(m-1)}, \label{eq:slrf20}
\end{align}
where $$\cF = \{f_B : f_B(\vec{x_l}, \vec{x_j}) = \| B\sigma(C\vec{x_l}) - B\sigma(C\vec{x_j})\|_2^2, \|B\|_F \le \beta \}$$
We repeat the above for all pairs $l \ne j$. The probability that for all pairs the expected loss will be close to the train loss is at least $1-\delta$ by the union bound. From~\eqref{eq:slrf20} we have that with probability $\ge 1-\delta$,
\begin{align}
     & S - \frac{1}{m(m-1)}\sum_{l=1}^m\sum_{j \ne l}\E_{\substack{\vec{x_l} \sim \cD(M_l),\\ \vec{x_j}\sim \cD(M_j)}}\| \hat{B}\sigma(C\vec{x_l}) - \hat{B}\sigma(C\vec{x_j})\|_2^2 \le 2\E[\cR_n(\cF)] + 4\beta^2\alpha^2\|C\|_2^2\frac{\sqrt{\log(m(m-1)/\delta)}}{\sqrt{n}} \notag\\
     \implies & \frac{1}{m(m-1)}\sum_{l=1}^m\sum_{j \ne l}\E_{\substack{\vec{x_l} \sim \cD(M_l),\\ \vec{x_j}\sim \cD(M_j)}}\| \hat{B}\sigma(C\vec{x_l}) - \hat{B}\sigma(C\vec{x_j})\|_2^2 \ge \frac{m(1 - O(\sqrt{\epsilon}))}{O(\|A\|_F^2)} - 2\cR_n(\cF) - 4\beta^2\alpha^2\|C\|_2^2\frac{\sqrt{\log(m(m-1) \delta)}}{\sqrt{n}}. \notag
\end{align}
Now we bound $\cR_n(\cF)$. Using Lemma~\ref{lem:common-rademacher}, we get that
\begin{align}
    \cR_n(\cF) &\le \frac{4\beta^2\alpha^2\|C\|_2^2}{\sqrt{n}}.
\end{align}
Note that to employ Lemma~\ref{lem:common-rademacher} we think of $\sigma(C\vec{x_{il}}) - \sigma(C\vec{x_{ij}}) = \vec{x}$ as the input to the functions in $\cF$.
Therefore,
\begin{align}
    \frac{1}{m(m-1)}\sum_{l=1}^m\sum_{j \ne l}\E_{\substack{\vec{x_l} \sim \cD(M_l),\\ \vec{x_j}\sim \cD(M_j)}}\| \hat{B}\sigma(C\vec{x_l}) - \hat{B}\sigma(C\vec{x_j})\|_2^2 &\ge \frac{m(1 - O(\sqrt{\epsilon}))}{O(\|A\|_F^2)} - \frac{8\beta^2\alpha^2\|C\|_2^2}{\sqrt{n}} \\ 
    &~~~- 4\beta^2\alpha^2\|C\|_2^2\frac{\sqrt{\log(m(m-1) \delta)}}{\sqrt{n}} \\
    &\ge \frac{m(1 - O(\sqrt{\epsilon}))}{O(\|A\|_F^2)},
\end{align}
for 
$$n \ge \Theta\left(\frac{\beta^4\alpha^4\|C\|_2^4\log(m\delta)\|A\|_F^4}{m^2}\right) = \Theta\left(s^{O(\log(1/\eps))}\log(m\delta)\right).$$
\end{proof}

Next we show that for a randomly chosen permutation, with high probability, we have that the inter-manifold representation distance averaged over consecutive pairs according to the permutation is also large.
\begin{lemma}
\label{lem:small-loss-rep-far-rand-perm}
Suppose $m$ is even and $m \ge \Theta(\log(2/\delta)s^{O(\log(1/\eps))}/K^2)$. Let $d(l,j) = E_{\substack{\vec{x_l} \sim \cD(M_l),\\ \vec{x_j}\sim \cD(M_j)}}\| \hat{B}\sigma(C\vec{x_l}) - \hat{B}\sigma(C\vec{x_j})\|_2^2$. Suppose we have that $\frac{1}{m(m-1)}\sum_{l=1}^m\sum_{j \ne l} d(l,j) \ge K$. Then for a randomly chosen permutation $\rho : [m] \to [m]$, we have that with probability $\ge 1-\delta$,
$$\frac{2}{m}\sum_{l=1}^{m/2} d(\rho(2l-1), \rho(2l)) \ge K/2.$$
\end{lemma}
\begin{proof}
First we have that 
\begin{align}
    \E_{\rho}\left[\frac{2}{m}\sum_{l=1}^{m/2} d(\rho(2l-1), \rho(2l)) \right] = \frac{1}{m(m-1)}\sum_{l=1}^m\sum_{j \ne l} d(l,j) \ge K.
\end{align}
This is because in expectation over random permutations, we see every pair $l,j$ the same number of times. The normalization ensures that the overall sums match.
Next, we show that concentration of the sum on the left hand side around its expected value.
Using a trick from \cite{talagrand1995concentration}, we view the process of choosing a random permutation on $[m]$ as follows. We start with the identity permutation. Then we perform a sequence of $m-1$ transpositions as follows. We transpose $(m, a_m)$, then $(m-1, a_{m-1})$ and so on till $(2,a_2)$ where each $a_j$ is uniformly samples from $[j]$. This will give us a uniformly random permutation at the end and it is defined by $\{a_l\}_{l=2}^m$ which are independent. From here, our strategy will be to bound the amount by which our sum $S_m = \frac{2}{m}\sum_{l=1}^{m/2} d(\rho(2l-1), \rho(2l))$ changes when the value of some $a_j$ is changed. Changing $a_j$ changes at most 3 locations in the final permutation (wherever $j$, the old $a_j$, the new $a_j$ end up). Therefore, at most $3$ terms in $S_m$ change. Noting that $\|B\sigma(C\vec{x_l}) - B\sigma(C\vec{x_j}) \|_2 \le 2\alpha \beta \|C\|_2$ we can deduce that by changing a single $a_j$,  $S_m$ changes by at most $\frac{6\alpha\beta\|C\|_2}{m}$. Now applying McDiarmid's inequality gives us
\begin{align}
    \Pr_{\rho}\left[ \left\lvert \frac{2}{m}\sum_{l=1}^{m/2} d(\rho(2l-1), \rho(2l)) - \E_{\rho}\left[\frac{2}{m}\sum_{l=1}^{m/2} d(\rho(2l-1), \rho(2l)) \right] \right\rvert \ge t\right] \le 2\exp\left(-\frac{t^2m}{36\alpha^2\beta^2\|C\|_2^2} \right).
\end{align}
Taking $t = \sqrt{\frac{\log(2/\delta)}{m}}6\alpha\beta\|C\|_2$, we get that with probability $1-\delta$,
\begin{align}
    \frac{2}{m}\sum_{l=1}^{m/2} d(\rho(2l-1), \rho(2l)) \ge K - t \ge K/2,
\end{align}
for $m = 144\log(2/\delta) \alpha^2\beta^2\|C\|_2^2/K^2 = \Theta(\log(2/\delta)s^{O(\log(1/\eps))}/K^2)$.
\end{proof}

Finally we show property~\eqref{eq:hash-far} of GSH for most new manifolds sampled from $\cM$.
\begin{lemma}
\label{lem:small-loss-rep-far-final}
For $M_{m+1}, M_{m+2} \sim \mathcal{M}^2$, with probability $\ge 1-\delta$,
$$E_{\substack{\vec{x_{m+1}} \sim \cD(M_{m+1}),\\ \vec{x_{m+2}}\sim \cD(M_{m+2})}}\| \hat{B}\sigma(C\vec{x_{m+1}}) - \hat{B}\sigma(C\vec{x_{m+2}})\|_2^2 \ge K/4,$$
for
$$m \ge \Theta\left(\frac{s^{O(\log(1/\eps))}\log(2/\delta)}{K^2} \right).$$
\end{lemma}
\begin{proof}
Consider the following process. We sample $m/2$ pairs from $\mathcal{M}^2$, $\{M_l = (M_{l1}, M_{l2})\}_{l=1}^{m/2}$. Define $d(M_l) = \E_{\substack{\vec{x_1} \sim \cD(M_{l1}),\\ \vec{x_2}\sim \cD(M_{l2})}}\| \hat{B}\sigma(C\vec{x_1}) - \hat{B}\sigma(C\vec{x_2})\|_2^2$. 
It is easy to see that our original sampling process of getting $M_1, \ldots, M_m$ and choosing a random permutation to order these $m$ manifolds in and pair consecutive ones is identical in distribution to the above described process. Hence, any probability statements for the former process hold also for the latter and vice versa.
Let 
$$\cF = \{d_B : M \to \mathbb{R}\;  | \; d_B(M) = \E_{\substack{\vec{x_1} \sim \cD(M_{1}),\\ \vec{x_2}\sim \cD(M_{2})}}\| \hat{B}\sigma(C\vec{x_1}) - \hat{B}\sigma(C\vec{x_2})\|_2^2, \text{ where } M \in \mathcal{M}^2, \|B\|_F \le \beta\}$$
We have that with probability $1-\delta$,
\begin{align}
    &\sup_{d_B \in \cF} \sum_{l=1}^{m/2} d_B(M_l) - \E_{M \sim \mathcal{M}^2}[d_B(M)] \le 2\E[\cR_{m/2}(\cF)] + 4\beta^2\alpha^2\|C\|_2^2\sqrt{\frac{2\log(2/\delta)}{m}} \notag\\
    \implies & \E_{M \sim \mathcal{M}^2}[d_{\hat{B}}(M)] \ge K/2 - 2\E[\cR_{m/2}(\cF)] - 4\beta^2\alpha^2\|C\|_2^2\sqrt{\frac{2\log(2/\delta)}{m}}.\notag
\end{align}
Now it remains to bound $\cR_{m/2}(\cF)$. Given $M \in \supp(\cM)$ and $\vec{x_1},\vec{x_2} \sim \cD(M_1), \cD(M_2)$ respectively, let $\vec{z} = \sigma(C\vec{x_1}) - \sigma(C\vec{x_2})$. Then we have $\E_{\vec{x_1},\vec{x_2} \sim \cD(M_1), \cD(M_2)}[\|\vec{z}\|_2^2] \le 4\alpha^2\|C\|_2^2$.
\begin{align}
    \cR_{m/2}(\cF) &= \frac{2}{m}\E_{\vxi, \{M_l = (M_{l1}, M_{l2}) \}_{l=1}^{m/2} }\left[ \sup_{B, \|B\|_F \le \beta} \sum_{l=1}^n \xi_l \E_{\substack{\vec{x_1} \sim \cD(M_{l1}),\\ \vec{x_2}\sim \cD(M_{l2})}}\|B\vec{z}\|_2^2\right] \le \frac{4\sqrt{2}\beta^2\alpha^2\|C\|_2^2}{\sqrt{m}} \notag
\end{align}
using a line of calculations similar to those done in the proof of Lemma~\ref{lem:gen-transfer-app}. Therefore, we have
\begin{align}
    \E_{M \sim \mathcal{M}^2}[d_{\hat{B}}(M)] &\ge K/2 - \frac{8\sqrt{2}\beta^2\alpha^2\|C\|_2^2}{\sqrt{m}} - 4\beta^2\alpha^2\|C\|_2^2\sqrt{\frac{2\log(2/\delta)}{m}} \notag\\
    &\ge K/2 - \frac{12\sqrt{2}\beta^2\alpha^2\|C\|_2^2\sqrt{\log(2/\delta)}}{\sqrt{m}} \ge K/4 \notag
\end{align}
for 
$$m \ge \frac{4608\beta^2\alpha^2\|C\|_2^2\log(2/\delta)}{K^2} = \Theta\left( \frac{s^{O(\log(1/\eps))}\log(2/\delta)}{K^2}\right).$$
\end{proof}

\section{Intra-Class Hashing Property Without Variance Regularization}
\label{sec:hashing-without-var-reg}
\begin{theorem}[Property~\eqref{eq:hash-near} without Variance Regularization]
\label{hashingtraindatanovariancereg}
Given our 3-layer neural network, and given $n$ train samples each from $m$ train manifolds, training the following objective results in a network that satisfies that $\hat{V}_{mn}(B\sigma(C\vec{x})) \rightarrow 0$ as $\lambda_1, \lambda_2 \rightarrow 0$
\begin{align}
\label{eq:no-var-reg-main-opt}
    \min_{A,B} \cL_{A,B}(Y, \hat{Y}) + \lambda_1 \|A\|_F^2 + \lambda_2\| B\|_F^2.
\end{align}
\end{theorem}
\begin{proof}



The main point to note is if $\lambda_1, \lambda_2 \rightarrow 0$ (that is very small) then the objective is dominated by $\cL_{A,B}(Y, \hat{Y})$ which is minimized only if the prediction $\hat Y$ does not depend on $\btheta$  -- because if it did then by replacing $\hat Y$ by $E_{n}[\hat Y]$ for each of the $m$ manifolds, decreases the objective as shown in Lemma~\ref{lem:mean-reduces-square-loss}. 

Now, we know that there is a ground truth model where $\hat{V}_{mn}(\vec{\hat y}) \le \epsilon$. Then it follows from Lemma~\ref{lem:var-rep-output-relation} that $\hat{V}_{mn}(\vec{r}) \le 2 \sqrt {\epsilon}$.
To get close to this ground truth we select $\lambda_1 = \eps/m$ and $\lambda_2 = \eps/s^{O(\log(1/\eps))}$. Hence by letting $\eps \to 0$ we get the desired result.

\end{proof}

\section{Recovering \texorpdfstring{$\bgamma$}{\textbf{r}} from the Representation}
\label{sec:recovering-gamma}
We have argued in Section~\ref{sec:intro} that in many cases where $\bgamma$ represents a set of semantic concepts such as the shape or texture of an image, it is of interest to recover exactly the latent vector $\bgamma$ and not just an isomorphism $f(\bgamma)$.
The next lemma shows that there is a linear transform that maps our learnt representation $r(\vec{x})$ to approximately $\bgamma$ associated with $\vec{x}$; however we can only show this with the unweighted square loss when the regularization weights $\lambda_1, \lambda_2$ is tiny and only for the train manifolds. Our experiments show that this reversibility holds even with our variant of the weighted square loss $\cL_{A,B}(Y, \hat{Y})$. 
\begin{lemma}[Reversibility of the Learnt Representation]
Consider the minimization $\min_{A,B} \E\left[\|Y - AB\sigma(Cx))\|_F^2\right] + \lambda_1(\|A\|_F^2) + \lambda_2(\| B\|_F^2)$ subject to $v_\theta(r) = 0$. As $\lambda_1, \lambda_2 \rightarrow 0$ and for infinite width $C$ layer, there is a linear transform $R$ so that $R \hat{B}\sigma(C\vec{x_l}) = \vec{\bgamma_l}$ for any $\vec{x_l}$ from any of the $m$ training manifolds.

\end{lemma}
\begin{proof}

We will show that if $\bgamma$ is not expressible as a linear transform of $r(\vec{x})$ then creating additional outputs of the $B$ layer that emit $\bgamma$ only improves the loss objective. First note that the width of the hidden layer never needs to be more than $m$ (as otherwise we can replace $A, B$ by their appropriate truncated-$\svd$ versions that are of at most width $m$ since the rank of $AB$ is at most $m$). So even if we add additional co-ordinates to $r(\vec{x})$ the width remains bounded. We have also assume that the variance at the representation layer is $0$ for each manifold so the representation layer is a function only of the manifold for the train data.

Note that we have assumed  $\lambda = 0$ and the width of the random ReLU layer goes to $\infty$. In this scenario, we know by Lemma~\ref{lem:linear-expressibility-app} that for every manifold $M_l \in \mathcal{M}$, the representation computed by $\sigma(C\vec{x_l})$ when $\vec{x_l} \sim \cD(M_l)$ is powerful enough to express $\vec{\bgamma_l}$ exactly. We will show that if $\vec{\bgamma_l}$ cannot be expressed as a linear combination of coordinates of $\vec{r} = \hat{B}\sigma(C\vec{x_l})$ over the training samples $\vec{x_l} \sim \cD(M_l)$ then the loss $\E_{x_l \sim M_l}[\|Y - A\vec{r}\|_F^2]$ is not at a minimum and can be further reduced. Let $A_l$ denote the $l^{th}$ row of $A$. Let $\hat{A_l}$ be the regression minimization for the term $\min_{A_l} \|Y_l - A_l \vec{r}\|_2$ which will be the optimal trained value of $A_l$. Let us find the improvement to this term by appending $\vec{\bgamma}^j$ the the $j^{th}$ coordinate of $\vec{\bgamma}$ which is the vector of $\gamma_{lj}$ over the different manifolds $l$  Let $\vec{\bgamma}^j$ denote the vector of $\gamma_{lj}$ over the different manifolds $l$. Note that in any linear regression problem the improvement obtained from a new coordinate of the input features can be quantified as follows: orthonormalize it with respect to the other coordinates and measure the square of the projection of the output vector along this new orthonormalized input coordinate. So when a new coordinate $\bgamma^j$ has been added to the input, the decrease in the square loss is $(\left\langle Y_l , \bgamma^{j'}/|\bgamma^{j'}|_2 \right\rangle)^2 = (Y_l^{\top} \bgamma^{j'})^2/|\bgamma^{j'}|_2^2$ where $\bgamma^{j'}$ is the component of $\bgamma^j$ that is orthogonal to $\vec{r}$; that is, $\bgamma^{j'} = \bgamma^j - d^\top \vec{r}$ so that $\langle \bgamma^{j'}, \vec{r} \rangle = 0$. Note that since $\bgamma^j$ have been normalized, the improvement is at least $(Y_l^\top \bgamma^{j'})^2 = \sum_{l=1}^m(\bgamma^{j'}_{l})^2$ (as $Y_l$ is the indicator of the $l^{th}$ coordinate).
So the total improvement over all the manifolds from $\bgamma^j$ is at least $\sum_l (\bgamma^{j'}_{l})^2 = \|\bgamma^{j'}\|_2^2$. 
Now since the loss is at a local minimum it must be that there is no improvement possible which means $\|\bgamma^{j'}\|_2 = 0$ which means $\bgamma^j = \vec{a}^\top \vec{h}$ for some $\vec{a}$ and the same argument must be true for each coordinate of $\bgamma$ and so $\bgamma = R \vec{h}$ for some $R$. \\
\end{proof}

\begin{remark}
Although we assumed $\lambda \rightarrow 0$ and width of ReLU layer tends to $\infty$,
note that if the width of $C$ is bounded and large, then $\bgamma$ can only be approximately expressed in terms of $\sigma(C\vec{x})$. In that case that approximate version of $\bgamma$ must be linearly expressible in terms of $\vec{h}$. \\
Also even if $\lambda$ is not $0$, note that as long as $\sum_j \|\bgamma^{j'}\|_2^2 > \lambda$ the increase in regularization loss is more than offset by the decrease in square loss -- to realize the improvement by $\sum_j \|\bgamma^{j'}\|_2^2$  the $A$ matrix will add an edge of weight $\langle Y_i, \bgamma^{j'} \rangle$ between each $\bgamma^j$ and $Y_i$ and the $B$ matrix needs to add new nodes corresponding to $b$ which has a bounded norm in terms of $\sigma(C\vec{x})$ which bounds the increase in Frobenius norm of $B$. 
\end{remark}

\section{First or Second Order Methods Converge Provably to a Local Optimum}
\label{sec:apx-opt-alg}
Here we point the reader to two results about some popular first-order optimization methods which have the property that they converge quickly to a local optimum for smooth optimization objectives.
The first is the work of \cite{ge2015escaping} which shows that for strictly-saddle objectives, a form of stochastic gradient descent provably converges to a local optimum. The second is the work of \cite{agarwal2017finding} who show that a second-order algorithm FastCubic converges to local optimum faster than gradient descent converges to any critical point for a set of smooth objectives which includes neural net training.
There are more references within the above works studying similar properties of other variants as well.
\section{Additional Experimental Details}
\label{sec:experiments-app}
 In this section, we support our theoretical results with an empirical study of the GSH property of DNNs on real and synthetic data. First, we detail our experimental setup and then discuss the experimental results. 
 \subsection{Experimental Setup} 
 
 We separate our experiments to two groups, based on the data generating process.
 \paragraph{Natural Images.} We train Myrtle-CNN~\cite{myrtle}---a five layer convolutional neural network---on MNIST and CIFAR-10 with $\ell_2$ regularization without regularizing the bias terms. For CIFAR-10 the width parameter is $c=128$ while for MNIST it is $c=32$ and we remove the last two pooling layers. For both cases, we train via the SGD optimizer for $50$ epochs with learning rate of $0.1$ then drop the learning rate to $0.01$ for another $100$ epochs with batch size of $128$. We use $\lambda=0.1$ The resulting test accuracies are $99.4\%$ for MNIST and $88.9\%$ for CIFAR-10 while they also perfectly fit the train. 
 
 \paragraph{Synthetic Data.} For the synthetic data we do the following data generating process, as to satisfy Assumption~\ref{ass:main}: first we randomly sample $\bgamma$ and $\btheta$ from the standard and $1/\sqrt{k}$ scaled Gaussians on $\R^{k}$ respectively for $k=11$. Then, we sample two random matrices $V, W$ from a scaled Gaussian $\cN\left(0, \frac{1}{d}I\right)$ on $\R^{d\times k}$ and use an analytical  function $\mathbf{p}$ such as the $\sin(\cdot)$ to generate $\x=\mathbf{p}(W\bgamma + U\btheta)$. The analytic functions we tried are $e^{x/2}, \sin(x),\cos(x)$ and $\log((1+x^2)/2)$. Note that the last two are even functions, so precise recovery of  $\bgamma$  is impossible as $f(x)=f(-x)$. To increase the complexity of the manifold we sum $4$ functions of each type, so for example, the final sine data generating function is $\sum_{i=1^4} \sin(V_i\bgamma + W_i\btheta)$. We also call a sum of all 4 functions the \emph{Mixture distribution}. Now, in order to generate examples from the same manifold, we repeat the above process with $V, W, \bgamma$ fixed and vary (generate) $\btheta$. We then train a three layer Multi-Layer-Perceptron (MLP) with width $1000$ for $200$ epochs via SGD with learning rate $0.1$, batch size of $32$ with $\lambda=0.01$-$\ell_2$ regularization parameter and $\ell_2$ loss for classification. The train/test accuracies after this procedure are $100\%$.
 
 $\textbf{Meta learning and $\bgamma$ recovery}$. The main advantage of  the synthetic data is that  we are able to generate as many manifold (and samples) as we want. Therefore, we can check what is the $\rho$ not only on manifolds we saw before, but also the behavior on the distribution. Moreover, we can generate enough manifolds to hope to fit a linear classifier on top of the representation. If the linear loss is small the representation is approximately  linearly isomorphic to $\bgamma$. This means that our representation successfully recovered the manifold geometry.
 
 In a similar fashion described in the \textbf{Synthetic Data} data generation process, we generate $4$ datasets: train, test, few-shot and few-shot-test.
 In the train (and test) there are $50$ classes (manifolds) with $8k$ train examples and $2k$ test examples per class. As for the few-shot (and few-shot-test) we generate $10000$ manifolds so  their representations will serve as train set for the linear classifier that will try to recover $\bgamma$. In order to estimate the $\rho$ on the distribution of unseen manifold, we sample $5$ samples from each manifold to a total of $50000$ few-shot (train) samples. We then use SGD with learning rate of $0.1$ and batch size of $32$ fit a linear model with loss $\|r(\x)-\bgamma\|_2^2$. Finally, we sample another $100$ manifolds with one sample per manifold to estimate how well the linear function recovered $\bgamma$. In order to measure how close the representation is to an isomorphism, we use normalized distance as a metric (see table~\ref{tab:main-res}), normalizing $\|Wr(\x)-\bgamma\|_2^2$, where $W$ is the learnt linear model, by the average distance between two  $\bgamma$s: $\|\bgamma_i-\bgamma_j\|_2^2$. So a random Gaussian will produce normalized distance of $1$ while perfect linear recovery will be normalized distance of $0$.
 
 \subsection{Experimental Results}
 
 Our results for real data are in shown in Figure~\ref{fig:histograms} and our results for synthetic data are summarized in Table~\ref{tab:main-res}. We observe high $\rho$ values for both real ($1.46$ for CIFAR-10 and $3.36$ for MNIST) and synthetic distributions ($\rho\ge10.7$) both on train and on test. Furthermore, for synthetic data, we see that even out-of-distribution $\rho$ is high. This implies that the classifier learnt is a GSH function on the population of manifolds, effectively inverting the data generating process. Further, we see that when the function is not even (and thus $\bgamma$ is not recoverable) as is the case for the Sine, Log, Exp and Mixture distributions, we are able to recover $\bgamma$ from the representation using a linear function. Specifically the normalized $\bgamma$ recovery distance is at $\approx0.1$ where a random $\gamma$ would yield a distance of $1$.

 \begin{table*}[!ht]
\caption{Our results on synthetic data. We provide the $\rho$ value for 5 different  synthetic distributions, on train, test and transfer (i.e., unseen $\bgamma$s). We also note the normalized distance of $\|\bgamma -\hat\bgamma\|$ where $\hat \bgamma$ is a linear classifier on top of the representation layer attempting to recover $\bgamma$. We are able to nicely recover the Mixture, Sine and Exp distributions, while recovering the $cos(x)$ and $\log((1+x^2)/2)$ proves more difficult. The reason is that we fail to recover $\bgamma$ for these functions is that they are even, so there is an ambiguity of whether the  sign is positive or negative.}\label{tab:main-res}
\begin{center}
\begin{tabular}{c|c|c|c|c|c}
\toprule
$\substack{\text{Data Generating}\\ \text{Function}}$ & $\substack{\text{Test}\\\text{Accuracy}}$ & $\rho$-Train  & $\rho$-Test & $\rho$-Transfer &
$\substack{\text{Normalized $\bgamma$}\\ \text{Recovery Distance}}$ \\
\midrule
Mixture & 100\% &19.13 & 19.09 & 20.46 & 0.11\\
Sine & 100\% & 26.8 & 26.8 & 28.4 & 0.09\\
Cosine & 100\% & 10.7 & 10.79 & 11.09 & 0.71\\
Log & 100\% & 12.38 & 12.38 & 11.23 & 0.71\\
Exp & 100\% & 25.2 & 25.2 & 28.77 & 0.08\\
\bottomrule
\end{tabular}
\end{center}
\end{table*}

 \begin{figure*}[!ht]
\begin{center}
  \captionsetup{width=0.5\textwidth}
\includegraphics[width=0.5\textwidth]{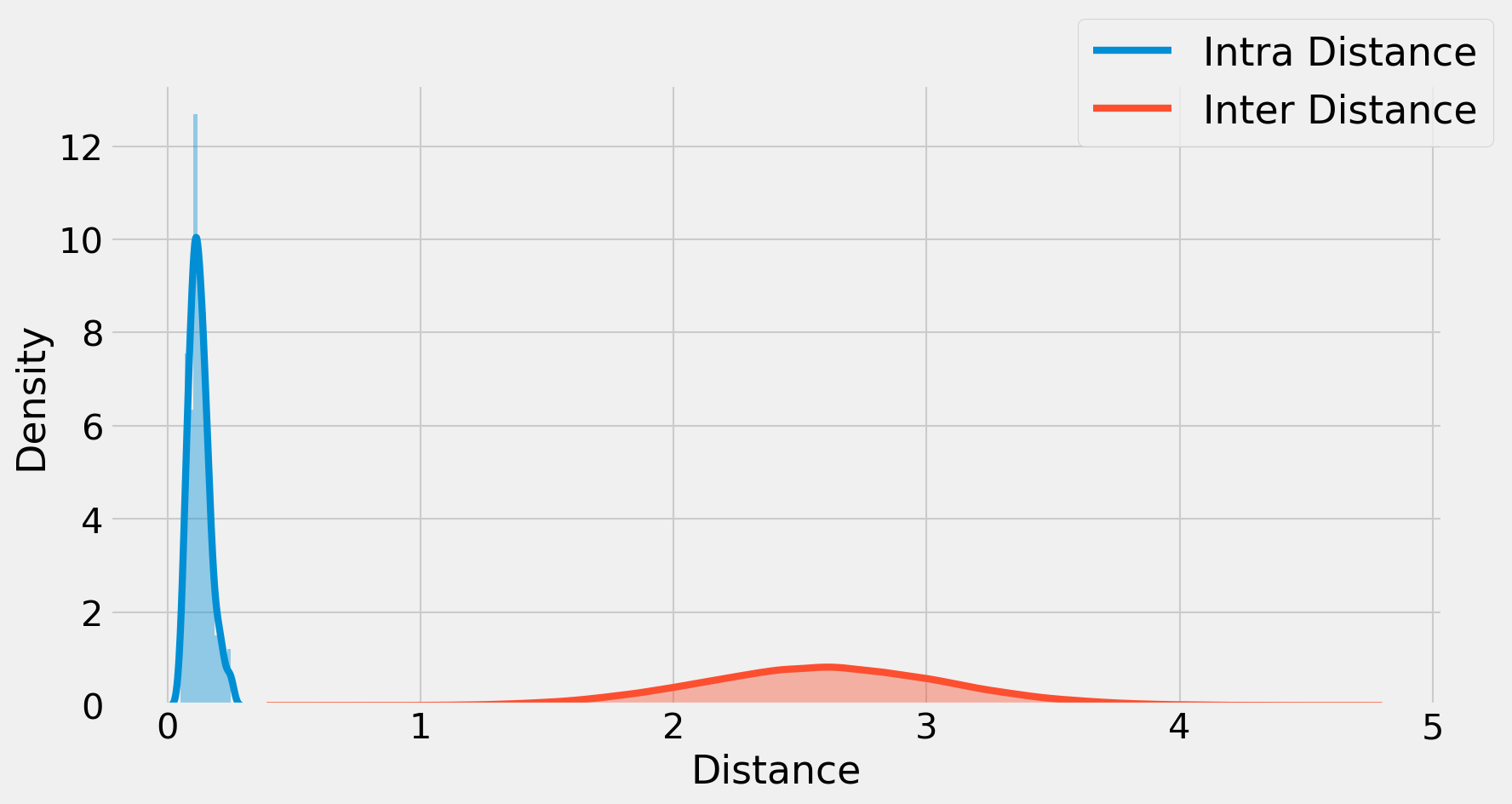}  
\caption{A comparison of intra vs inter class distances for \emph{Unseen Manifold} for an MLP trained on Mixed synthetic data. Remarkably, even on unseen manifolds the GSH property holds, that is the representation is invariant to the ``noisy feature" $\btheta$ while being sensitive to the semantically meaningful feature $\bgamma$.}
  \label{fig:histograms-app}
\end{center}
  
\end{figure*}

 \begin{figure*}[!ht]
\begin{center}
  \captionsetup{width=0.5\textwidth}
\includegraphics[width=0.5\textwidth]{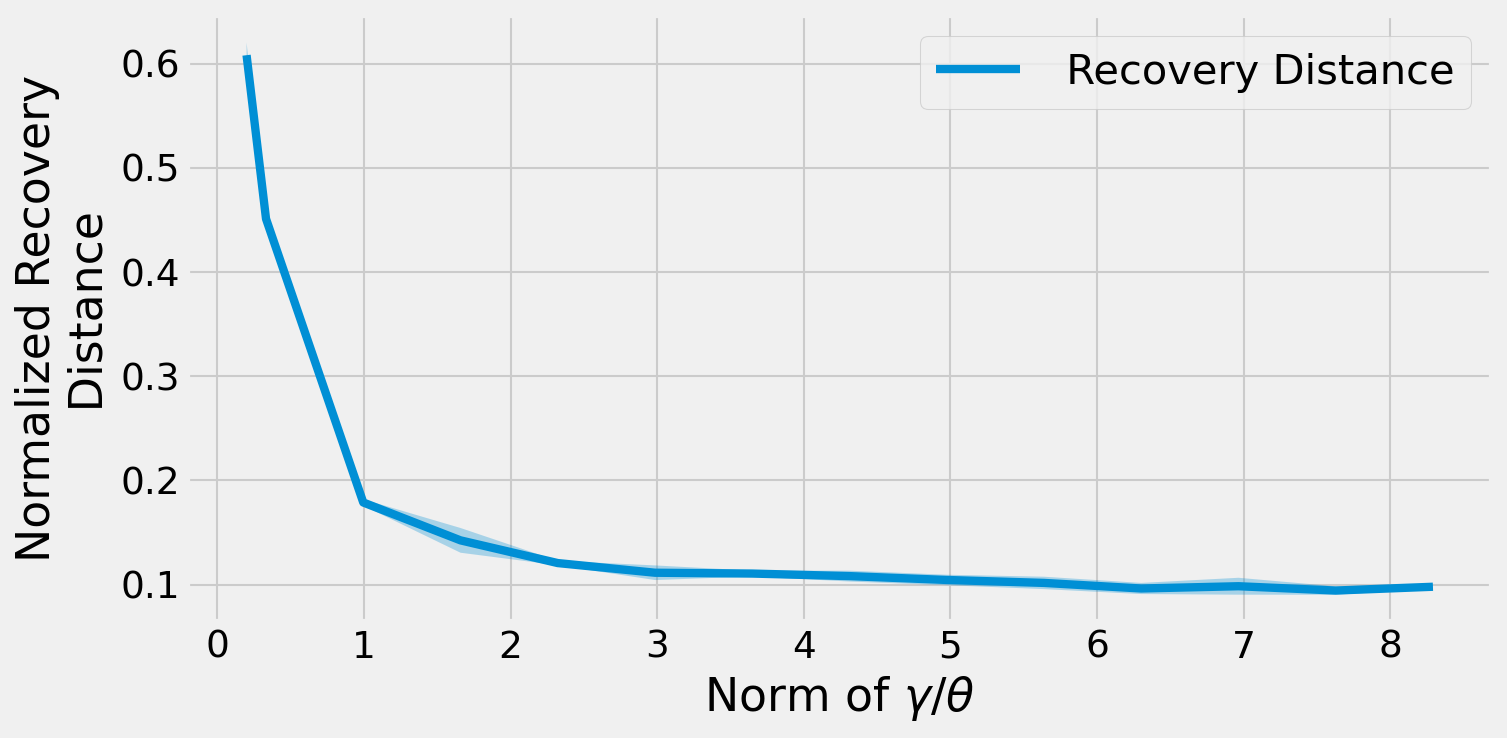} 
\caption{We show the ability to recover $\bgamma$ with linear regression over the representation with varying the norm of $\btheta/\bgamma$. We see that when the norm of $\bgamma$ is dominating, the learnt representation is almost a linear function of $\bgamma$. In contrast, when $\btheta$ has larger norm, the learnt representation becomes a non-linear function of $\bgamma$. (We know that $\bgamma$ and the representation are isomorphic as long as $\rho$ is large enough.). }
  \label{fig:inverting}
\end{center}
  
\end{figure*}

\end{document}